\documentclass[twoside,11pt]{article}

\usepackage{blindtext}

\usepackage[preprint]{jmlr2e}

\usepackage[utf8]{inputenc}
\usepackage[T1]{fontenc}
\usepackage[english]{babel}

\usepackage{xargs}
\usepackage{xurl}
\usepackage{xcolor}
\usepackage[mathscr]{euscript}
\usepackage{amsfonts, amsmath, amssymb}
\usepackage{mathtools}
\usepackage[ruled]{algorithm2e}
\usepackage{scalerel}
\usepackage{stmaryrd}
\usepackage{lmodern}
\usepackage{bm}

\usepackage{tikz, pgfplots}
\usepackage{tkz-tab}
\pgfplotsset{compat=1.16}
\usepgfplotslibrary{groupplots}

\let\originalleft\left
\let\originalright\right
\renewcommand{\left}{\mathopen{}\mathclose\bgroup\originalleft}
\renewcommand{\right}{\aftergroup\egroup\originalright}

\newcommand{\kron}{\mathbin{\scalerel*{\boxtimes}{\otimes}}}

\newcommandx{\lito}[2][1={}, 2={}]{\operatorname*{o_\mathnormal{#1}^\mathnormal{#2}}}
\newcommandx{\bigO}[2][1={}, 2={}]{\operatorname*{\mathcal{O}_\mathnormal{#1}^\mathnormal{#2}}}
\newcommandx{\bigTh}[2][1={}, 2={}]{\operatorname*{\Theta_\mathnormal{#1}^\mathnormal{#2}}}

\DeclareMathOperator*{\argmax}{\arg\max}

\DeclareMathOperator{\Diag}{Diag}
\DeclareMathOperator{\Dist}{Dist}

\DeclareMathOperator{\Sp}{Sp}
\DeclareMathOperator{\Span}{Span}
\DeclareMathOperator{\Supp}{Supp}
\DeclareMathOperator{\Tr}{Tr}
\DeclareMathOperator{\Var}{Var}

\DeclareMathOperator*{\bigkron}{\scalerel*{\boxtimes}{\bigotimes}}

\newcommand{\eqdef}{\overset{\text{def}}{=}}
\newcommand{\simiid}{\overset{\text{i.i.d.}}{\sim}}
\newcommand{\abs}[1]{\lvert #1 \rvert}
\newcommand{\Abs}[1]{\left\lvert #1 \right\rvert}
\newcommand{\deriv}[2]{\frac{\partial #1}{\partial #2}}
\newcommand{\esp}[1]{\mathbb{E}[#1]}
\newcommand{\Esp}[1]{\mathbb{E} \left[ #1 \right]}
\newcommand{\norm}[1]{\lVert #1 \rVert}
\newcommand{\Norm}[1]{\left\lVert #1 \right\rVert}
\newcommand{\scal}[2]{\langle #1, #2 \rangle}
\newcommand{\Scal}[2]{\left\langle #1, #2 \right\rangle}
\newcommand{\set}[1]{\{ #1 \}}

\newcommand{\tucker}[2]{\llbracket #1 ; #2 \rrbracket}
\newcommand{\Tucker}[2]{\left\llbracket #1 ; #2 \right\rrbracket}

\newcommand{\bzero}{{\mathbf{0}}}

\newcommand{\ba}{{\bm{a}}}
\newcommand{\bb}{{\bm{b}}}
\newcommand{\bc}{{\bm{c}}}

\newcommand{\bu}{{\bm{u}}}
\newcommand{\bv}{{\bm{v}}}
\newcommand{\bw}{{\bm{w}}}
\newcommand{\bx}{{\bm{x}}}
\newcommand{\by}{{\bm{y}}}
\newcommand{\bz}{{\bm{z}}}

\newcommand{\bA}{{\bm{A}}}
\newcommand{\bB}{{\bm{B}}}
\newcommand{\bC}{{\bm{C}}}
\newcommand{\bD}{{\bm{D}}}

\newcommand{\bG}{{\bm{G}}}
\newcommand{\bH}{{\bm{H}}}
\newcommand{\bI}{{\bm{I}}}

\newcommand{\bL}{{\bm{L}}}
\newcommand{\bM}{{\bm{M}}}
\newcommand{\bN}{{\bm{N}}}
\newcommand{\bO}{{\bm{O}}}
\newcommand{\bP}{{\bm{P}}}
\newcommand{\bQ}{{\bm{Q}}}

\newcommand{\bS}{{\bm{S}}}
\newcommand{\bT}{{\bm{T}}}
\newcommand{\bU}{{\bm{U}}}
\newcommand{\bV}{{\bm{V}}}
\newcommand{\bW}{{\bm{W}}}
\newcommand{\bX}{{\bm{X}}}
\newcommand{\bY}{{\bm{Y}}}
\newcommand{\bZ}{{\bm{Z}}}

\newcommand{\bDelta}{{\bm{\Delta}}}

\newcommand{\bLambda}{{\bm{\Lambda}}}

\newcommand{\bSigma}{{\bm{\Sigma}}}

\newcommand{\bOmega}{{\bm{\Omega}}}

\newcommand{\scrA}{{\mathscr{A}}}
\newcommand{\scrB}{{\mathscr{B}}}

\newcommand{\scrG}{{\mathscr{G}}}
\newcommand{\scrH}{{\mathscr{H}}}

\newcommand{\scrN}{{\mathscr{N}}}

\newcommand{\scrT}{{\mathscr{T}}}

\newcommand{\bscrA}{{\bm{\mathscr{A}}}}
\newcommand{\bscrB}{{\bm{\mathscr{B}}}}

\newcommand{\bscrG}{{\bm{\mathscr{G}}}}
\newcommand{\bscrH}{{\bm{\mathscr{H}}}}

\newcommand{\bscrN}{{\bm{\mathscr{N}}}}

\newcommand{\bscrP}{{\bm{\mathscr{P}}}}

\newcommand{\bscrT}{{\bm{\mathscr{T}}}}

\newcommand{\bscrX}{{\bm{\mathscr{X}}}}

\newcommand{\calA}{{\mathcal{A}}}

\newcommand{\calC}{{\mathcal{C}}}
\newcommand{\calD}{{\mathcal{D}}}

\newcommand{\calK}{{\mathcal{K}}}
\newcommand{\calL}{{\mathcal{L}}}

\newcommand{\calN}{{\mathcal{N}}}

\newcommand{\bbC}{{\mathbb{C}}}

\newcommand{\bbE}{{\mathbb{E}}}

\newcommand{\bbP}{{\mathbb{P}}}

\newcommand{\bbR}{{\mathbb{R}}}

\newcommand{\rmd}{{\mathrm{d}}}

\newcommand{\rmi}{{\mathrm{i}}}

\newcommand{\rmF}{{\mathrm{F}}}

\definecolor{C0}{HTML}{1F77B4}
\definecolor{C1}{HTML}{FF7F0E}
\definecolor{C2}{HTML}{2CA02C}
\definecolor{C3}{HTML}{D62728}
\definecolor{C4}{HTML}{9467BD}
\definecolor{C5}{HTML}{8C564B}
\definecolor{C6}{HTML}{E377C2}
\definecolor{C7}{HTML}{7F7F7F}
\definecolor{C8}{HTML}{BCBD22}
\definecolor{C9}{HTML}{17BECF}

\usepackage{lastpage}
\jmlrheading{26}{2025}{1-\pageref{LastPage}}{2/24; Revised 6/24}{1/25}{24-0193}{Hugo Lebeau, Florent Chatelain and Romain Couillet}

\ShortHeadings{A Random Matrix Approach to Low-Multilinear-Rank Tensor Approximation}{Lebeau, Chatelain and Couillet}
\firstpageno{1}

\begin{document}

\title{A Random Matrix Approach to Low-Multilinear-Rank Tensor Approximation}

\author{\name Hugo Lebeau \email hugo.lebeau@univ-grenoble-alpes.fr \\
       \addr Université Grenoble Alpes, CNRS, Inria, Grenoble INP, LIG \\
       Grenoble, 38000, France
       \AND
       \name Florent Chatelain \email florent.chatelain@grenoble-inp.fr \\
       \addr Université Grenoble Alpes, CNRS, Grenoble INP, GIPSA-lab \\
       Grenoble, 38000, France
       \AND
       \name Romain Couillet \email romain.couillet@univ-grenoble-alpes.fr \\
       \addr Université Grenoble Alpes, CNRS, Inria, Grenoble INP, LIG \\
       Grenoble, 38000, France}

\editor{Genevera Allen}

\maketitle

\begin{abstract}
This work presents a comprehensive understanding of the estimation of a planted low-rank signal from a general spiked tensor model near the computational threshold. Relying on standard tools from the theory of large random matrices, we characterize the large-dimensional spectral behavior of the unfoldings of the data tensor and exhibit relevant signal-to-noise ratios governing the detectability of the principal directions of the signal. These results allow to accurately predict the reconstruction performance of truncated multilinear SVD (MLSVD) in the non-trivial regime. This is particularly important since it serves as an initialization of the higher-order orthogonal iteration (HOOI) scheme, whose convergence to the best low-multilinear-rank approximation depends entirely on its initialization. We give a sufficient condition for the convergence of HOOI and show that the number of iterations before convergence tends to $1$ in the large-dimensional limit.
\end{abstract} 


\section{Introduction}
\label{sec:intro}

Information retrieval from large amounts of data has become a common task of signal processing and machine learning in the past decades. Often, these data have several modes as they may come from various sources, modalities, domains, and so on. Tensors (multi-way arrays) are therefore a natural representation for such datasets --- they appear in multiple areas such as brain imaging \citep{zhou_tensor_2013}, neurophysiological measurements \citep{seely_tensor_2016}, community detection \citep{anandkumar_tensor_2013}, compression of hyperspectral images \citep{li_tensor_2010}, spatio-temporal gene expression \citep{liu_characterizing_2022}, recommender systems \citep{karatzoglou_multiverse_2010, rendle_pairwise_2010, frolov_tensor_2017} and topic modeling \citep{anandkumar_tensor_2014}. Indeed, tensors as multi-way arrays provide a more detailed representation of data than mere matrices (two-way arrays) as they convey a structural information. For instance, the modes of a data tensor can represent \textit{pixel $\times$ pixel $\times$ wavelength $\times$ sample} in hyperspectral imaging \citep{zhang_tensor_2013, kanatsoulis_hyperspectral_2018}, \textit{time $\times$ spatial scale $\times$ electrode} in the EEG analysis by \citet{acar_multiway_2007} or \textit{neuron $\times$ time $\times$ stimuli} in the study of the visual cortex by \citet{rabinowitz_attention_2015}.

In an information retrieval context, it is common to make use of tensor decompositions in order to estimate a sought signal. In their fMRI study, \citet{hunyadi_tensor_2017} perform a blind source separation via a joint tensor decomposition on a \textit{channel $\times$ time $\times$ patient} array, whereas \citet{williams_unsupervised_2018} use a low-rank tensor approximation on a \textit{neuron $\times$ time $\times$ trial} array as a dimensionality reduction technique to study neural dynamics. In fact, supposing that the signal has a low-rank structure is a natural sparsity assumption \citep{kadmon_statistical_2018, anandkumar_tensor_2014}, and low-rank tensor approximations are key tools to extract information from multi-way data.

In the present work, we propose a random matrix analysis of a general \textit{low-rank information $+$ noise} tensor model and precisely quantify the amount of information which can be recovered with a low-rank tensor approximation depending on the signal-to-noise ratio (SNR). For a general introduction to tensors, we refer the reader to \citet{comon_tensors_2014, comon_tensors_2009, landsberg_tensors_2011, hackbusch_tensor_2012} and, for an emphasis on statistical learning applications, \citet{bi_tensors_2021, sun_tensors_2021}. In the remainder of the introduction, the main concepts and challenges behind low-rank tensor estimation are presented in Section \ref{sec:intro:low-rank}. Then, Section \ref{sec:intro:related_work} introduces some important related works. Our main results are finally summarized in Section \ref{sec:intro:contributions}. All the notations are properly defined in Section \ref{sec:preliminaries}.

\subsection{Low-Rank Tensor Estimation}
\label{sec:intro:low-rank}

What is meant by a low-rank approximation of a tensor? And how is the \emph{rank} of a tensor actually defined? Let us start with a familiar matrix case: a matrix $\bM \in \bbR^{n_1 \times n_2}$ is a two-way array (or order-$2$ tensor). A singular value decomposition (SVD) allows us to write $\bM$ in a compact way as the sum of $R$ rank-$1$ terms, $\bM = \sum_{i = 1}^R \sigma_i \bu_i \bv_i^\top = \bU \bSigma \bV^\top$ where $\bU$, $\bV$ are respectively $n_1 \times R$ and $n_2 \times R$ matrices with \emph{orthonormal} columns and $\bSigma$ is the $R \times R$ \emph{diagonal} matrix of singular values. The rank of $\bM$ is here the minimal number of rank-$1$ terms in which it can be exactly decomposed. Extending this notion to tensors therefore seems straightforward: a tensor\footnote{It is chosen of order $3$ for simplicity of exposure.} $\bscrT \in \bbR^{n_1 \times n_2 \times n_3}$ has rank $R$ if it is the minimal number of rank-$1$ terms in which it can be exactly decomposed, $\bscrT = \sum_{i = 1}^R \sigma_i \ba_i \otimes \bb_i \otimes \bc_i$. What we have just described is the canonical polyadic decomposition (CPD) of $\bscrT$, it dates back to \citet{hitchcock_expression_1927} and is unique under very mild conditions \citep{kolda_tensor_2009}. However, we have lost an important property in this process: the unit vectors $\ba_i$ (resp.\ $\bb_i$, $\bc_i$) are, in general, no longer orthonormal. Conversely, retaining the orthonormality property inevitably results in the loss of the diagonality property, $\bscrT = \sum_{i = 1}^{r_1} \sum_{j = 1}^{r_2} \sum_{k = 1}^{r_3} \scrG_{i, j, k} \bu_i \otimes \bv_j \otimes \bw_k$. This latter decomposition is called a Tucker decomposition and dates back to \citet{tucker_mathematical_1966}. In fact, the best way to represent $\bscrT$ with a Tucker decomposition is to choose the $\bu_i$ (resp.\ $\bv_i$, $\bw_i$) as the left singular vectors of the unfolding of $\bscrT$ along mode $1$ (resp.\ $2$, $3$)\footnote{This is properly defined in Section \ref{sec:preliminaries:tensors}.}. This is called the multilinear SVD (MLSVD, \citealp{de_lathauwer_multilinear_2000}) and gives rise to a new definition of rank: the \emph{multilinear}-rank $(r_1, r_2, r_3)$. Note that, in the matrix case, $r_1 = r_2 = R$ since both the diagonality and orthonormality properties are verified. However, $r_1, r_2, r_3$ are, in general, not equal in the tensor case, but $\max(r_1, r_2, r_3) \leqslant R \leqslant \min(r_1 r_2, r_2 r_3, r_1 r_3)$. See, e.g., \citet{sidiropoulos_tensor_2017} for details. Other relevant references for the reader interested in tensor decompositions are \citet{kolda_tensor_2009, cichocki_tensor_2015, rabanser_introduction_2017}.

Given an order-$d$ tensor $\bscrT \in \bbR^{n_1 \times \ldots \times n_d}$ of possibly very high rank, we are interested in finding a low-rank approximation, i.e., an $n_1 \times \ldots \times n_d$ tensor $\bscrX$ which minimizes the distance $\norm{\bscrT - \bscrX}_\rmF$ on a set of low-rank tensors. Yet, the problem of the best rank-$R$ approximation of a tensor is ill-posed as soon as $R > 1$ because the set of rank-$R$ tensors is not closed \citep{kolda_tensor_2009}. Instead, we shall consider the best low-\emph{multilinear}-rank problem, which is always well-posed,
\begin{equation} \label{eq:problem}
\min_{\operatorname{rank}(\bscrX) \leqslant (r_1, \ldots, r_d)} \Norm{\bscrT - \bscrX}_\rmF^2.
\end{equation}
It is well known in the matrix case that the best rank-$R$ approximation can be easily computed by truncating the SVD to its $R$ most energetic terms \citep{eckart_approximation_1936, mirsky_symmetric_1960}. Could this also be true for the MLSVD? Unfortunately, counter-examples exist \citep{kolda_counterexample_2003}, showing that there is no tensor equivalent of the Eckart-Young-Mirsky theorem. Worse still, Problem \eqref{eq:problem} is in fact NP-hard \citep{hillar_most_2013}. Nevertheless, despite not being the best low-multilinear-rank approximation, the truncated MLSVD $\hat{\bscrT}$ remains a very good ``first guess'' as it verifies $\norm{\bscrT - \bscrT_\star}_\rmF \leqslant \norm{\bscrT - \hat{\bscrT}}_\rmF \leqslant \sqrt{d} \norm{\bscrT - \bscrT_\star}_\rmF$ where $\bscrT_\star$ denotes a solution to Problem \eqref{eq:problem} and $d$ is the order of the tensor \citep{grasedyck_literature_2013, hackbusch_tensor_2012}. It is a cheap (it consists only in $d$ standard matrix SVDs) and quasi-optimal low-multilinear-rank approximation of $\bscrT$. Moreover, it is often used as an initialization of numerical methods which estimate a solution to Problem \eqref{eq:problem}, among which the most common is the higher-order orthogonal iteration (HOOI) algorithm \citep{kroonenberg_principal_1980, kapteyn_approach_1986, de_lathauwer_best_2000}.

Another motivation for the analysis of the low-multilinear-rank approximation problem is that it has also a practical interest for the numerical computation of the canonical polyadic decomposition (CPD). Indeed, when dealing with large tensors, it is computationally more efficient to first compress the tensor with a low-multilinear-rank approximation and then compute the CPD on the smaller core tensor rather than computing the CPD of the large tensor directly \citep{bro_improving_1998}. This is done, e.g., by the \texttt{cpd} function of the \texttt{MATLAB} toolbox \texttt{Tensorlab} \citep{vervliet_tensorlab_2016}.

\subsection{Related Work}
\label{sec:intro:related_work}

Multilinear SVD (MLSVD) has a wide range of applications and is often used to extract relevant information from multi-way arrays. For instance, it has been used in human motion recognition \citep{vasilescu_human_2002}, face recognition \citep{vasilescu_multilinear_2003}, handwritten digit classification \citep{savas_handwritten_2007} but also genomics \citep{omberg_tensor_2007, omberg_global_2009, muralidhara_tensor_2011} and syndromic surveillance \citep{fanaee-t_eigenevent_2015}.

The analysis of \emph{spiked} tensor models --- i.e., low-rank perturbations of large random tensors --- has started with the introduction by \citet{montanari_statistical_2014} of the rank-$1$ symmetric spiked tensor model, $\bscrT = \beta \bx^{\otimes d} + \bscrN$ with $\norm{\bx} = 1$, $\bscrN$ Gaussian noise and $\beta$ a parameter controlling the signal-to-noise ratio (SNR). They show that estimation of $\bx$ from $\bscrT$ is \emph{theoretically} possible as soon as $\beta$ is above a certain threshold $\beta_c$ behaving like $\sqrt{d \log d}$, which is reminiscent of the now well-known spiked \emph{matrix} model where signal reconstruction is only possible above a critical threshold \citep{peche_largest_2006} --- a phenomenon called the BBP phase transition \citep{baik_phase_2005}. The behavior of singular values and singular vectors of spiked matrix models is comprehensively studied by \citet{benaych-georges_singular_2012}. Contrary to the matrix case however, \citet{montanari_statistical_2014} make the disturbing observation that none of the polynomial-time estimation algorithms among tensor unfolding, power iteration and approximate message passing (AMP) succeed unless $\beta$ diverges as the dimensions of the tensor grow large. The results of \citet{hopkins_tensor_2015, hopkins_power_2017} suggest that no polynomial-time algorithm can succeed unless $\beta \gtrsim N^{\frac{d - 2}{4}}$, where $N$ scales as the dimensions of the data tensor. While \citet{perry_statistical_2020} show that the information-theoretic threshold is of order $1$, this indicates the existence of a \emph{computational-to-statistical gap} in spiked tensor estimation, as in a myriad of other problems \citep{bandeira_notes_2018, zdeborova_statistical_2016}.

The landscape of the rank-$1$ symmetric spiked tensor model is studied by \citet{ben_arous_landscape_2019}, who show that the number of local optima to Problem \eqref{eq:problem} grows exponentially with the size of the tensor, but all lie close to a subspace orthogonal to the sought solution, except for one if $\beta$ exceeds a critical threshold $\beta_c$. Completing this analysis, \citet{jagannath_statistical_2020} show\footnote{The setting considered by \citet{jagannath_statistical_2020} is more general than the one of \citet{ben_arous_landscape_2019} because the noise in their model is not necessarily symmetric but the perturbation is still a rank-$1$ symmetric tensor.} the existence of two close but different thresholds $\beta_s < \beta_c$ such that the solution aligned with the underlying signal is a local minimum of Problem \eqref{eq:problem} as soon as $\beta > \beta_s$ but becomes a \emph{global} one only if $\beta > \beta_c$. Relying on the Kac-Rice method, \citet{ros_complex_2019} thoroughly study such high-dimensional landscapes and classify the different behaviors and phase transitions which can occur.

So far, we have only referred to works dealing with the rank-$1$ symmetric case, but there are also some studies on higher-(low-)rank spiked models. \citet{chevreuil_non-detectability_2018} give a sufficient (but not necessary) condition for the \emph{non-detectability} of a rank-$R$ asymmetric signal perturbed by an additive Gaussian noise. \citet{chen_phase_2021} also discuss signal detectability in the rank-$R$ symmetric case. The statistical inference of finite rank tensors is studied by \citet{chen_statistical_2022} who identify the limit free energy of the model in terms of a variational formula. \citet{zhang_tensor_2018} consider a general $\textit{low-multilinear-rank signal}~ \bscrP + \textit{Gaussian noise}~ \bscrN$ model and bring to light the same statistical-to-computational gap: if $\norm{\bscrP}_\rmF$ is above a statistical threshold of order $1$ then Problem \eqref{eq:problem} has a solution which is aligned with the signal but is computationally intractable unless $\norm{\bscrP}_\rmF$ is above a computational threshold of order $N^{\frac{d - 2}{4}}$. In this strong SNR regime, the higher-order orthogonal iteration (HOOI) algorithm \citep{de_lathauwer_best_2000} is minimax-optimal. In fact, it is also proved by \citet{ben_arous_algorithmic_2020} that, with Langevin dynamics and gradient descent, the algorithmic threshold behaves like $N^\alpha$ with $\alpha > 0$. Other algorithmic thresholds have been shown as well for semi-definite and spectral relaxations of the maximum likelihood problem \citep{hopkins_tensor_2015, hopkins_fast_2016, kim_community_2017}. AMP and tensor power iteration algorithms achieve $N^{\frac{d - 1}{2}}$ \citep{lesieur_statistical_2017, huang_power_2022} while tensor unfolding methods (truncated MLSVD and HOOI algorithm) achieve $N^{\frac{d - 2}{4}}$ as already conjectured by \citet{montanari_statistical_2014} and later proven by \citet{hopkins_tensor_2015, ben_arous_long_2023} in the rank-$1$ case. The convergence of the HOOI algorithm towards a local maximum for a \emph{sufficiently close} initialization is proven by \citet{xu_convergence_2018} and \citet{feldman_sharp_2023} show that it achieves exact recovery of a rank-$1$ perturbation in the large $N$ regime when it is initialized with the dominant singular vectors of the unfoldings.

Recently, a new approach relying on tools from random matrix theory has broaden the understanding of spiked tensor models. In particular, \citet{goulart_random_2022} study the rank-$1$ symmetric case and are able to recover explicitly the same $\beta_s$ threshold as \citet{jagannath_statistical_2020} as well as to precisely quantify the alignment between a solution to Problem \eqref{eq:problem} and the signal. A similar analysis is carried out by \citet{seddik_when_2022} for the more general asymmetric case, relying solely on classical techniques from random matrix theory. Such tools show promise for the theoretical understanding of learning from tensor data \citep{seddik_learning_2023}. In particular, the results of \citet{ben_arous_long_2023} and \citet{feldman_spiked_2023} on \emph{long} random matrices, similar to those considered in this work, provide instructive insight into the recovery performance of tensor unfolding methods.

\subsection{Summary of Contributions}
\label{sec:intro:contributions}

In low-rank tensor approximation, tensor unfolding methods achieve the best known performance among polynomial-time algorithms. Motivated by several works suggesting that such method could actually reach the computational threshold \citep{hopkins_tensor_2015, hopkins_power_2017, zhang_tensor_2018, wein_kikuchi_2019}, we propose a thorough random matrix analysis of the low-multilinear-rank tensor approximation problem.

Consider the general spiked tensor model,
\begin{equation} \label{eq:model}
\bscrT = \bscrP + \frac{1}{\sqrt{N}} \bscrN \quad \in \bbR^{n_1 \times \ldots \times n_d}, \qquad \scrN_{i_1, \ldots, i_d} \simiid \calN(0, 1),
\end{equation}
where $d \geqslant 3$ is the order of the tensor, $\bscrN$ is an additive Gaussian noise, $N$ is a parameter controlling the size of the tensor such that the ratio $n_\ell / N$ is constant\footnote{This ensures that the spectral norm of $\frac{1}{\sqrt{N}} \bscrN$ is of order $1$ \citep{tomioka_spectral_2014}.} (at least for $N$ above a certain threshold value $N_0$) for all $\ell \in \{1, \ldots, d\}$ (for instance, $N = n_1$ or $N = \sum_{\ell = 1}^d n_\ell$) and $\bscrP$ is a low-multilinear-rank deterministic tensor, i.e., which can be decomposed as
\begin{equation} \label{eq:P_decompositon}
\bscrP = \sum_{q_1 = 1}^{r_1} \ldots \sum_{q_d = 1}^{r_d} \scrH_{q_1, \ldots, q_d} [\bx^{(1)}_{q_1} \otimes \ldots \otimes \bx^{(d)}_{q_d}] ~\eqdef~ \Tucker{\bscrH}{\bX^{(1)}, \ldots, \bX^{(d)}},
\end{equation}
with $\bscrH \in \bbR^{r_1 \times \ldots \times r_d}$ and $\bX^{(\ell)}$ an $n_\ell \times r_\ell$ matrix with orthonormal columns $\bx^{(\ell)}_{q_\ell}$. The range of $\bX^{(\ell)}$ is the $\ell$-th singular subspace of $\bscrP$. This decomposition is illustrated for the case $d = 3$ in Figure \ref{fig:tucker}. Model \eqref{eq:model} with $\bscrP$ as in Equation \eqref{eq:P_decompositon} is the most general spiked tensor model --- i.e., low-rank perturbation of a large random tensor. Indeed, any of the models referred to in the previous Section \ref{sec:intro:related_work} fall into this definition since decomposition \eqref{eq:P_decompositon} always exists and low CPD-rank is equivalent to low multilinear rank thanks to the inequality $\max_\ell \{ r_\ell \} \leqslant R \leqslant \min_\ell \{ \prod_{\ell' \neq \ell} r_{\ell'} \}$ \citep{sidiropoulos_tensor_2017}.

\begin{figure}
\centering
\newcommand{\ra}{1.25}
\newcommand{\rb}{1.5}
\newcommand{\rc}{1}
\newcommand{\na}{2.5}
\newcommand{\nb}{3}
\newcommand{\nc}{2}
\newcommand{\dd}{0.1}

\begin{tikzpicture}

\node (P) {\begin{tikzpicture}

\draw (0, 0) -- (0, -\na) -- (\nb, -\na) -- (\nb, 0) -- (0, 0);
\draw (\nb+\nc/1.41421, -\na+\nc/1.41421) -- (\nb+\nc/1.41421, \nc/1.41421) -- (\nc/1.41421, \nc/1.41421);
\draw (\nb, -\na) -- (\nb+\nc/1.41421, -\na+\nc/1.41421);
\draw (\nb, 0) -- (\nb+\nc/1.41421, \nc/1.41421);
\draw (0, 0) -- (\nc/1.41421, \nc/1.41421);
\draw [dotted] (\nc/1.41421, \nc/1.41421) -- (\nc/1.41421, -\na+\nc/1.41421) -- (\nb+\nc/1.41421, -\na+\nc/1.41421);
\draw [dotted] (0, -\na) -- (\nc/1.41421, -\na+\nc/1.41421);
\node [anchor=west] at (0, -\na/2) {$n_1$};
\node [anchor=south] at (\nb/2, -\na) {$n_2$};
\node [anchor=east] at (\nb+\nc/1.41421/2, \nc/1.41421/2) {$n_3$};
\node [anchor=center] at (\nb/2+\nc/1.41421/2, -\na/2+\nc/1.41421/2) {$\bscrP$};

\end{tikzpicture}};

\node [anchor=west] (equal) at ([xshift=20pt] P.east) {$=$};

\node [anchor=west] (mlsvd) at ([xshift=20pt] equal.east) {\begin{tikzpicture}
\draw (0, 0) -- (0, -\ra) -- (\rb, -\ra) -- (\rb, 0) -- (0, 0);
\draw (\rb+\rc/1.41421, -\ra+\rc/1.41421) -- (\rb+\rc/1.41421, \rc/1.41421) -- (\rc/1.41421, \rc/1.41421);
\draw (\rb, -\ra) -- (\rb+\rc/1.41421, -\ra+\rc/1.41421);
\draw (\rb, 0) -- (\rb+\rc/1.41421, \rc/1.41421);
\draw (0, 0) -- (\rc/1.41421, \rc/1.41421);
\draw [dotted] (\rc/1.41421, \rc/1.41421) -- (\rc/1.41421, -\ra+\rc/1.41421) -- (\rb+\rc/1.41421, -\ra+\rc/1.41421);
\draw [dotted] (0, -\ra) -- (\rc/1.41421, -\ra+\rc/1.41421);
\node [anchor=west] at (0, -\ra/2) {$r_1$};
\node [anchor=south] at (\rb/2, -\ra) {$r_2$};
\node [anchor=east] at (\rb+\rc/1.41421/2, \rc/1.41421/2) {$r_3$};
\node [anchor=center] at (\rb/2+\rc/1.41421/2, -\ra/2+\rc/1.41421/2) {$\bscrH$};

\draw (-\dd, 0) -- (-\dd-\ra, 0) -- (-\dd-\ra, -\na) -- (-\dd, -\na) -- (-\dd, 0);
\node [anchor=north] at (-\dd-\ra/2, 0) {$r_1$};
\node [anchor=west] at (-\dd-\ra, -\na/2) {$n_1$};
\draw [decorate, decoration={brace}] (-\dd-\ra-\dd, -\na) -- (-\dd-\ra-\dd, 0) node [midway, left] {$\bX^{(1)}$};

\draw (0, -\ra-\dd) -- (0, -\ra-\dd-\rb) -- (\nb, -\ra-\dd-\rb) -- (\nb, -\ra-\dd) -- (0, -\ra-\dd);
\node [anchor=west] at (0, -\ra-\dd-\rb/2) {$r_2$};
\node [anchor=south] at (\nb/2, -\ra-\dd-\rb) {$n_2$};
\draw [decorate, decoration={brace}] (\nb, -\ra-\dd-\rb-\dd) -- (0, -\ra-\dd-\rb-\dd) node [midway, below] {$\bX^{(2)}$};

\draw (\rb+\dd, 0) -- (\rb+\dd+\rc, 0) -- (\rb+\dd+\rc+\nc/1.41421, \nc/1.41421) -- (\rb+\dd+\nc/1.41421, \nc/1.41421) -- (\rb+\dd, 0);
\node [anchor=south west] at (\rb+\dd+\rc/2, 0) {$r_3$};
\node [anchor=east] at (\rb+\dd+\rc+\nc/1.41421/2, \nc/1.41421/2) {$n_3$};
\draw [decorate, decoration={brace}] (\rb+\dd+\rc+\nc/1.41421+\dd, \nc/1.41421) -- (\rb+\dd+\rc+\dd, 0) node [midway, right] {$\bX^{(3)}$};
\end{tikzpicture}};

\end{tikzpicture}
\caption{Illustration of the Tucker decomposition \eqref{eq:P_decompositon} of an $n_1 \times n_2 \times n_3$ tensor $\bscrP$ with multilinear rank $(r_1, r_2, r_3)$. $\bscrH$ is the $r_1 \times r_2 \times r_3$ core tensor and $\bX^{(1)}, \bX^{(2)}, \bX^{(3)}$ are matrices with orthonormal columns spanning the singular subspaces of $\bscrP$.}
\label{fig:tucker}
\end{figure}

In the regime where $N \to +\infty$ --- representing the fact that, in practice, the dimensions of the tensor are large compared to its rank ---, we study the estimation of $\bscrP$ from $\bscrT$ with a truncated MLSVD, which serves as initialization of the HOOI algorithm. In particular, we reveal that the interesting \emph{non-trivial} regime is characterized by the $\bigTh(N^{\frac{d - 2}{2}})$ quantity\footnote{Note that this corresponds to the regime of the algorithmic threshold but, here, our measure of the signal power is $\norm{\bscrP}_\rmF^2$ (and not $\norm{\bscrP}_\rmF$) hence the $\bigTh(N^{\frac{d - 2}{2}})$ instead of $\bigTh(N^{\frac{d - 2}{4}})$.} $\sigma_N = \frac{1}{N} \prod_{\ell = 1}^d \sqrt{n_\ell}$.
\begin{itemize}
\item If $\norm{\bscrP}_\rmF^2 / \sigma_N \xrightarrow[N \to +\infty]{} 0$, then the noise completely masks the signal, and truncated MLSVD fails to recover $\bscrP$.
\item If $\norm{\bscrP}_\rmF^2 / \sigma_N \xrightarrow[N \to +\infty]{} +\infty$, then the signal clearly stands out from the noise, and reconstruction of $\bscrP$ with a truncated MLSVD is easy.
\item If $\norm{\bscrP}_\rmF^2 / \sigma_N = \bigTh(1)$ as $N \to +\infty$, then we are precisely in the \emph{non-trivial} regime between the two previous situations, and truncated MLSVD may partially recover $\bscrP$.
\end{itemize}
It is the analysis of this last regime which is of practical interest. Given the low-multilinear-rank approximation $\hat{\bscrT} = \tucker{\hat{\bscrG}}{\hat{\bU}^{(1)}, \ldots, \hat{\bU}^{(d)}}$ obtained with a truncated MLSVD of $\bscrT$, we quantify how well $\hat{\bscrT}$ reconstructs $\bscrP$ in this non-trivial regime. To do so, we study the spectral properties of the unfoldings (i.e., matricizations) of the tensor $\bscrT$, i.e., the $n_\ell \times \prod_{\ell' \neq \ell} n_{\ell'}$ matrices $\bT^{(\ell)}$ whose columns are mode-$\ell$ fibers of $\bscrT$ and the columns of $\hat{\bU}^{(\ell)}$ are its dominant left singular vectors. Such \emph{long} matrices (the second dimension grows faster than the first one) have already been studied by \citet{ben_arous_long_2023} in order to analyze the properties of tensor-unfolding methods in the particular setting of a rank-$1$ spike. Here, we tackle this problem with a highly different approach relying solely on classical tools from random matrix theory \citep{couillet_random_2022} and give very general results on the spiked tensor model that go beyond the specific rank-$1$ one case. Moreover, we justify the practical use of truncated MLSVD as an initialization of the HOOI algorithm by showing its optimality in the large $N$ regime.

Although the spectrum of $\bT^{(\ell)} \bT^{(\ell) \top}$, $\ell \in \{1, \ldots, d\}$, diverges as $N \to +\infty$, we show that its eigenvalues (i.e., the squared singular values of $\bT^{(\ell)}$) gather in an interval $[\mu^{(\ell)}_N \pm 2 \sigma_N]$ with $\mu^{(\ell)}_N = \frac{1}{N} \prod_{\ell' \neq \ell} n_{\ell'} = \bigTh(N^{d - 2})$. More precisely, the empirical spectral distribution of the centered-and-scaled matrix $\frac{1}{\sigma_N} [\bT^{(\ell)} \bT^{(\ell) \top} - \mu^{(\ell)}_N \bI_{n_\ell}]$ converges weakly to the semicircle distribution on $[-2, +2]$ (see Figure \ref{fig:eig}, Theorem \ref{thm:ed} and Corollary \ref{cor:lsd}). Furthermore, we show that a BBP phase transition phenomenon occurs: for each singular value of $\bP^{(\ell)}$ (the unfolding of $\bscrP$ along mode $\ell$) which is above the threshold $\sqrt{\sigma_N}$, an eigenvalue of $\bT^{(\ell)} \bT^{(\ell) \top}$ isolates itself on the right side of the \emph{bulk} (see Figure \ref{fig:eig}) and its corresponding eigenvector (i.e., left singular vector of $\bT^{(\ell)}$) is aligned with the corresponding singular subspace of $\bscrP$. The position of the isolated eigenvalue and this alignment are efficiently predicted by Theorem \ref{thm:spike} (see also Figure \ref{fig:eig}).

\begin{figure}
\centering
\input{subspace_alignments}
\caption{Alignments between singular subspaces (see Section \ref{sec:analysis:reconstruction}) of the observation $\bscrT = \sqrt{\omega} \bscrP_\circ + \frac{1}{\sqrt{N}} \bscrN$ and of the signal $\bscrP_\circ$, with $\norm{\bscrP_\circ}_\rmF^2 = \frac{\sqrt{n_1 n_2 n_3}}{N}$, as a function of the signal-to-noise ratio $\omega$. Theoretical alignments (Theorem \ref{thm:spike}) achieved with truncated MLSVD are compared with simulations and those achieved with the HOOI algorithm. Empirical results are averaged over $10$ trials, with error bars representing standard deviation. \textbf{Experimental setting:} $d  = 3$, $(n_1, n_2, n_3) = (100, 200, 300)$, $N = n_1 + n_2 + n_3$ and $(r_1, r_2, r_3) = (3, 4, 5)$.}
\label{fig:subspace_alignments}
\end{figure}

As a result, Figure \ref{fig:subspace_alignments} plots, for an order-$3$ tensor, as a function of the signal-to-noise ratio (SNR) $\omega = \norm{\bscrP}_\rmF^2 / \sigma_N$, the alignments between the singular subspace of the signal $\bscrP$ spanned by $\bX^{(\ell)}$ and the dominant singular subspace of the observation $\bscrT$ spanned by $\hat{\bU}^{(\ell)}$. Solid curves are the alignments predicted by Theorem \ref{thm:spike} while dotted curves are empirical alignments computed on a $100 \times 200 \times 300$ tensor with signal-rank $(3, 4, 5)$. If the SNR $\omega$ is too small, there is no alignment, meaning that truncated MLSVD fails to recover $\bscrP$ --- the signal is masked by the noise. When it exceeds a critical value (see Theorem \ref{thm:spike} and Section \ref{sec:analysis:reconstruction} for details), a phase transition phenomenon occurs\footnote{In fact, we will see in Section \ref{sec:analysis:reconstruction} that there is one phase transition for each principal direction of the singular subspaces of $\bscrP$, resulting in $\sum_{\ell = 1}^d r_\ell$ phase transitions. Their positions corresponds to sudden changes of slope in the solid curves of Figure \ref{fig:subspace_alignments}.}: the alignment starts to grow --- i.e., truncated MLSVD now partially recovers $\bscrP$ --- and converges to $1$ as $\omega \to +\infty$.

Besides, Figure \ref{fig:subspace_alignments} also plots the empirical alignments between the singular subspaces of $\bscrP$ and those estimated with the HOOI algorithm \citep{de_lathauwer_best_2000} given in Algorithm \ref{alg:hooi}, whose truncated MLSVD serves as initialization. This yields much better alignments, especially close to the phase transition. In fact, we show in Theorem \ref{thm:hooi} that the HOOI algorithm converges to a solution to Problem \eqref{eq:model} as soon as its initialization sufficiently preserves the underlying signal. This provides new insight into the computational barrier: initialization is the limiting factor here. Had one prior information on the solution, one could initialize the HOOI algorithm in the right basin of attraction and still be able to perfectly (i.e., with alignment $1$) reconstruct the signal in the regime $1 \ll \norm{\bscrP}_\rmF \ll N^{\frac{d - 2}{4}}$, which is computationally hard but statistically easy (see details in Section \ref{sec:best:hooi} and discussion in Section \ref{sec:best:reconstructibility}).

In a nutshell, our contributions can be summarized as follows.
\begin{itemize}
\item We characterize, in the large $N$ limit, the behavior of the singular values of the unfoldings of the tensor $\bscrT$ --- denoted $\bT^{(\ell)}$, $\ell \in \{1, \ldots, d\}$ --- when it follows the general spiked tensor model \eqref{eq:model} (Theorem \ref{thm:ed} and Corollary \ref{cor:lsd}). This is performed through the analysis of the limiting spectral distribution of the symmetric matrix $\bT^{(\ell)} \bT^{(\ell) \top}$ using standard tools from the theory of large random matrices.
\item We give a precise condition, depending on a $\bigTh(N^{\frac{d - 2}{4}})$ threshold on the signal, for the detectability of a principal direction of the $\ell$-th singular subspace of $\bscrP$ from the unfolding $\bT^{(\ell)}$. This corresponds to the presence of an isolated eigenvalue in the spectrum of $\bT^{(\ell)} \bT^{(\ell) \top}$ with associated eigenvector aligned with the sought singular subspace. We find similar formulae as \citet{feldman_spiked_2023}\footnote{\citet{feldman_spiked_2023} studies the spiked model associated with long random matrices within a similar framework as \citet{benaych-georges_eigenvalues_2011}. Our approach is different in that it mostly relies on a \emph{deterministic equivalent} which is introduced in Theorem \ref{thm:ed}.} for the asymptotic position of this isolated eigenvalue, as well as the quality of the alignment (Theorem \ref{thm:spike}).
\item Relying on our random matrix analysis, we characterize the performance of truncated MLSVD in the reconstruction of the signal $\bscrP$ from the observation $\bscrT$ (Section \ref{sec:analysis:reconstruction}).
\item We show that exact reconstruction of $\bscrP$ from $\bscrT$ is possible in the large $N$ regime with the HOOI algorithm \citep{de_lathauwer_best_2000} as long as $\norm{\bscrP}_\rmF \gg 1$ and it is initialized in the right bassin of attraction (Theorem \ref{thm:hooi}). Without prior information, this depends on the detectability of $\bscrP$ in the truncated MLSVD of $\bscrT$, which is only possible above the $\bigTh(N^{\frac{d - 2}{4}})$ computational threshold. Moreover, as $N \to +\infty$, the number of iterations needed for the convergence of the algorithm converges to $1$.
\end{itemize}

In section \ref{sec:preliminaries}, we introduce our notations, tensor-related operations and decompositions and useful tools from random matrix theory. Section \ref{sec:analysis} presents the random matrix analysis of long matrices emerging from the unfoldings of tensors following the general spiked tensor model \eqref{eq:model}. These results are presented in the context of truncated MLSVD and exploited to quantitatively explain its reconstruction performances. Then, relying on these results, Section \ref{sec:best} deals with the numerical estimation of a solution to Problem \eqref{eq:problem} with the HOOI algorithm. We show its asymptotic optimality and provide insight into the limiting factors for numerical estimation below the computational threshold. We conclude and discuss our results in Section \ref{sec:conclusion}. Most proofs are deferred to the appendix.

\section{Preliminaries on Tensors and Random Matrix Theory}
\label{sec:preliminaries}

We start by introducing some notation and the main tools which are needed to expose the results of the next sections.

\subsection{General Notations}
\label{sec:preliminaries:general}

$a$, $\ba$, $\bA$ and $\bscrA$ respectively denote a scalar, a vector, a matrix and a tensor. $a_i$, $A_{i, j}$ and $\scrA_{i_1, \ldots, i_d}$ are their entries. For $x \in \bbR$, $[x]^+ = \max(0, x)$. The imaginary part of $z \in \bbC$ is $\Im[z]$. The set $\{1, \ldots, n\}$ of positive integers smaller or equal to $n$ is denoted $[n]$. $\delta_x$ is the Dirac measure at point $x$. The support of a probability measure $\mu$ is denoted $\Supp \mu$. The notation $X \sim \calL$ means that the random variable $X$ is distributed according to the law $\calL$. Given a sequence of random variables $(X_n)_{n \geqslant 0}$, its convergence in distribution to $\calL$ is denoted $X_n \xrightarrow[n \to +\infty]{\calD} \calL$ and its almost sure convergence to $L$ is denoted $X_n \xrightarrow[n \to +\infty]{\text{a.s.}} L$. The normal distribution with mean $\mu$ and variance $\sigma^2$ is denoted $\calN(\mu, \sigma^2)$. The span of an $n_1 \times n_2$ matrix $\bA$ is $\Span \bA = \{ \bA \bx \mid \bx \in \bbR^{n_2} \} \subset \bbR^{n_1}$. The singular values of $\bA$ in \emph{non-increasing} order are denoted $s_1(\bA) \geqslant s_2(\bA) \geqslant \ldots \geqslant 0$. Given an $n \times n$ matrix $\bB$, its trace is $\Tr \bB = \sum_{i = 1}^n B_{i, i}$ and its spectrum, $\Sp \bB$, is the set of all its eigenvalues. $\norm{\cdot}$ denotes the standard Euclidean norm for vectors and the corresponding operator norm (spectral norm) for matrices.

Given two real sequences $(u_n)_{n \geqslant 0}$ and $(v_n)_{n \geqslant 0}$, we write $u_n = \bigO_{n \to +\infty}(v_n)$ if there exists a constant $C > 0$ and an integer $n_0$ such that $\lvert u_n \rvert \leqslant C \lvert v_n \rvert$ as soon as $n \geqslant n_0$. If $u_n = \bigO_{n \to +\infty}(v_n)$ and $v_n = \bigO_{n \to +\infty}(u_n)$\footnote{I.e., if there exist two constants $c, C > 0$ and an integer $n_0$ such that $c \lvert v_n \rvert \leqslant \lvert u_n \rvert \leqslant C \lvert v_n \rvert$ as soon as $n \geqslant n_0$.}, then we write $u_n = \bigTh_{n \to +\infty}(v_n)$. We also write $u_n \ll_{n \to +\infty} v_n$ (or $v_n \gg_{n \to +\infty} u_n$) if, for all $\varepsilon > 0$, there exists an integer $n_0$ such that $\lvert u_n \rvert \leqslant \varepsilon \lvert v_n \rvert$ as soon as $n \geqslant n_0$. When it is clear from context that $n \to +\infty$, we simply write $u_n = \bigTh(v_n)$, $u_n = \bigO(v_n)$, $u_n \ll v_n$ or $u_n \gg v_n$.

Unless stated otherwise, $d$ represents the order of a tensor; $\ell$ is an index ranging from $1$ to $d$ and $i_\ell$, $q_\ell$ range from $1$ to $n_\ell$, $r_\ell$ respectively.

\subsection{Tensors, Related Operations and Decompositions}
\label{sec:preliminaries:tensors}

For our purposes, tensors are considered as multi-way arrays: $\bscrT \in \bbR^{n_1 \times \ldots \times n_d}$ is a collection of elements $\scrT_{i_1, \ldots, i_d} \in \bbR$ with $i_\ell \in [n_\ell]$, $\ell \in [d]$. An \emph{$\ell$-fiber} of $\bscrT$ is the vector of $\bbR^{n_\ell}$ obtained by fixing all the indices of $\bscrT$ but the $\ell$-th. This is the generalization of columns and rows of matrices, which are respectively $1$- and $2$-fibers. \emph{Unfolding} (or matricization) is the process by which a matrix is built from a tensor --- $\bT^{(\ell)}$ is the unfolding of $\bscrT$ along mode $\ell$, i.e., the $n_\ell \times \prod_{\ell' \neq \ell} n_{\ell'}$ matrix whose columns are $\ell$-fibers of $\bscrT$\footnote{The order of the columns does not matter, as long as other operations, such as the Kronecker product, are defined in a consistent manner.}.

Given two $n_1 \times \ldots \times n_d$ tensors $\bscrA, \bscrB$, their scalar product is
\[
\Scal{\bscrA}{\bscrB}_\rmF \eqdef \sum_{i_1 = 1}^{n_1} \dots \sum_{i_d = 1}^{n_d} \scrA_{i_1, \ldots, i_d} \scrB_{i_1, \ldots, i_d}.
\]
The Frobenius norm of $\bscrA$ is $\norm{\bscrA}_\rmF \eqdef \sqrt{\scal{\bscrA}{\bscrA}_\rmF}$.

Just as a matrix is said to have rank $1$ if it can be expressed as the \emph{outer product} of two vectors, $\bx \by^\top \eqdef \bx \otimes \by$, an order-$d$ tensor is said to have rank $1$ if it can be expressed as the outer product of $d$ vectors, i.e., $\scrT_{i_1, \ldots, i_d} = x^{(1)}_{i_1} \ldots x^{(d)}_{i_d} \iff \bscrT \eqdef \bigotimes_{\ell = 1}^d \bx^{(\ell)}$.

Given two matrices $\bA$ and $\bB$ of respective sizes $n_1 \times n_2$ and $p_1 \times p_2$, their \emph{Kronecker product}, denoted $\bA \kron \bB$, is the $n_1 p_1 \times n_2 p_2$ matrix such that $[\bA \kron \bB]_{p_1 (i_1 - 1) + j_1, p_2 (i_2 - 1) + j_2} = A_{i_1, i_2} B_{j_1, j_2}$ for all $(i_1, i_2) \in [n_1] \times [n_2], (j_1, j_2) \in [p_1] \times [p_2]$. It is also defined for $n$-dimensional vectors, seen as $n \times 1$ matrices. This product is useful to express unfoldings of tensors defined as outer products. For instance,
\begin{equation} \label{eq:rank1_unfolding}
[\bx \otimes \by \otimes \bz]^{(1)} = \bx (\by \kron \bz)^\top, \quad [\bx \otimes \by \otimes \bz]^{(2)} = \by (\bx \kron \bz)^\top, \quad [\bx \otimes \by \otimes \bz]^{(3)} = \bz (\bx \kron \by)^\top.
\end{equation}
Among the various properties of the Kronecker product \citep[Chapter 10]{abadir_matrix_2005}, we highlight that it is bilinear, associative, \emph{non}-commutative, $(\bA \kron \bB)^\top = \bA^\top \kron \bB^\top$ and $(\bA \kron \bB) (\bC \kron \bD) = (\bA \bC) \kron (\bB \bD)$ when the matrix products $\bA \bC$ and $\bB \bD$ are defined.

Because of their multidimensional nature, tensors can be very large and cause computational and storage difficulties. Hence, tensor decompositions are an appropriate way to provide a sparse representation of these objects and reveal relevant knowledge from multi-way data arrays. We briefly present two of the most common decompositions, namely CPD and MLSVD, although many others exist (see, e.g., \citealp{kolda_tensor_2009, vervliet_breaking_2014}).

\paragraph{Canonical Polyadic Decomposition (CPD).} $\bscrT$ is decomposed as a sum of $R$ rank-$1$ terms, $\bscrT = \sum_{i = 1}^R \sigma_i \bigotimes_{\ell = 1}^d \ba^{(\ell)}_{i}$, with $\ba^{(\ell)}_i$ of unit norm, $\sigma_i > 0$ and minimal $R$. Introduced by \citet{hitchcock_expression_1927}, and also called CANDECOMP or PARAFAC, this decomposition is, up to permutation, essentially unique under mild conditions \citep{kolda_tensor_2009}.

\paragraph{Multilinear Singular Value Decomposition (MLSVD).} $\bscrT$ is written as the multilinear tensor-matrix product of an $r_1 \times \ldots \times r_d$ core tensor $\bscrG$ with $d$ factors $\bU^{(\ell)}$, $\ell \in [d]$, of respective sizes $n_\ell \times r_\ell$ and orthonormal columns $\bu^{(\ell)}_{q_\ell}$,
\[
\bscrT = \sum_{q_1 = 1}^{r_1} \dots \sum_{q_d = 1}^{r_d} \scrG_{q_1, \ldots, q_d} \bigotimes_{\ell = 1}^d \bu^{(\ell)}_{q_\ell} ~\eqdef~ \Tucker{\bscrG}{\bU^{(1)}, \ldots, \bU^{(d)}}.
\]
The integer $r_\ell$ is the dimension of the subspace spanned by the $\ell$-fibers of $\bscrT$, whose an orthonormal basis is formed by the columns of $\bU^{(\ell)}$. This decomposition, introduced by \citet{tucker_mathematical_1966}, is also called higher-order SVD (HOSVD, \citealp{de_lathauwer_multilinear_2000}). From this decomposition, the unfoldings of $\bscrT$ can be expressed with matrix and Kronecker products between the factor matrices and the unfoldings of the core tensor $\bscrG$. For instance, if $\bscrT = \tucker{\bscrG}{\bU, \bV, \bW}$,
\[
\bT^{(1)} = \bU \bG^{(1)} (\bV \kron \bW)^\top, \quad \bT^{(2)} = \bV \bG^{(2)} (\bU \kron \bW)^\top, \quad \bT^{(3)} = \bW \bG^{(3)} (\bU \kron \bV)^\top.
\]
Note how these expressions generalize \eqref{eq:rank1_unfolding}.

\paragraph{} The \emph{contraction} of $\bscrT$ on $\bA^{(1)}, \ldots, \bA^{(d)}$ with $\bA^{(\ell)} \in \bbR^{n_\ell \times p_\ell}$, $\ell \in [d]$, is the $p_1 \times \ldots \times p_d$ tensor $\bscrT(\bA^{(1)}, \ldots, \bA^{(d)})$ whose $(j_1, \ldots, j_d)$-entry is $\sum_{i_1 = 1}^{n_1} \dots \sum_{i_d = 1}^{n_d} \scrT_{i_1, \ldots, i_d} \prod_{\ell = 1}^d A^{(\ell)}_{i_\ell, j_\ell}$. If $\bscrT = \tucker{\bscrG}{\bU^{(1)}, \ldots, \bU^{(d)}}$ then $\bscrG = \tucker{\bscrT}{\bU^{(1) \top}, \ldots, \bU^{(d) \top}} = \bscrT(\bU^{(1)}, \ldots, \bU^{(d)})$. The tensor contraction generalizes the matrix operation $\bA^{(1) \top} \bM \bA^{(2)} = \bM(\bA^{(1)}, \bA^{(2)})$.

\begin{remark}[Uniqueness of the MLSVD up to isometries]
For all $r_\ell \times r_\ell$ orthogonal matrices $\bO^{(\ell)}$, $\ell \in [d]$, we have the equivalent decomposition
\[
\Tucker{\bscrG}{\bU^{(1)}, \ldots, \bU^{(d)}} = \Tucker{\breve{\bscrG}}{\bU^{(1)} \bO^{(1)}, \ldots, \bU^{(d)} \bO^{(d)}}
\]
where $\breve{\bscrG} = \bscrG(\bO^{(1)}, \ldots, \bO^{(d)})$ is the contraction of $\bscrG$ on $\bO^{(1)}, \ldots, \bO^{(d)}$. Nevertheless, up to isometries, the multilinear singular value decomposition is unique \citep[Property 4]{de_lathauwer_multilinear_2000}.
\end{remark}

The actual rank of a tensor is usually understood as its CPD-rank $R$ whereas its MLSVD-rank $(r_1, \ldots, r_d)$ is called \emph{multilinear rank}. These two ranks are constrained by the following inequalities,
\begin{equation} \label{eq:ineq_rank}
\max_{1 \leqslant \ell \leqslant d} r_\ell \leqslant R \leqslant \min_{1 \leqslant \ell \leqslant d} \prod_{\ell' \neq \ell} r_{\ell'},
\end{equation}
which implies one can generally speak of a ``low-rank tensor'' without having to mention the rank to which it refers.

\subsection{Tools from Random Matrix Theory}
\label{sec:preliminaries:rmt}

The results presented in this work rely on tools from the theory of large random matrices. Given a certain $n \times n$ matrix $\bM$ whose entries are random variables, one is interested in the behavior its eigenvalues and eigenvectors as $n \to +\infty$. Common questions are, in this regime, how does the \emph{empirical spectral distribution} (ESD) of $\bM$, i.e., $\frac{1}{n} \sum_{\lambda \in \Sp \bM} \delta_{\lambda}$, behave? Does it (weakly) converge to a \emph{limiting spectral distribution} (LSD)? Are there any isolated eigenvalues? If so, are the corresponding eigenvectors aligned with a relevant subspace?

In case $\bM$ is symmetric, its eigenvalues and eigenvectors are real and a key tool to answer the preceding questions is its resolvent, $\bQ_\bM(z) = \left( \bM - z \bI_n \right)^{-1}$, defined for all $z \in \bbC \setminus \Sp \bM$. Indeed, $\frac{1}{n} \Tr \bQ_\bM(z) = \frac{1}{n} \sum_{\lambda \in \Sp \bM} \frac{1}{\lambda - z}$ is the \emph{Stieltjes transform} of its empirical spectral distribution.

\begin{definition}[Stieltjes transform] \label{def:stieltjes_transform}
The Stieltjes transform of a real probability measure $\mu$ is defined for all $z \in \bbC \setminus \Supp \mu$ as $m_\mu(z) = \int_\bbR \frac{1}{t - z} \rmd \mu(t)$.
\end{definition}

The knowledge of $m_\mu$ is equivalent to the knowledge of $\mu$ thanks to the inversion formula $\mu([a, b]) = \frac{1}{\pi} \lim_{\eta \downarrow 0} \int_a^b \Im[m_\mu(x + \rmi \eta)] ~\rmd x$ for any continuity points $a, b$ of $\mu$ \citep[Theorem 2.1]{couillet_random_2022}. Hence, following the behavior of $\frac{1}{n} \Tr \bQ_\bM(z)$ as $n \to +\infty$ gives insight into the limiting spectral distribution of $\bM$ via its Stieltjes transform. To do so, we seek first a \emph{deterministic equivalent} of $\bQ_\bM$.

\begin{definition}[Deterministic equivalent] \label{def:matrix_equivalent}
Let $\bX$ be a random $n \times n$ matrix and $\bar{\bX}$ be a deterministic $n \times n$ matrix. We write $\bX \leftrightarrow \bar{\bX}$ if, for all deterministic matrices $\bA \in \bbR^{n \times n}$ and vectors $\ba, \bb \in \bbR^n$ of bounded norms (spectral and Euclidean norms respectively)\footnote{Strictly speaking, we refer to \emph{sequences} of matrices and vectors whose size grow with the index $n$. Yet, this mathematical aspect is made implicit in order to simplify notation and be closer to practical considerations where these sequences, in fact, do not exist --- the assumption $n \to +\infty$ models the fact that $n$ is large but remains finite.},
\[
\frac{1}{n} \Tr \bA \left( \bX - \bar{\bX} \right) \xrightarrow[n \to +\infty]{\text{a.s.}} 0 \qquad \text{and} \qquad \ba^\top \left( \bX - \bar{\bX} \right) \bb \xrightarrow[n \to +\infty]{\text{a.s.}} 0.
\]
The matrix $\bY$ is called a \emph{deterministic equivalent} of $\bX$. For more details on this notion, see \citet[\S 2.1.4]{couillet_random_2022}.
\end{definition}
The following lemma will be extensively used to derive such equivalents.
\begin{lemma}[\citealp{stein_estimation_1981}] \label{lem:stein}
Let $Z \sim \calN(0, 1)$ and $f : \bbR \to \bbC$ be a continuously differentiable function. When the following expectations exist, $\esp{Z f(Z)} = \esp{f'(Z)}$.
\end{lemma}

Once a deterministic equivalent $\bar{\bQ}_\bM$ of $\bQ_\bM$ is found, the limiting spectral distribution of $\bM$ is accessible through the limit of $\frac{1}{n} \Tr \bar{\bQ}_\bM$ as $n \to +\infty$. In the study of spiked models, eigenvalues are usually gathered in a \emph{bulk}, described by the limiting spectral distribution (LSD), with a finite number of isolated eigenvalues, which do not appear in the LSD (see Figure \ref{fig:eig} for example). The position of these isolated eigenvalues are singular points of the resolvent $\bQ_\bM$. Hence, the asymptotic analysis of the resolvent provides equations yielding their almost sure asymptotic position. Moreover, thanks to Cauchy's integral formula, for all $\ba \in \bbR^n$, $\lvert \ba^\top \bu \rvert^2 = \frac{-1}{2 \rmi \pi} \oint_\gamma \ba^\top \bQ_\bM(z) \ba ~\rmd z$ where $\gamma$ is a positively-oriented simple closed complex contour circling around an isolated eigenvalue (assuming it has multiplicity $1$) with all other eigenvalues left outside and $\bu$ is the eigenvector associated to the corresponding eigenvalue. The almost sure asymptotic value of this contour integral can be computed with our deterministic equivalent, thus yielding formulae for the alignments of the \emph{spike} eigenvectors with relevant directions of the model.

This briefly summarizes the main techniques we use in our proofs, see Appendix \ref{app:more_tools} for an extension. A complete presentation of these tools can be found in \citet{couillet_random_2022}. Other valuable references for the reader interested in random matrix theory are \citet{potters_first_2020, bai_spectral_2010, pastur_eigenvalue_2011, tao_topics_2012}.

\section{Analysis of Truncated MLSVD under the General Spiked Tensor Model}
\label{sec:analysis}

This section presents a random matrix analysis of the general spiked tensor model introduced in Equation \eqref{eq:model} using tools presented in Section \ref{sec:preliminaries:rmt}. We give precise results on the spectral behavior of the unfoldings of the observed tensor $\bscrT$, and specify the achievable performance in the estimation of the underlying signal $\bscrP$ with a truncated MLSVD. Although, as explained in Section \ref{sec:intro:low-rank}, this approach is only \emph{quasi-optimal}, it is very easy to implement and represents an excellent ``first guess'' to initialize a numerical scheme converging to a solution to Problem \eqref{eq:problem}, which is discussed in Section \ref{sec:best}.

\subsection{Random Matrix Results on the Model}
\label{sec:analysis:main}

Under the general spiked tensor model \eqref{eq:model}, we consider an $n_1 \times \ldots \times n_d$ tensor $\bscrT = \bscrP + \frac{1}{\sqrt{N}} \bscrN$ of order $d \geqslant 3$, modeling a low-rank deterministic signal $\bscrP$ corrupted by an additive Gaussian noise tensor $\bscrN$ whose entries are independent $\calN(0, 1)$ random variables\footnote{We highlight the fact that the Gaussian noise assumption is not restrictive. Firstly, the universality result of \citet{gurau_universality_2014} shows that, as $N \to +\infty$, the distribution of a random tensor with i.i.d.\ entries has the same limit than that of a tensor with i.i.d.\ \emph{Gaussian} entries. Moreover, our results can be extended to non-Gaussian noise up to a control on the moments of the distribution with the ``interpolation trick'' of \citet[Corollary 3.1]{lytova_central_2009}. These are technical aspects which go beyond the scope of our work. Examples of results relying on this method are \citet[Theorems 18.4.2 and 19.2.1]{pastur_eigenvalue_2011}; \citet{merlevede_universality_2015, banna_limiting_2015}.}. We denote by $(r_1, \ldots, r_d)$ the multilinear rank of $\bscrP$ and study this model in the asymptotic regime where $N \to +\infty$ with $n_\ell = \bigTh(N)$ and $r_\ell = \bigTh(1)$, $\ell \in [d]$.

The estimation of $\bscrP$ with a truncated MLSVD on $\bscrT$ is simply the computation of the dominant singular subspaces of $\bscrT$. Specifically, $\hat{\bU}^{(\ell)} \in \bbR^{n_\ell \times r_\ell}$ gathers the $r_\ell$ dominant left singular vectors of $\bT^{(\ell)}$ --- and thus, $\hat{\bU}^{(\ell) \top} \hat{\bU}^{(\ell)} = \bI_{r_\ell}$. Then, a low-multilinear-rank approximation of $\bscrT$ is $\hat{\bscrT} = \tucker{\hat{\bscrG}}{\hat{\bU}^{(1)}, \ldots, \hat{\bU}^{(d)}}$ with an $r_1 \times \ldots \times r_d$ core tensor $\hat{\bscrG} = \bscrT(\hat{\bU}^{(1)}, \ldots, \hat{\bU}^{(d)})$. An equivalent expression is $\hat{\bscrT} = \tucker{\bscrT}{\hat{\bU}^{(1)} \hat{\bU}^{(1) \top}, \ldots, \hat{\bU}^{(d)} \hat{\bU}^{(d) \top}}$, which explicitly shows that $\hat{\bscrT}$ is the projection of $\bscrT$ on its dominant singular subspaces\footnote{In the more familiar matrix case ($d = 2$), the expression $\tucker{\bT}{\hat{\bU}^{(1)} \hat{\bU}^{(1) \top}, \hat{\bU}^{(2)} \hat{\bU}^{(2) \top}}$ is equivalent to $\hat{\bU}^{(1)} \hat{\bU}^{(1) \top} \bT \hat{\bU}^{(2)} \hat{\bU}^{(2) \top}$.}. Thus, the quality of this estimation hinges upon the alignments between the singular subspaces of $\bscrP$ and the dominant singular subspaces of $\bscrT$. Namely, denoting $\hat{\bu}^{(\ell)}_{q_\ell}$, for $q_\ell \in [r_\ell]$, the columns of $\hat{\bU}^{(\ell)} = \begin{bmatrix} \hat{\bu}^{(\ell)}_1 & \dots & \hat{\bu}^{(\ell)}_{r_\ell} \end{bmatrix}$ and given that $\bscrP = \tucker{\bscrH}{\bX^{(1)}, \ldots, \bX^{(d)}}$, the quantities of interest are $\norm{\bX^{(\ell) \top} \hat{\bu}^{(\ell)}_{q_\ell}}^2$ since they represent how much of $\hat{\bu}^{(\ell)}_{q_\ell}$ is in the $\ell$-th singular subspace of the signal $\bscrP$.

In order to understand how the singular subspaces of $\bscrP$ are perturbed by the addition of noise, we study the spectral properties of the unfoldings $\bT^{(\ell)} = \bP^{(\ell)} + \frac{1}{\sqrt{N}} \bN^{(\ell)}$. In fact, since we are only interested in the \emph{left} singular vectors of $\bT^{(\ell)}$, it is more convenient to consider the $n_\ell \times n_\ell$ \emph{symmetric} matrix $\bT^{(\ell)} \bT^{(\ell) \top}$. Note that this is different from standard spiked matrix models \citep{benaych-georges_eigenvalues_2011} because the second dimension of $\bT^{(\ell)}$ grows at a faster polynomial rate than the first one ($\bigTh(N^{d - 1})$ versus $\bigTh(N)$). Hence, it is easy to see that the spectrum of $\bT^{(\ell)} \bT^{(\ell) \top}$ should diverge as $N \to +\infty$: set $\bscrP = \bzero_{n_1 \times \ldots \times n_d}$ for simplicity and consider the expected mean of the eigenvalues,
\[
\frac{1}{n_\ell} \Esp{\sum_{\lambda \in \Sp \left( \bT^{(\ell)} \bT^{(\ell) \top} \right)} \lambda} = \frac{1}{n_\ell} \Esp{\Tr \left( \frac{1}{N} \bN^{(\ell)} \bN^{(\ell) \top} \right)} = \frac{1}{N} \prod_{\ell' \neq \ell} n_{\ell'} \xrightarrow[N \to +\infty]{} +\infty.
\]
Hence, we need to consider instead a \emph{centered-and-scaled} version of our random matrix $\bT^{(\ell)} \bT^{(\ell) \top}$ to properly study the behavior of its spectrum. The quantities $\mu^{(\ell)}_N$ and $\sigma_N$ introduced in Theorem \ref{thm:ed} below are such that the eigenvalues of $\frac{1}{\sigma_N} \left[ \bT^{(\ell)} \bT^{(\ell) \top} - \mu^{(\ell)}_N \bI_{n_\ell} \right]$ neither diverge nor vanish but stay at a $\bigTh(1)$ scale as $N \to +\infty$.

\begin{theorem}[Deterministic equivalent] \label{thm:ed}
For $\ell \in [d]$, define the following quantities,
\[
\mu^{(\ell)}_N = \frac{1}{N} \prod_{\ell' \neq \ell} n_{\ell'}, \qquad \sigma_N = \frac{1}{N} \sqrt{\prod_{\ell \in [d]} n_\ell}.
\]
As $N \to +\infty$, if the ratio $\norm{\bscrP}_\rmF^2 / \sigma_N$ is bounded, then the resolvent of the centered-and-scaled matrix $\frac{1}{\sigma_N} \left[ \bT^{(\ell)} \bT^{(\ell) \top} - \mu^{(\ell)}_N \bI_{n_\ell} \right]$ has the following deterministic equivalent (Definition \ref{def:matrix_equivalent}), defined for all $\tilde{z} \in \bbC \setminus \bbR$,
\begin{multline*}
\widetilde{\bQ}^{(\ell)}(\tilde{z}) \eqdef \left( \frac{1}{\sigma_N} \left[ \bT^{(\ell)} \bT^{(\ell) \top} - \mu^{(\ell)}_N \bI_{n_\ell} \right] - \tilde{z} \bI_{n_\ell} \right)^{-1} \\ \longleftrightarrow~ \bar{\bQ}^{(\ell)}(\tilde{z}) \eqdef \left( \frac{1}{\sigma_N} \bP^{(\ell)} \bP^{(\ell) \top} + \frac{1}{\widetilde{m}(\tilde{z})} \bI_{n_\ell} \right)^{-1}
\end{multline*}
where $\displaystyle \widetilde{m}(\tilde{z}) \eqdef \lim_{N \to +\infty} \frac{1}{n_\ell} \Tr \widetilde{\bQ}^{(\ell)}(\tilde{z})$ does not depend on $\ell \in [d]$ and satisfies the following equation,
\begin{equation} \label{eq:stieltjes}
\widetilde{m}^2(\tilde{z}) + \tilde{z} \widetilde{m}(\tilde{z}) + 1 = 0.
\end{equation}
\end{theorem}
\begin{proof}
See Appendix \ref{proof:thm:ed}.
\end{proof}

This theorem is fundamental. It gives a deterministic equivalent $\bar{\bQ}^{(\ell)}(\tilde{z})$ of the resolvent of the centered-and-scaled matrix $\frac{1}{\sigma_N} \left[ \bT^{(\ell)} \bT^{(\ell) \top} - \mu^{(\ell)}_N \bI_{n_\ell} \right]$, which is our entry point into the precise characterization of its spectral behavior, following the approach presented in Section \ref{sec:preliminaries:rmt}. First of all, notice that the ``scaling parameters'' $\mu^{(\ell)}_N$ and $\sigma_N$ are respectively $\bigTh(N^{d - 2})$ and $\bigTh(N^{\frac{d - 2}{2}})$, meaning that the eigenvalues of $\bT^{(\ell)} \bT^{(\ell) \top}$ grow at a speed $N^{d - 2}$ and spread over an interval whose length grows as $\sqrt{N^{d - 2}}$ and does not depend on the mode $\ell$.

Moreover, the relation given in Equation \eqref{eq:stieltjes} characterizes $\widetilde{m}$, the Stieltjes transform of the limiting spectral distribution (LSD) of $\frac{1}{\sigma_N} \left[ \bT^{(\ell)} \bT^{(\ell) \top} - \mu^{(\ell)}_N \bI_{n_\ell} \right]$. We see that this LSD is the same regardless of the low-rank perturbation $\bP^{(\ell)} \bP^{(\ell) \top}$, as it is expected that the perturbation should only cause the presence of a \emph{finite} number of isolated eigenvalues in the empirical spectral distribution. Notice that, if $\bP^{(\ell)} = \bzero_{n_1 \times \ldots \times n_d}$, we simply have $\bar{\bQ}^{(\ell)}(\tilde{z}) = \widetilde{m}(\tilde{z}) \bI_{n_\ell}$. As a corollary, we recover the limiting spectral distribution of long random matrices, which was first characterized by \citet{bai_convergence_1988}.

\begin{corollary}[Limiting spectral distribution] \label{cor:lsd}
As $N \to +\infty$, the empirical spectral distribution of the centered-and-scaled matrix $\frac{1}{\sigma_N} \left[ \bT^{(\ell)} \bT^{(\ell) \top} - \mu^{(\ell)}_N \bI_{n_\ell} \right]$ converges weakly almost surely to $\mu_\mathrm{SC}$, the semicircle distribution on $[-2, +2]$,
\[
\rmd \mu_\mathrm{SC}(x) \eqdef \frac{1}{2 \pi} \sqrt{\left[ 4 - x^2 \right]^+} \rmd x.
\]
\end{corollary}
\begin{proof}
Following Equation \eqref{eq:stieltjes}, $\widetilde{m}(\tilde{z}) = \frac{1}{2} \left[ -z \pm \sqrt{z^2 - 4} \right]$ where the $\pm$ sign is chosen so that $\widetilde{m}$ satisfies the properties of a Stieltjes transform, in particular $\Im[\tilde{z}] \Im[\widetilde{m}(\tilde{z})] > 0$ for all $\tilde{z} \in \bbC \setminus \bbR$. Then, the result follows from the Stieltjes transform inversion formula introduced in Section \ref{sec:preliminaries:rmt} (see also \citealp[\S 2.1.2]{couillet_random_2022}).
\end{proof}

This result states that the limiting spectral distribution of $\frac{1}{\sigma_N} \left[ \bT^{(\ell)} \bT^{(\ell) \top} - \mu^{(\ell)}_N \bI_{n_\ell} \right]$ is, in fact, the very-well-known semicircle distribution, first observed by \citet{wigner_characteristic_1955, wigner_distribution_1958} in the study of certain special classes of random matrices arising in quantum mechanics. It indicates that, as $N \to +\infty$, the density of eigenvalues of $\bT^{(\ell)} \bT^{(\ell) \top}$ is a stretched semicircle on $[\mu^{(\ell)}_N \pm 2 \sigma_N]$. This phenomenon is illustrated in the first row of Figure \ref{fig:eig}, where the empirical spectral distribution (ESD) of $\bT^{(\ell)} \bT^{(\ell) \top}$ is represented with the corresponding stretched semicircle for every mode $\ell$ of an order-$3$ tensor of size $300 \times 500 \times 700$ following the general spiked tensor model \eqref{eq:model}.

\begin{figure}
\centering
\input{eig}
\caption{\textbf{Top:} empirical spectral distribution (ESD) of $\bT^{(\ell)} \bT^{(\ell) \top}$. The orange curve is the density of the stretched semicircle on $[\mu^{(\ell)}_N \pm 2 \sigma_N]$ (Corollary \ref{cor:lsd}). Green dashed lines represent asymptotic positions of spikes $\mu^{(\ell)}_N + \sigma_N \tilde{\xi}^{(\ell)}_{q_\ell}$ (Theorem \ref{thm:spike}). \textbf{Bottom:} Observed alignments between the dominant eigenvectors of $\bT^{(\ell)} \bT^{(\ell) \top}$ and $\bP^{(\ell)} \bP^{(\ell) \top}$ (purple bars) with their predicted asymptotic values $[\zeta^{(\ell)}_{q_\ell}]^+$ (red curve, Theorem \ref{thm:spike}). \textbf{Experimental setting:} $d = 3$, $(n_1, n_2, n_3) = (300, 500, 700)$, $N = n_1 + n_2 + n_3$, $(r_1, r_2, r_3) = (3, 4, 5)$ and $\norm{\bscrP}_\rmF^2 / \sigma_N = 15$.}
\label{fig:eig}
\end{figure}

\begin{remark}[From Mar\v{c}enko-Pastur to Wigner] \label{rmk:mp_wigner}
Given a random matrix $\bX \in \bbR^{p_1 \times p_2}$ with i.i.d.\ $\calN(0, \frac{1}{p_2})$ entries, it is well known that the empirical spectral distribution of $\bX \bX^\top$ converges weakly to the Mar\v{c}enko-Pastur distribution as $p_1, p_2 \to +\infty$ with $p_1 / p_2 \to c > 0$ (\citealp{marcenko_distribution_1967}; \citealp[Corollary 7.2.5]{pastur_eigenvalue_2011}; \citealp[Theorem 2.4]{couillet_random_2022}; \citealp[Chapter 4]{potters_first_2020}). On the other hand, the standard semicircle distribution $\mu_\mathrm{SC}$ is known to be the limiting spectral distribution of symmetric $p \times p$ random matrices with i.i.d.\ (up to symmetry) $\calN(0, \frac{1}{p})$ entries (\citealp[Corollary 2.2.8]{pastur_eigenvalue_2011}; \citealp[Theorem 2.5]{couillet_random_2022}; \citealp[Chapter 2]{potters_first_2020}). Here, Corollary \ref{cor:lsd} shows that if $p_2$ grows at a faster polynomial rate than $p_1$, the matrix $\bX \bX^\top$ behaves asymptotically (up to a deterministic rescaling and shift) like a Wigner matrix, even if its entries are not independent. Experimentally, we observe that, if $n_2$ and $n_3$ are chosen small compared to $n_1$ (in contradiction with our assumption $n_1, n_2, n_3 = \bigTh(N)$), e.g., $(n_1, n_2, n_3) = (1000, 40, 40)$, then the empirical spectral distribution of $\bT^{(1)} \bT^{(1) \top}$ is better modeled by a Mar\v{c}enko-Pastur distribution than by a Wigner semicircle.
\end{remark}

\begin{remark}[Confinement of the spectrum]
The weak convergence of the empirical \linebreak spectral distribution to $\mu_\mathrm{SC}$ stated in Corollary \ref{cor:lsd} could allow for a negligible amount of eigenvalues to stay outside the support of the limiting spectral distribution. In fact, in Appendix \ref{proof:thm:ed}, we show an even more precise statement on the global behavior of the eigenvalues of $\frac{1}{N} \bN^{(\ell)} \bN^{(\ell) \top}$ (the model without signal): for all $\varepsilon > 0$, there exists an integer $N_0$ such that $\Dist(\frac{1}{\sigma_N} [\lambda - \mu^{(\ell)}_N], [-2, 2]) \leqslant \varepsilon$ almost surely for all $\lambda \in \Sp \frac{1}{N} \bN^{(\ell)} \bN^{(\ell) \top}$ as soon as $N \geqslant N_0$. This means that no eigenvalue of $\frac{1}{\sigma_N} \left[ \frac{1}{N} \bN^{(\ell)} \bN^{(\ell) \top} - \mu^{(\ell)}_N \bI_{n_\ell} \right]$ stays outside the support of the semicircle distribution $[-2, 2]$ almost surely as $N \to +\infty$.
\end{remark}

The empirical spectral distributions of Figure \ref{fig:eig} also show isolated eigenvalues on the right side of each semicircle. They are caused by the low-rank perturbation $\bscrP$ which, in this setting, has multilinear rank $(3, 4, 5)$. The estimate of $\bscrP$ given by a truncated MLSVD on $\bscrT$ has its singular subspaces spanned by the dominant eigenvectors of $\bT^{(\ell)} \bT^{(\ell) \top}$, i.e., precisely those associated with these isolated eigenvalues. Hence, a precise characterization of the behavior of these \emph{spikes} is needed to plainly understand the recovery performance of this estimate. As explained in Section \ref{sec:preliminaries:rmt}, this can be achieved with the deterministic equivalent given in Theorem \ref{thm:ed}.

\begin{theorem}[Spike behavior] \label{thm:spike}
For $\ell \in [d]$ and $q_\ell \in [r_\ell]$, define the following quantities\footnote{We recall that the notation $s_i(\bA)$ denotes the $i$-th singular value of $\bA$ in non-increasing order.},
\[
\rho^{(\ell)}_{q_\ell} = \frac{s_{q_\ell}^2(\bP^{(\ell)})}{\sigma_N}, \qquad \tilde{\xi}^{(\ell)}_{q_\ell} = \rho^{(\ell)}_{q_\ell} + \frac{1}{\rho^{(\ell)}_{q_\ell}} \qquad \text{and} \qquad \zeta^{(\ell)}_{q_\ell} = 1 - \frac{1}{\left[ \rho^{(\ell)}_{q_\ell} \right]^2}.
\]
As $N \to +\infty$, if the ratio $\norm{\bscrP}_\rmF^2 / \sigma_N$ is bounded and $\rho^{(\ell)}_{q_\ell} > 1$, then
\[
\frac{1}{\sigma_N} \left[ s_{q_\ell}^2(\bT^{(\ell)}) - \mu^{(\ell)}_N \right] \xrightarrow{\text{a.s.}} \tilde{\xi}^{(\ell)}_{q_\ell} \qquad \text{and} \qquad \Norm{\bX^{(\ell) \top} \hat{\bu}^{(\ell)}_{q_\ell}}^2 \xrightarrow{\text{a.s.}} \zeta^{(\ell)}_{q_\ell}
\]
where $\hat{\bu}^{(\ell)}_{q_\ell}$ is the $q_\ell$-th dominant left singular vector of $\bT^{(\ell)}$.
\end{theorem}
\begin{proof}
See Appendix \ref{proof:thm:spike}.
\end{proof}

The first quantity defined in this theorem, $\rho^{(\ell)}_{q_\ell}$, should be understood as a signal-to-noise ratio (SNR). Indeed, the squared $q_\ell$-th singular value of $\bP^{(\ell)}$ (i.e., the $q_\ell$-th eigenvalue of $\bP^{(\ell)} \bP^{(\ell) \top}$), $s_{q_\ell}^2(\bP^{(\ell)})$, measures the ``strength'' of the signal in its $q_\ell$-th principal direction, whereas $\sigma_N$ measures the spread of the noise, as seen in Theorem \ref{thm:ed}. The two quantities $\tilde{\xi}^{(\ell)}_{q_\ell}$ and $\zeta^{(\ell)}_{q_\ell}$ depend only on the value of this SNR and indicate respectively the position of an isolated eigenvalue in the sprectrum of $\bT^{(\ell)} \bT^{(\ell) \top}$ and the alignment of the corresponding eigenvector with the sought signal. In fact, we observe a \emph{phase transition} phenomenon: if the SNR is large enough, i.e., if $\rho^{(\ell)}_{q_\ell} > 1$, an eigenvalue of $\bT^{(\ell)} \bT^{(\ell) \top}$ isolates itself from the semicircle\footnote{Indeed, note that $\rho^{(\ell)}_{q_\ell} > 1 \implies \tilde{\xi}^{(\ell)}_{q_\ell} > 2$.} around $\mu^{(\ell)}_N + \sigma_N \tilde{\xi}^{(\ell)}_{q_\ell}$. Moreover, recalling that $\bscrP = \tucker{\bscrH}{\bX^{(1)}, \ldots, \bX^{(d)}}$, the eigenvector associated with this isolated eigenvalue is aligned with the subspace spanned by $\bX^{(\ell)}$, which is the $\ell$-th singular subspace of $\bscrP$. The quality of this alignment is given by $0 < \zeta^{(\ell)}_{q_\ell} \leqslant 1$.

Most importantly, this result reveals the \emph{non-trivial regime} for the estimation of $\bscrP$ with a truncated MLSVD. Since $\sigma_N = \bigTh(N^{\frac{d - 2}{2}})$, it shows that $\norm{\bscrP}_\rmF^2 = \sum_{q_\ell = 1}^{r_\ell} s_{q_\ell}^2(\bP^{(\ell)})$ must also be of the same order. Indeed, if $\norm{\bscrP}_\rmF^2 \ll N^{\frac{d - 2}{2}}$, then $\rho^{(\ell)}_{q_\ell} \to 0$, the SNR is too small and no signal can be recovered, whereas if $\norm{\bscrP}_\rmF^2 \gg N^{\frac{d - 2}{2}}$, then $\rho^{(\ell)}_{q_\ell} \to +\infty$, the SNR is very high and recovery of $\bscrP$ is easy. It is precisely between these two regimes, i.e., $\norm{\bscrP}_\rmF = \bigTh(N^{\frac{d - 2}{4}})$, that the recovery is non-trivial. Note that this observation is in line with the results of \citet{ben_arous_long_2023} and \cite{zhang_tensor_2018}. In this non-trivial regime, the quantities $\zeta^{(\ell)}_{q_\ell}$ given in Theorem \ref{thm:spike} precisely quantify how well the dominant eigenvectors of $\bT^{(\ell)} \bT^{(\ell) \top}$ are aligned with the sought signal, i.e., the singular subspaces of $\bscrP$. In section \ref{sec:analysis:reconstruction} below, this result is used to study the reconstruction performance of truncated MLSVD.

\begin{remark}[Spiked Wigner model]
The reader familiar with the spiked Wigner model may have recognized the expressions of $\tilde{\xi}^{(\ell)}_{q_\ell}$ and $\zeta^{(\ell)}_{q_\ell}$ given in Theorem \ref{thm:spike}. Indeed, given a symmetric $p \times p$ random matrix $\bW$ with i.i.d.\ (up to symmetry) $\calN(0, \frac{1}{p})$ entries, the spectrum of $[\rho \bx \bx^\top + \bW]$ with $\norm{\bx} = 1$ follows a semicircle distribution as $p \to +\infty$ with an isolated eigenvalue at $\rho + \frac{1}{\rho}$ if, and only if, $\rho > 1$ \citep{feral_largest_2007, edwards_eigenvalue_1976, furedi_eigenvalues_1981}. Moreover, the corresponding eigenvector $\bu$ is such that $\lvert \bx^\top \bu \rvert^2 \to 1 - \frac{1}{\rho^2}$ almost surely as $p \to +\infty$ \citep{benaych-georges_eigenvalues_2011}. As discussed in Remark \ref{rmk:mp_wigner}, up to a deterministic rescaling and shift, $\bT^{(\ell)} \bT^{(\ell) \top}$ asymptotically behaves like a spiked Wigner matrix.
\end{remark}

Theorem \ref{thm:spike} is illustrated in Figure \ref{fig:eig}. In the first row, asymptotic positions of isolated eigenvalues $\mu^{(\ell)}_N + \sigma_N \tilde{\xi}^{(\ell)}_{q_\ell}$ are represented by the green dashed lines. In our experiment, $\bscrP$ has multilinear rank $(3, 4, 5)$. Hence $3$, $4$ and $5$ isolated eigenvalues are expected in the spectrum of $\bT^{(\ell)} \bT^{(\ell) \top}$ for $\ell = 1$, $2$ and $3$ respectively. This is indeed the case for $\ell = 1$ and $\ell = 2$ but not $\ell = 3$ where there are only $4$ spike eigenvalues. In fact, $s_5(\bP^{(3)})$ is not ``energetic enough'' to extricate itself from the bulk of eigenvalues, i.e, the SNR $\rho^{(3)}_5$ is below the phase transition threshold. Hence the $5$-th dominant left singular vector of $\bT^{(\ell)}$ is not informative as it is not aligned with the $3$-rd singular subspace of $\bscrP$, spanned by $\bX^{(3)}$. The second row of Figure \ref{fig:eig} depicts the alignments of the spiked eigenvectors $\hat{\bu}^{(\ell)}_{q_\ell}$ with the corresponding singular subspaces of $\bscrP$ as well as the asymptotic alignment given by Theorem \ref{thm:spike} as a function of the position of the associated eigenvalue. It appears that the higher is the SNR $\rho^{(\ell)}_{q_\ell}$, the farther is the isolated eigenvalue from the bulk and the more is the corresponding eigenvector aligned with the span of $\bX^{(\ell)}$. This assertion can be intuitively understood in terms of ``energy'' $s_{q_\ell}^2(\bP^{(\ell)})$ of a given principal direction. More energy pushes the eigenvalue farther from the bulk and aligns the corresponding eigenvector with the corresponding singular subspace of $\bscrP$.

\subsection{Reconstruction Performance of Truncated MLSVD}
\label{sec:analysis:reconstruction}

Our random matrix results allow to accurately study the reconstruction performance of truncated MLSVD. Given a data tensor $\bscrT$ following the general spiked tensor model \eqref{eq:model}, we consider its low-rank approximation $\hat{\bscrT} = \tucker{\hat{\bscrG}}{\hat{\bU}^{(1)}, \ldots, \hat{\bU}^{(d)}}$ where $\hat{\bU}^{(\ell)}$ is the $n_\ell \times r_\ell$ matrix whose columns are the $r_\ell$ dominant singular vectors of $\bscrT$ and $\hat{\bscrG} = \bscrT(\hat{\bU}^{(1)}, \ldots, \hat{\bU}^{(d)})$. $\hat{\bscrT}$ is the projection of $\bscrT$ on its dominant singular subspaces, hence the name \emph{truncated MLSVD} as it generalizes the truncated SVD of matrices. Given Model \eqref{eq:model}, the underlying signal estimated by $\hat{\bscrT}$ is $\bscrP = \tucker{\bscrH}{\bX^{(1)}, \ldots, \bX^{(d)}}$. The reconstruction performance of $\hat{\bscrT}$ hence depends on how well the \emph{subspace}\footnote{The object of importance is indeed the subspace and not the matrix $\hat{\bU}^{(\ell)}$ since any other matrix $\hat{\bU}^{(\ell)} \bO^{(\ell)}$, with $\bO^{(\ell)}$ an $r_\ell \times r_\ell$ orthogonal matrix, would span the same subspace and therefore give the same approximation.} spanned by $\hat{\bU}^{(\ell)}$ estimates the one spanned by $\bX^{(\ell)}$.

Metrics between singular subspaces are often expressed in terms of \emph{principal angles} (\citealp{bjorck_numerical_1973}; \citealp[II.4]{stewart_matrix_1990}), which generalize the concept of angle between lines. Given two subspaces (here, $\Span \bX^{(\ell)}$ and $\Span \hat{\bU}^{(\ell)}$), one can define a set of mutual angles which are invariant under isometric transformation.
\begin{definition}[Principal angles] \label{def:principal_angles}
The principal angles $\theta^{(\ell)}_{q_\ell} \in [0, \frac{\pi}{2}]$ between the subspaces spanned by $\bX^{(\ell)}$ and $\hat{\bU}^{(\ell)}$ are recursively defined for $q_\ell = 1, \ldots, r_\ell$ by
\[
\cos \theta^{(\ell)}_{q_\ell} = \bx_{q_\ell}^\top \bu_{q_\ell} \quad \text{with} \quad (\bx_{q_\ell}, \bu_{q_\ell}) \in \argmax_{\substack{(\bx, \bu) \in \Span \bX^{(\ell)} \times \Span \hat{\bU}^{(\ell)} \\ \bx^\top \bx_{q_\ell'} = 0,~ \bu^\top \bu_{q_\ell'} = 0,~  1 \leqslant q_\ell' < q_\ell}} \bx^\top \bu.
\]
\end{definition}
Moreover, we have the following useful property.
\begin{proposition}[\citealp{bjorck_numerical_1973}] \label{prop:principal_angles}
The $q_\ell$-th singular value of $\bX^{(\ell) \top} \hat{\bU}^{(\ell)}$ in non-increasing order equals the cosine of the $q_\ell$-th principal angle,
\[
s_{q_\ell}(\bX^{(\ell) \top} \hat{\bU}^{(\ell)}) = \cos \theta^{(\ell)}_{q_\ell}, \qquad \ell \in [d], \quad q_\ell \in [r_\ell].
\]
\end{proposition}

In fact, information about the alignment between the subspaces induced by $\bX^{(\ell)}$ and $\hat{\bU}^{(\ell)}$ are contained entirely in the $r_\ell \times r_\ell$ matrix $\bX^{(\ell) \top} \hat{\bU}^{(\ell)}$. Following Definition \ref{def:principal_angles}, we know from Theorem \ref{thm:spike} that, as $N \to +\infty$, $\cos^2 \theta^{(\ell)}_{q_\ell} \to [\zeta^{(\ell)}_{q_\ell}]^+$ almost surely\footnote{Indeed, since $\zeta^{(\ell)}_1 \geqslant \ldots \geqslant \zeta^{(\ell)}_{r_\ell}$, observe that, in Definition \ref{def:principal_angles}, $\lvert \bx_{q_\ell}^\top \bu_{q_\ell} \rvert^2$ is asymptotically bounded by $[\zeta^{(\ell)}_{q_\ell}]^+$.}, where we have used the handy notation $[\cdot]^+ = \max(\cdot, 0)$ since $\rho^{(\ell)}_{q_\ell} > 1 \iff \zeta^{(\ell)}_{q_\ell} > 0$. Hence, using Proposition \ref{prop:principal_angles},
\begin{equation} \label{eq:align}
\frac{1}{r_\ell} \Norm{\bX^{(\ell) \top} \hat{\bU}^{(\ell)}}_\rmF^2 = \frac{1}{r_\ell} \sum_{q_\ell = 1}^{r_\ell} \cos^2 \theta^{(\ell)}_{q_\ell} \xrightarrow[N \to +\infty]{\text{a.s.}} \frac{1}{r_\ell} \sum_{q_\ell = 1}^{r_\ell} \left[ \zeta^{(\ell)}_{q_\ell} \right]^+.
\end{equation}
Therefore, the quantity $\frac{1}{r_\ell} \norm{\bX^{(\ell) \top} \hat{\bU}^{(\ell)}}_\rmF^2 \in [0, 1]$ appears as a relevant measure of alignment between the singular subspaces of $\bscrP$ and $\hat{\bscrT}$ and does not depend on the chosen orthonormal bases $\bX^{(\ell)}$ and $\hat{\bU}^{(\ell)}$. More details on metrics between subspaces can be found in \citet[II.4]{stewart_matrix_1990}.

In Figure \ref{fig:subspace_alignments} are represented the alignments between the singular subspaces of $\hat{\bscrT}$ and $\bscrP = \sqrt{\omega} \bscrP_\circ$ with $\norm{\bscrP_\circ}^2 = \sigma_N$ as a function of the signal-to-noise ratio $\omega$. The fact that $\norm{\bscrP_\circ}^2 = \sigma_N$ ensures that the estimation problem is \emph{non-trivial} (neither too easy nor too hard) as $\rho^{(\ell)}_{q_\ell} = \bigTh(1)$. Plain curves are the alignments given by Theorem \ref{thm:spike} as $N \to +\infty$ (right-hand side of Equation \eqref{eq:align}) whereas dotted curves are simulation results at finite $N$ (left-hand side of Equation \eqref{eq:align}). In this setting, $\bscrP_\circ$ has multilinear rank $(3, 4, 5)$. Hence, its ``energy'' $\norm{\bscrP_\circ}^2$ is spread among $3$, $4$ and $5$ directions along modes $1$, $2$ and $3$ respectively. Each break in the plain curves correspond to a value of $\omega$ such that $\rho^{(\ell)}_{q_\ell} = 1$, that is, $\omega = \sigma_N / s_{q_\ell}^2(\bP_\circ^{(\ell)})$. In other words, there are $r_\ell$ phase transitions along mode $\ell$ whose positions depend on the singular values of $\bP_\circ^{(\ell)}$. If $\omega$ is too small (here, $\omega \lesssim 2$), truncated MLSVD fails to recover any direction of the singular subspaces of $\bscrP_\circ$. As $\omega$ passes the first phase transition (here, at $\omega = \sigma_N / s_1^2(\bP_\circ^{(1)}) \approx 2$), a first principal direction is partially reconstructed. Then, more and more phase transitions occur, corresponding to more and more principal directions being recovered as $\omega$ grows. Simultaneously, the reconstruction of previous directions keeps improving. Eventually, as $\omega$ is large, subspace alignments approach $1$, indicating that truncated MLSVD accurately recovers the singular subspaces of $\bscrP_\circ$.

The reader has not missed the dashed lines in Figure \ref{fig:subspace_alignments} showing much better subspace alignments than truncated MLSVD. They result from the numerical estimation of the \emph{best} rank-$(3, 4, 5)$ approximation of $\bscrT$ with the HOOI algorithm, which is discussed in the next section.

\begin{remark}[Reconstruction of $\bscrH$ from $\bscrT$]
Guarantees on the recovery of the core tensor $\bscrH$ can be deduced from Theorem \ref{thm:spike} as well. Without loss of generality, we can assume that $\bX^{(\ell) \top} \hat{\bU}^{(\ell)}$ is, up to an almost-surely vanishing additive term, a diagonal matrix with entries $\zeta^{(\ell)}_1, \ldots, \zeta^{(\ell)}_{r_\ell}$ (otherwise replace $\hat{\bU}^{(\ell)}$ by $\hat{\bU}^{(\ell)} \bO^{(\ell)}$ for a well-chosen orthogonal matrix $\bO^{(\ell)}$). Then, $\hat{\bscrG} = \bscrT(\hat{\bU}^{(1)}, \ldots, \hat{\bU}^{(d)})$ is the corresponding estimator of $\bscrH$ and
\[
\hat{\bscrG} = \hat{\bscrH} + \frac{1}{\sqrt{N}} \bscrN(\hat{\bU}^{(1)}, \ldots, \hat{\bU}^{(d)}) \quad \text{with} \quad \hat{\bscrH} = \Tucker{\bscrH}{\hat{\bU}^{(1) \top} \bX^{(1)}, \ldots, \hat{\bU}^{(d) \top} \bX^{(d)}}.
\]
We will see in Lemma \ref{lem:bound} below that $\norm{\frac{1}{\sqrt{N}} \bscrN(\hat{\bU}^{(1)}, \ldots, \hat{\bU}^{(d)})}_\rmF = \bigO(1)$ almost surely as $N \to +\infty$. On the other hand, we know that $\norm{\bscrH}_\rmF = \norm{\bscrP}_\rmF = \bigTh(N^{\frac{d - 2}{4}}) \gg \bigO(1)$ as soon as $d \geqslant 3$ and the entries of $\hat{\bscrH}$ are proportional to those of $\bscrH$:
\[
\hat{\scrH}_{q_1, \ldots, q_d} = \scrH_{q_1, \ldots, q_d} \prod_{\ell = 1}^d \zeta^{(\ell)}_{q_\ell} + \epsilon_{q_1, \ldots, q_d}, \qquad \ell \in [d], \quad q_\ell \in [r_\ell],
\]
up to an almost-surely vanishing additive term $\epsilon_{q_1, \ldots, q_d}$ as $N \to +\infty$.

Moreover, regarding the reconstruction of $\bscrT$, we know from \citet[Property 10]{de_lathauwer_multilinear_2000} that
\[
\Norm{\bscrT - \hat{\bscrT}}_\rmF^2 \leqslant \sum_{\ell = 1}^d \sum_{i_\ell = r_\ell + 1}^{n_\ell} s_{i_\ell}^2(\bT^{(\ell)}) = \sum_{\ell = 1}^d \left( \Norm{\bscrT}_\rmF^2 - \sum_{q_\ell = 1}^{r_\ell} s_{q_\ell}^2(\bT^{(\ell)}) \right)
\]
and the asymptotic behavior of $s_{q_\ell}^2(\bT^{(\ell)})$ is given by Theorem \ref{thm:spike}.
\end{remark}

\section{Numerical Estimation of the Best Low-Multilinear-Rank Approximation}
\label{sec:best}

In search of an efficient estimator of the planted signal $\bscrP$, one naturally considers the best low-multilinear-rank approximation of $\bscrT$, that is, a solution to Problem \eqref{eq:problem}. As explained in the introduction, this is NP-hard in general but numerical schemes can compute it in polynomial time above the computational threshold \citep{montanari_statistical_2014, zhang_tensor_2018}. In this section, we examine the most standard of these numerical schemes, namely the Higher Order Orthogonal Iteration (HOOI) algorithm \citep{de_lathauwer_best_2000, kroonenberg_principal_1980, kapteyn_approach_1986}, and discuss the numerical difficulties faced in the computation of a solution to Problem \eqref{eq:problem}.

\subsection{Higher-Order Orthogonal Iteration}
\label{sec:best:hooi}

Following \citet[Theorem 4.2]{de_lathauwer_best_2000}, the maximum likelihood estimation formulated in Problem \eqref{eq:problem} is equivalent to
\begin{equation} \label{eq:problem_max}
\left( \bU^{(1)}_\star, \ldots, \bU^{(d)}_\star \right) \in \argmax_{\bU^{(\ell)} \in V_{r_\ell}(\bbR^{n_\ell}),~ \ell \in [d]} \frac{1}{2} \Norm{\bscrT(\bU^{(1)}, \ldots, \bU^{(d)})}_\rmF^2
\end{equation}
where $V_{r_\ell}(\bbR^{n_\ell}) \eqdef \{\bU^{(\ell)} \in \bbR^{n_\ell \times r_\ell} \mid \bU^{(\ell) \top} \bU^{(\ell)} = \bI_{r_\ell}\}$ is the set of $r_\ell \times n_\ell$ matrices with orthonormal columns, known as the \emph{Stiefel manifold} \citep{chikuse_statistics_2003, absil_optimization_2009}. Then, since the Frobenius norm of $\bscrT(\bU^{(1)}, \ldots, \bU^{(d)})$ is equal to the Frobenius norm of any of its unfoldings,
\[
\Norm{\bscrT(\bU^{(1)}, \ldots, \bU^{(d)})}_\rmF = \Norm{\bU^{(\ell) \top} \bT^{(\ell)} \bigkron_{\ell' \neq \ell} \bU^{(\ell')}}_\rmF, \qquad \ell \in [d].
\]
And we see from Problem \eqref{eq:problem_max} that $\bU^{(\ell)}_\star$ is the matrix gathering the $r_\ell$ dominant left singular vectors of $\bT^{(\ell)} \bigkron_{\ell' \neq \ell} \bU^{(\ell')}_\star$. This is precisely what motivates the HOOI algorithm presented in Algorithm \ref{alg:hooi}. It performs fixed-point iterations to compute a solution $\bU^{(1)}_\star, \ldots, \bU^{(d)}_\star$ satisfying the previous property\footnote{In fact, this property corresponds to first-order optimality conditions of Problem \eqref{eq:problem_max}, with the squared singular values as Lagrange multipliers.}.

\begin{algorithm} \label{alg:hooi}
\DontPrintSemicolon
\caption{Higher-Order Orthogonal Iteration \citep{de_lathauwer_best_2000}}
\lFor{$\ell = 1, \ldots, d$}{$\bU^{(\ell)}_0 \gets$ $r_\ell$ dominant left singular vectors of $\bT^{(\ell)}$}
\Repeat{convergence at $t = T$}{
\lFor{$\ell = 1, \ldots, d$}{$\bU^{(\ell)}_{t + 1} \gets$ $r_\ell$ dominant left singular vectors of $\bT^{(\ell)} \bigkron_{\ell' \neq \ell} \bU^{(\ell')}_t$}
}
$\bscrG_{\text{HOOI}} \gets \bscrT(\bU^{(1)}_T, \ldots, \bU^{(d)}_T)$
\end{algorithm}

The HOOI algorithm is initialized with $\bU^{(1)}_0, \ldots, \bU^{(d)}_0$, the truncated MLSVD\footnote{Consistently with the notations of Section \ref{sec:analysis}, this is $\hat{\bU}^{(1)}, \ldots, \hat{\bU}^{(d)}$.} of $\bscrT$. Given the results of Section \ref{sec:analysis}, this is indeed a very good and easily computable first guess. Then, fixed-point iterations are repeated in order to find a solution $\bU^{(1)}_\star, \ldots, \bU^{(d)}_\star$ such that $\bU^{(\ell)}_\star$ spans the left $r_\ell$-dimensional dominant singular subspace of $\bT^{(\ell)} \bigkron_{\ell' \neq \ell} \bU^{(\ell')}_\star$ for all $\ell \in [d]$, which corresponds to a solution to Problem \eqref{eq:problem_max}. In practice, the stopping criterion can be chosen as a negligible change in the norm of the estimated core tensor, $\norm{\bscrT(\bU^{(1)}_t, \ldots, \bU^{(d)}_t)}_\rmF$.

\citet{xu_convergence_2018} showed that the convergence of this algorithm towards a local minimum of Problem \eqref{eq:problem_max} is guaranteed as long as its initialization is \emph{sufficiently close} to this local minimum. In light of our previous results, we can provide further insight into this ``sufficiently close'' property. Indeed, in Theorem \ref{thm:hooi} below, we show that an initialization with non-zero alignment with the signal $\bscrP$ is sufficient to ensure that the HOOI algorithm perfectly reconstructs it \emph{after a single iteration}.

Before introducing Theorem \ref{thm:hooi}, we formulate an important preliminary result which essentially states that, given $\bA^{(\ell)} \in V_{r_\ell}(\bbR^{n_\ell})$, $\ell \in [d]$, the quantity $\frac{1}{\sqrt{N}} \norm{\bscrN(\bA^{(1)}, \ldots, \bA^{(d)})}_\rmF$ is almost surely bounded as $N \to +\infty$.
\begin{lemma} \label{lem:bound}
With probability at least $1 - \delta$,
\begin{multline*}
\sup_{\bA^{(\ell)} \in V_{r_\ell}(\bbR^{n_\ell}),~ \ell \in [d]} \Norm{\bscrN(\bA^{(1)}, \ldots, \bA^{(d)})}_\rmF^2 \\
\leqslant 16 \left[ \left( \sum_{\ell = 1}^d r_\ell \left( n_\ell - \frac{r_\ell + 1}{2} \right) \right) \log \frac{C d}{\log \frac{3}{2}} + \log \left( \frac{1}{\delta} \max \left( 1, e^{\frac{1}{2} \prod_{\ell = 1}^d r_\ell - 1} \right) \right) \right]
\end{multline*}
where $C > 0$ is a universal constant.
\end{lemma}
\begin{proof}
See Appendix \ref{proof:lem:bound}.
\end{proof}
This result is crucial to handle the behavior of the noise in our analysis of Algorithm \ref{alg:hooi} (Appendix \ref{proof:thm:hooi}), which leads to the following result on the alignment between the singular subspaces of the signal $\bscrP$ (spanned by $\bX^{(\ell)}$) and those estimated from the observation $\bscrT$ after the first iteration of HOOI (spanned by $\bU^{(\ell)}_1$).
\begin{theorem}[Asymptotic optimality of HOOI] \label{thm:hooi}
As $N \to +\infty$, if $\norm{\bscrP}_\rmF \gg 1$ and
\[
\min_{\ell \in [d],~ q_\ell \in [r_\ell]} \Norm{\bscrP(\bU^{(1)}_0, \ldots, \bx^{(\ell)}_{q_\ell}, \ldots, \bU^{(d)}_0)}_\rmF \eqdef L_N \gg \Norm{\bscrP}_\rmF^{1 / 2}
\]
where $\bx^{(\ell)}_{q_\ell}$ is the $q_\ell$-th column of $\bX^{(\ell)}$, then,
\[
\frac{1}{r_\ell} \Norm{\bX^{(\ell) \top} \bU^{(\ell)}_1}_\rmF^2 = 1 + \bigO \left( \frac{\Norm{\bscrP}_\rmF}{L_N^2} \right) \qquad \text{almost surely.}
\]
\end{theorem}
\begin{proof}
See Appendix \ref{proof:thm:hooi}.
\end{proof}

It is important to carefully understand the assumptions of Theorem \ref{thm:hooi}. Firstly, it assumes that $\norm{\bscrP}_\rmF \gg 1$ as $N \to +\infty$, that is, the signal is not necessarily in the non-trivial regime $\bigTh(N^{\frac{d - 2}{4}})$ but can be smaller or bigger as long as $\norm{\bscrP}_\rmF \to +\infty$, regardless its speed. Then, the second assumption $\norm{\bscrP(\bU^{(1)}_0, \ldots, \bx^{(\ell)}_{q_\ell}, \ldots, \bU^{(d)}_0)}_\rmF \gg \norm{\bscrP}_\rmF^{1 / 2}$ means that each principal directions of the $\ell$-th singular subspace are sufficiently preserved after contraction on $\{ \bU^{(\ell')}_0 \}_{\ell' \neq \ell}$. When these assumptions are verified, Theorem \ref{thm:hooi} states that the matrices $\bU^{(1)}_1, \ldots, \bU^{(d)}_1$ computed after the \emph{first} iteration of HOOI \emph{perfectly} reconstruct the singular subspaces of the sought signal $\bscrP$ as the dimensions of the tensor, $n_1, \ldots, n_d$, grow large. More formally, as $N \to +\infty$, the alignment $\frac{1}{r_\ell} \norm{\bX^{(\ell) \top} \bU^{(\ell)}_1}_\rmF^2$ approaches $1$ almost surely. Furthermore, the speed of this convergence behaves like $\norm{\bscrP}_\rmF / L_N^2$.

Theorem \ref{thm:hooi} does not assume a particular choice of initialization $\bU^{(1)}_0, \ldots, \bU^{(d)}_0$ and gives a sufficient condition for it to ensure the convergence of the algorithm. Nevertheless, as it is presented in Algorithm \ref{alg:hooi}, truncated MLSVD is a standard choice of initialization. In this case, the assumption $L_N = \bigTh(\norm{\bscrP})$ is verified as soon as enough principal directions are recovered. According to Theorem \ref{thm:spike}, this is only possible if $\norm{\bscrP}_\rmF \geqslant \sqrt{\sigma_N} = \bigTh(N^{\frac{d - 2}{4}})$ since a \emph{necessary} condition is $\rho^{(\ell)}_1 > 1$ for all $\ell \in [d]$, while a \emph{sufficient} condition is $\rho^{(\ell)}_{q_\ell} > 1$ for all $\ell \in [d]$ and $q_\ell \in [d]$. In other words, convergence at speed $\norm{\bscrP}_\rmF^{-1}$ as assured above a critical signal-to-noise ratio lying between the first and the last phase transition of each mode, and which depends on the particular structure of the core tensor $\bscrH$. Yet, in most cases, this happens quite early, right after the first phase transitions, see for example Figure \ref{fig:subspace_alignments}.

We emphasize the fact that the assumption of Theorem \ref{thm:hooi} can already be verified as soon as $\min_{\ell} \rho^{(\ell)}_1 > 1$. That is, there is no need for all the principal directions to be reconstructed at initialization. In fact, it could very well be that $\max_\ell \rho^{(\ell)}_2 < 1$. If the singular subspaces of $\bscrP$ are sufficiently preserved with the initialization $\bU^{(1)}_0, \ldots, \bU^{(d)}_0$ --- i.e., if $L_N \gg \norm{\bscrP}^{1 / 2}$ ---, then the other principal directions still emerge after the first iteration.

In the simpler rank-$1$ case, these technical considerations vanish and we recover the result of \citet[Theorem 4.2]{feldman_sharp_2023}: if $\bscrP = \beta_N \bigotimes_{\ell \in [d]} \bx^{(\ell)}$ then a \emph{necessary and sufficient} condition for $L_N = \bigTh(\norm{\bscrP}_\rmF)$ is simply $\beta_N^2 > \sigma_N$. Indeed, $\rho^{(\ell)}_1 = \beta_N^2 / \sigma_N$ for all $\ell \in [d]$. Moreover, when this assumption if verified, Theorem \ref{thm:hooi} ensures the asymptotic exact reconstruction of $\bx^{(1)}, \ldots, \bx^{(d)}$ in a single iteration with a $\norm{\bscrP}_\rmF^{-1} = \beta_N^{-1}$ speed of convergence.

\begin{remark}[Practical implications]
In practice, one should still run several iterations of Algorithm \ref{alg:hooi} until a certain stopping criterion is verified as this effectively improves the final estimate and converges to a solution to the maximum likelihood estimation \eqref{eq:problem} \citep{xu_convergence_2018}. Theorem \ref{thm:hooi} states that the reconstruction performance of HOOI after the first iteration increases as we consider larger tensors, until it reaches perfect recovery in the large $N$ limit. In other words, the number of iterations required to achieve a specific level of accuracy in maximum likelihood estimation tends to $1$ as $N \to +\infty$.
\end{remark}

\begin{figure}
\centering
\input{hooi_alignment}
\caption{Alignments between singular subspaces of the observation $\bscrT = \bscrP + \frac{1}{\sqrt{N}} \bscrN$ and of the signal $\bscrP$, with $\norm{\bscrP}_\rmF^2 / \sigma_N = 10$, at initialization of Algorithm \ref{alg:hooi} (i.e., truncated MLSVD) and after the first iteration, as a function of the size of the tensor given by the parameter $N$. \textbf{Left:} $\frac{1}{r_\ell} \norm{\bX^{(\ell) \top} \bU^{(\ell)}_0}_\rmF^2$. \textbf{Middle:} $\frac{1}{r_\ell} \norm{\bX^{(\ell) \top} \bU^{(\ell)}_1}_\rmF^2$. \textbf{Right:} $(1 - \frac{1}{r_\ell} \norm{\bX^{(\ell) \top} \bU^{(\ell)}_1}_\rmF^2) \times \sqrt{\sigma_N}$. \textbf{Experimental setting:} $d = 3$, $(\frac{n_1}{N}, \frac{n_2}{N}, \frac{n_3}{N}) = (\frac{1}{6}, \frac{2}{6}, \frac{3}{6})$, $N = n_1 + n_2 + n_3$ and $(r_1, r_2, r_3) = (3, 4, 5)$.}
\label{fig:hooi_alignment}
\end{figure}

Theorem \ref{thm:hooi} is illustrated in Figure \ref{fig:hooi_alignment}. As a function of $N$ --- the size of the tensor --- we represent the subspace alignments observed at initialization and after the first iteration for a fixed signal-to-noise ratio $\norm{\bscrP}_\rmF^2 / \sigma_N = 10$. The left panel compares the observed alignments achieved with truncated MLSVD (initialization of Algorithm \ref{alg:hooi}) with the asymptotic alignments predicted by Theorem \ref{thm:spike}. As $N$ grows, the observed alignments remain around their asymptotic values, with only a decrease in variance. The middle panel presents the alignments after the first iteration. Here, as $N$ increases, we observe an increase in the values of the alignments, which approach $1$, consistently with Theorem \ref{thm:hooi}. This is specified in the right panel where the value $(1 - \frac{1}{r_\ell} \norm{\bX^{(\ell) \top} \bU^{(\ell)}_1}_\rmF^2) \times \sqrt{\sigma_N}$ is plotted. According to Theorem \ref{thm:hooi}, this value should be $\bigO(1)$ since $L_N = \bigTh(\norm{\bscrP}_\rmF) = \bigTh(\sqrt{\sigma_N})$ here. The observed behavior confirms the $\norm{\bscrP}_\rmF^{-1}$ speed of convergence asserted in Theorem \ref{thm:hooi}.

\subsection{Discussion on Signal Reconstructibility}
\label{sec:best:reconstructibility}

Our results on truncated MLSVD (Section \ref{sec:analysis}) and HOOI (Section \ref{sec:best:hooi}) bring insight into the computational-to-statistical gap observed in the low-multilinear-rank approximation problem. Truncated MLSVD can only work efficiently if $\norm{\bscrP}_\rmF$ is at least $\bigTh(N^{\frac{d - 2}{4}})$ and its reconstruction performance has been discussed in Section \ref{sec:analysis:reconstruction}. However, Theorem \ref{thm:hooi} suggests that it is possible to perfectly reconstruct the signal $\bscrP$ from the observation $\bscrT$ as long as $\norm{\bscrP}_\rmF \gg 1$ and HOOI is accurately initialized. Yet, it is known that, without prior information on $\bscrP$, maximum likelihood estimation \eqref{eq:problem} is NP-hard below the $\bigTh(N^{\frac{d - 2}{4}})$ computational threshold \citep{zhang_tensor_2018}, which lies precisely in the non-trivial regime of truncated MLSVD.

In fact, what can be understood from Theorem \ref{thm:hooi} is that it suffices to have an initialization $\bU^{(\ell)}_0$, $\ell \in [d]$, slightly aligned with the underlying signal $\bscrP$ to be in the right basin of attraction and allow the convergence of Algorithm \ref{alg:hooi} towards a solution to Problem \eqref{eq:problem}. This complements the results of \citet{xu_convergence_2018} on the conditions of convergence of HOOI. Furthermore, a solution to Problem \eqref{eq:problem} is aligned with $\bscrP$ as soon as $\norm{\bscrP}_\rmF = \bigTh(1)$ \citep{ben_arous_landscape_2019, jagannath_statistical_2020, zhang_tensor_2018}. Hence, as the HOOI algorithm is meant to compute a maximum likelihood estimator, with the assumption $\norm{\bscrP}_\rmF \gg 1$ made in Theorem \ref{thm:hooi}, it is expected that these iterations allow to perfectly recover the signal asymptotically. Maximum likelihood estimation is indeed (theoretically) trivial if $\norm{\bscrP}_\rmF \to +\infty$. It is more surprising however that this already happens at the first iteration.

As said previously, the choice of initialization $\bU^{(\ell)}_0$, $\ell \in [d]$, does not matter in Theorem \ref{thm:hooi}. In fact, without prior information, a truncated MLSVD is the best choice as it allows to partially reconstruct the signal at the $\bigTh(N^{\frac{d - 2}{4}})$ computational threshold. Nevertheless, had one prior information allowing such a reconstruction in the regime $1 \ll \bscrP \ll N^{\frac{d - 2}{4}}$ --- where truncated MLSVD would not be fruitful ---, HOOI would still be able to perfectly reconstruct the signal $\bscrP$ given this initialization.

Hence, HOOI initialized with a truncated MLSVD, as it is presented in Algorithm \ref{alg:hooi}, allows to numerically compute a maximum likelihood estimator (solution to Problem \eqref{eq:problem}) but only above the phase transition of truncated MLSVD. Indeed, its initialization plays a determining role: it must place $\bU^{(\ell)}_0$, $\ell \in [d]$, in the right basin of attraction, which, without prior information, is only possible above the $\bigTh(N^{\frac{d - 2}{4}})$ computational threshold.

Finally, we highlight the fact that these results concern the large $N$ limit. In practice, it makes no sense to talk about $\bigTh(1)$ or $\bigTh(N^{\frac{d - 2}{4}})$ regimes at finite $N$. Figure \ref{fig:subspace_alignments} also depicts the subspace alignments achieved with the maximum likelihood estimator computed with Algorithm \ref{alg:hooi} on $\bscrT = \sqrt{\omega} \bscrP_\circ + \frac{1}{\sqrt{N}} \bscrN$. Although $\norm{\bscrP_\circ}_\rmF = \sqrt{\sigma_N} = \bigTh(N^{\frac{d - 2}{4}})$, HOOI does not achieve perfect recovery of $\bscrP$ as one might expect from Theorem \ref{thm:hooi} (even if several iterations were performed here). In fact, at finite $N$, perfect recovery is not feasible. But, as $N$ grows, the dashed line would approach $1$ above the (computational) phase transition determined by the truncated MLSVD and stay close to $0$ below.

\section{Concluding Remarks}
\label{sec:conclusion}

The analysis presented in this work yields theoretical and practical insights into the estimation of a low-rank signal from an observation $\bscrT = \bscrP + \frac{1}{\sqrt{N}} \bscrN$ following the most general spiked tensor model. While \citet{zhang_tensor_2018} gave a general overview of the different regimes governing the estimation of $\bscrP$ with a low-multilinear-rank approximation of $\bscrT$ --- thereby confirming the existence of a computational-to-statistical gap ---, our results shed light on the non-trivial aspects at stake around the $\bigTh(N^{\frac{d - 2}{4}})$ computational threshold. This is of particular importance as practical applications lie in this non-trivial regime where signal and noise have the same magnitude and must be decoupled. In particular, truncated MLSVD and HOOI are very standard and efficient algorithms to compute low-multilinear-rank approximations. Performances of the latter rely strongly on the quality of its initialization, which is usually performed with a truncated MLSVD in the absence of prior information. This approach allows the detection of the underlying signal as early as the computational threshold contrary to other methods such as AMP or tensor power iteration, which are efficient above a $\bigTh(N^{\frac{d - 1}{2}})$ algorithmic threshold \citep{montanari_statistical_2014}.

Relying on standard tools and methods from the theory of large random matrices, we have characterized the spectral behavior of the unfoldings of $\bscrT$ in the large $N$ limit. Specifically, our first main result shows that, when properly centered and scaled, the eigenvalues of $\bT^{(\ell)} \bT^{(\ell) \top}$ asymptotically follow a semicircle distribution. The rescaling exhibits their mean $\mu^{(\ell)}_N = \bigTh(N^{d - 2})$ and a quantity $\sigma_N = \bigTh(N^{\frac{d - 2}{2}})$ governing their spread. From our denoising perspective, $\sigma_N$ indicates the \emph{strength} of the noise. Indeed, while the global behavior of the eigenvalues is controlled by the noise, the addition of a low-rank signal causes the presence of a finite number of eigenvalues outside the limiting semicircle distribution with corresponding eigenvectors aligned with the singular subspaces of the sought signal $\bscrP$. Yet, the existence of these outlier eigenvalues hinges on the values of the signal-to-noise ratios $\rho^{(\ell)}_{q_\ell} = s_{q_\ell}^2(\bP^{(\ell)}) / \sigma_N$, manifesting a BBP phase transition phenomenon. When they exist, the positions of these isolated eigenvalues and the quality of the corresponding alignments are completely determined by $\rho^{(\ell)}_{q_\ell}$. These results justify the use of a truncated MLSVD to estimate $\bscrP$ from the observation $\bscrT$ and allow the precise characterization of the achievable reconstruction performances in the non-trivial regime, i.e., close to the computational threshold. In particular, we have seen that each singular value of $\bP^{(\ell)}$ determines the position of a phase transition corresponding to the detectability of the corresponding principal direction.

Although truncated MLSVD does not yield the best low-multilinear-rank approximation --- i.e., a maximum likelihood solution ---, it serves as an excellent initialization for the HOOI algorithm, which converges to such an estimator if it is initialized \emph{sufficiently close} to it \citep{xu_convergence_2018}. In fact, we precise this last assertion by showing that, as long as the initialization preserves the singular subspaces of $\bscrP$ in a sense precised in Theorem \ref{thm:hooi}, HOOI converges to a maximum likelihood solution in a number of iterations which tends to $1$ as $N \to +\infty$. Hence, when it is initialized with a truncated MLSVD, it shares the same phase transition, whose position depends on the singular values $s^{(\ell)}_{q_\ell}(\bP^{(\ell)})$ of the unfoldings of $\bscrP$. Yet, given prior information, HOOI can still reconstruct the maximum likelihood solution below the computational threshold $\norm{\bscrP}_\rmF = \bigTh(N^{\frac{d - 2}{4}})$, where its success depends entirely on the quality of its initialization.

This work gives a comprehensive understanding of the low-multilinear-rank approximation problem near the computational threshold, which has both practical and theoretical implications. Besides, from a theoretical perspective, the behavior of the maximum likelihood estimator is still unclear near the statistical threshold --- that is, in the regime where $\norm{\bscrP}_\rmF = \bigTh(1)$. Several works have studied the rank-$1$ symmetric case \citep{ben_arous_landscape_2019, jagannath_statistical_2020} and the approach developed by \citet{seddik_when_2022} in their analysis of the rank-$1$ \emph{asymmetric} case may be an attractive direction to consider. Relying solely on classical tools from random matrix theory, they bring the study of the best rank-$1$ tensor approximation down to that of a structured matrix defined from the contractions of the data tensor on its dominant singular vectors. Extending this procedure to our general spiked tensor model \eqref{eq:model} presents no conceptual difficulty, despite being computationally cumbersome due to the multiple dimensions of the singular subspaces. It is an interesting line of investigation to refine our understanding of the statistical limits to low-rank tensor estimation from spiked models.

\acks{This work is supported by the MIAI LargeDATA Chair at Université Grenoble Alpes. The authors acknowledge support of the Institut Henri Poincaré (UAR 839 CNRS-Sorbonne Université), and LabEx CARMIN (ANR-10-LABX-59-01)}

\appendix

\section{More Random Matrix Tools}
\label{app:more_tools}

In this appendix we state a few useful results for our proofs in Appendices \ref{proof:thm:ed} and \ref{proof:thm:spike}. This completes the basic tools given in Section \ref{sec:preliminaries:rmt}. Let us also recall valuable references for the interested reader: \citet{bai_spectral_2010, pastur_eigenvalue_2011, tao_topics_2012, potters_first_2020, couillet_random_2022}.

\subsection{Stieltjes Transforms}

Following the definition of Stieltjes transform (Definition \ref{def:stieltjes_transform}), we state a few of its properties. They can be found in \citet[Proposition 2.1.2]{pastur_eigenvalue_2011}.
\begin{proposition}[Properties of Stieltjes transforms] \label{prop:properties_stieltjes}
Let $m$ be the Stieltjes transform of a real probability measure $\mu$.
\begin{enumerate}
\item $m$ is analytic on $\bbC \setminus \bbR$ and $m(\bar{z}) = \overline{m(z)}$;
\item if $\Im z \neq 0$ then $\Im[m(z)] \Im z > 0$;
\item for all $z \in \bbC \setminus \bbR$, $\abs{m(z)} \leqslant 1 / \abs{\Im z}$;
\item if $\mu$ has a density at $x \in \bbR$ then $\frac{\rmd \mu}{\rmd x}(x) = \lim_{y \downarrow 0} \Im[m(x + \rmi y)]$.
\end{enumerate}
\end{proposition}
Note that the resolvent $\bQ_\bS(z) \eqdef (\bS - z \bI_n)^{-1}$ of a symmetric matrix $\bS$ also satisfies $\norm{\bQ_\bS(z)} \leqslant 1 / \abs{\Im z}$.

In order to take into account the potential divergence of $m(z)$ or $\norm{\bQ_\bS(z)}$ near the real axis, we introduce the following notation.
\begin{definition}
We write $u_n(z) = \bigO[z](v_n)$ whenever there exist two polynomials $P, Q$ with nonnegative coefficients and an integer $n_0$ such that $\abs{u_n(z)} \leqslant \frac{P(\abs{z})}{\abs{\Im z} Q(\abs{\Im z})} \abs{v_n}$ as soon as $n \geqslant n_0$.
\end{definition}
In particular, $m(z) = \bigO[z](1)$ and $\norm{\bQ_\bS(z)} = \bigO[z](1)$.

\subsection{Concentration Tools}

In order to prove the concentration around its expectation of a scalar quantity, we will rely on the two following lemmas.

\begin{lemma}[Poincaré-Nash inequality] \label{lem:poincare-nash}
Let $\bz \sim \calN(\bzero_p, \bI_p)$ and $f : \bbR^p \to \bbC$ be a differentiable function with polynomially bounded partial derivatives $\partial_1 f, \ldots, \partial_p f$. Then,
\[
\Var(f(\bz)) \leqslant \Esp{\Norm{\nabla f(\bz)}^2}
\]
where $\nabla = \begin{bmatrix} \partial_1 & \ldots & \partial_p \end{bmatrix}^\top$.
\end{lemma}
\begin{proof}
See \citet[Proposition 2.1.6]{pastur_eigenvalue_2011}.
\end{proof}

\begin{lemma} \label{lem:as_convergence}
Let $(X_n)_{n \geqslant 0}$ be a sequence of random variables on $\bbC$. If $\esp{X_n} \to \bar{X} \in \bbC$ as $n \to +\infty$ and there exists an integer $k \geqslant 1$ such that $\sum_{n \geqslant 0} \esp{\lvert X_n - \esp{X_n} \rvert^k} < +\infty$, then $X_n \to \bar{X}$ almost surely as $n \to +\infty$.
\end{lemma}
\begin{proof}
Let $\varepsilon > 0$. By Markov's inequality \citep[equation 5.31]{billingsley_probability_2012},
\[
\sum_{n \geqslant 0} \bbP \left( \left\lvert X_n - \esp{X_n} \right\rvert \geqslant \varepsilon \right) \leqslant \frac{1}{\varepsilon^\kappa} \sum_{n \geqslant 0} \Esp{\left\lvert X_n - \esp{X_n} \right\rvert^\kappa} < +\infty.
\]
Thus, by the first Borel-Cantelli lemma\footnote{Given a sequence of events $(A_n)_{n \geqslant 0}$ in a probability space $(\bOmega, \calA, \bbP)$, if $\sum_{n \geqslant 0} \bbP(A_n)$ converges then $\bbP(\limsup_{n \geqslant 0} A_n) = 0$.} \citep[Theorem 4.3]{billingsley_probability_2012},
\[
\bbP \left( \limsup_{n \to +\infty} \left\lvert X_n - \esp{X_n} \right\rvert \geqslant \varepsilon \right) = 0.
\] 
This implies the almost sure convergence of $(X_n)_{n \geqslant 0}$ to $\bar{X}$:
\[
\bbP \left( \limsup_{n \to +\infty} \left\lvert X_n - \bar{X} \right\rvert \geqslant \varepsilon \right) \leqslant \bbP \left( \limsup_{n \to +\infty} \left[ \left\lvert X_n - \esp{X_n} \right\rvert + \left\lvert \Esp{X_n} - \bar{X} \right\rvert \right] \geqslant \varepsilon \right) = 0.
\]
\end{proof}

\subsection{The Helffer-Sjöstrand Formula}

A basic property of Stieltjes transforms is that, if $f : \bbR \to \bbR$ is a bounded continuous function and $m$ is the Stieltjes transform of a real probability measure, then $\int f \rmd \mu = \lim_{y \downarrow 0} \frac{1}{\pi} \int f(x) m(x + \rmi y) \rmd x$. Yet, this formula is not always satisfying because $m(x + \rmi y)$ may diverge as $y \downarrow 0$. The following powerful result gives a more pleasing formula when $f$ is smooth and compactly supported.

\begin{proposition}[Helffer-Sjöstrand formula] \label{prop:helffer-sjostrand}
Let $\mu$ be a probability measure on $\bbR$ and $m$ be its Stieltjes transform. Let $f : \bbR \to \bbR$ be a \emph{compactly supported} function which has a continuous $(k + 1)$-th derivative ($k \geqslant 1$). We define $\Phi_k[f]$, the quasi-analytic extension of $f$ on $\bbC^+ \eqdef \set{z \in \bbC \mid \Im z > 0}$ as
\[
\Phi_k[f](z) = \sum_{l = 0}^k \frac{\left( \rmi \Im z \right)^l}{l!} f^{(l)}(\Re z) \chi(\Im z)
\]
where $\chi : \bbR \to [0, 1]$ is an infinitely differentiable even function such that\footnote{An example of such function is $\chi : y \mapsto g \left( \frac{1 + y}{1 - \delta} \right) g \left( \frac{1 - y}{1 - \delta} \right)$ with $g(u) = \frac{f(u)}{f(u) + f(1 - u)}$ and $f(u) = \left\{ \begin{array}{ll} \exp(-1 / u) & \text{if}~ u > 0 \\ 0 & \text{if}~ u \leqslant 0 \end{array} \right.$.} $\chi(y) = 0$ if $\abs{y} \geqslant 1$ and $\chi(y) = 1$ if $\abs{y} \leqslant \delta$ for some $\delta \in ]0, 1[$. Then,
\[
\int_\bbR f ~\rmd \mu = \frac{2}{\pi} \Re \int_{\bbC^+} \deriv{\Phi_k[f]}{\bar{z}}(z) m(z) ~\rmd z
\]
where $\displaystyle \deriv{}{\bar{z}} = \frac{1}{2} \left(\deriv{}{\Re z} + \rmi \deriv{}{\Im z}\right)$ is the Wirtinger derivative.
\end{proposition}
\begin{proof}
The support of $\Phi_k[f]$ is compact therefore an integration by parts gives
\begin{multline*}
\frac{2}{\pi} \Re \int_{\bbC^+} \deriv{\Phi_k[f]}{\bar{z}}(z) m(z) ~\rmd z \\
\begin{aligned}
&= \begin{multlined}[t] \frac{2}{\pi} \Re \left[ \frac{1}{2} \int_0^{+\infty} \left( \int_{-\infty}^{+\infty} \deriv{\Phi_k[f]}{x}(x + \rmi y) m(x + \rmi y) ~\rmd x \right) \rmd y \right. \\ \left. + \frac{\rmi}{2} \int_{-\infty}^{+\infty} \left( \int_0^{+\infty} \deriv{\Phi_k[f]}{y}(x + \rmi y) m(x + \rmi y) ~\rmd y \right) \rmd x \right] \end{multlined} \\
&= \frac{2}{\pi} \Re \left[ \frac{-\rmi}{2} \int_\bbR \lim_{y \downarrow 0} \left\{ \Phi_k[f](x + \rmi y) m(x + \rmi y) \right\} ~\rmd x - \int_{\bbC^+} \Phi_k[f](z) \deriv{m}{\bar{z}}(z) ~\rmd z \right].
\end{aligned}
\end{multline*}
Since $m$ is an analytic function, the Cauchy-Riemann equations give $\partial m / \partial \bar{z} = 0$ and the second integral of the right-hand side is zero. Then, as the limit of a product is the product of the limits and $\lim_{y \downarrow 0} \Phi[f](x + \rmi y) = f(x)$, we have
\[
\frac{2}{\pi} \Re \int_{\bbC^+} \deriv{\Phi_k[f]}{\bar{z}}(z) m(z) ~\rmd z = \Re \left[ \frac{-\rmi}{\pi} \int_\bbR f(x) \left( \lim_{y \downarrow 0} m(x + \rmi y) \right) ~\rmd x \right].
\]
Finally, observing that $\Re[-\rmi z] = \Im z$ and $\frac{\rmd x}{\pi} \lim_{y \downarrow 0} \Im[m(x + \rmi y)] = \rmd \mu(x)$ (Proposition \ref{prop:properties_stieltjes}) concludes the proof.
\end{proof}

Since $\chi$ is constant on $]0, \delta]$, $0 < \Im z \leqslant \delta \implies \deriv{\Phi_k[f]}{\bar{z}}(z) = \frac{1}{2} \frac{\left( \rmi \Im z \right)^k}{k!} f^{(k + 1)}(\Re z)$. The greater $k$, the faster the convergence to zero of $\deriv{\Phi_k[f]}{\bar{z}}$ as $\Im z \downarrow 0$. This is particularly useful to evaluate quantities such as $\int f \rmd \mu_1 - \int f \rmd \mu_2$ from the knowledge that the Stieltjes transforms $m_1, m_2$ of $\mu_1, \mu_2$ (respectively) satisfy $\abs{m_1(z) - m_2(z)} = \bigO(\abs{\Im z}^{-k})$ as $\Im z \downarrow 0$:
\[
\int f ~\rmd \mu_1 - \int f ~\rmd \mu_2 = \frac{2}{\pi} \Re \int_{\bbC^+} \underbrace{\deriv{\Phi_k[f]}{\bar{z}}(z) \left( m_1(z) - m_2(z) \right)}_{=\bigO(1) ~\text{as}~ \Im z \downarrow 0} ~\rmd z.
\]
The smoothness of $f$ compensates for the divergence of $m_1(z) - m_2(z)$ near the real axis.

The next lemma is a handy result to evaluate some integrals which appear after using the Helffer-Sjöstrand formula.

\begin{lemma} \label{lem:stieltjes_distrib}
Let $\calK$ be a compact subset of $\bbR$ and $h$ be an analytic function on $\bbC \setminus \calK$ such that
\begin{enumerate}
\item $\lim_{\abs{z} \to +\infty} h(z) = 0$,
\item $h(\bar{z}) = \overline{h(z)}$ for all $z \in \bbC \setminus \calK$,
\item there exist an integer $n_0 \geqslant 1$ and a constant $C_\calK > 0$ such that
\[
\Abs{h(z)} \leqslant C_\calK \max(\Dist(z, \calK)^{-n_0}, 1) \qquad \text{for all}~ z \in \bbC \setminus \calK.
\]
\end{enumerate}
Then, $\lim_{y \downarrow 0} \frac{1}{\pi} \int_\bbR \Im[h(x + \rmi y)] ~\rmd x = \lim_{y \to +\infty} -\rmi y h(\rmi y)$.
\end{lemma}
\begin{proof}
According to Theorem 5.4 of \citet{schultz_non-commutative_2005} (see also \citealp[Theorem 4.3]{capitaine_largest_2009}; \citealp[Lemma 9.1]{loubaton_almost_2016}), $h$ is the Stieltjes transform of a compactly-supported Schwartz distribution \citep[Chapter 6]{rudin_functional_1991} and therefore satisfies the stated result.
\end{proof}

\section{Proof of Theorem \ref{thm:ed}}
\label{proof:thm:ed}

The resolvent of $\bT^{(\ell)} \bT^{(\ell) \top}$ is defined for all $z \in \bbC \setminus \Sp \bT^{(\ell)} \bT^{(\ell) \top}$ as
\[
\bQ^{(\ell)}(z) \eqdef \left( \bT^{(\ell)} \bT^{(\ell) \top} - z \bI_{n_\ell} \right)^{-1}.
\]
We will often drop the dependence in $z$ to simplify notations.

Since $\bQ^{(\ell) -1} \bQ^{(\ell)} = \bI_{n_\ell}$ and $\bT^{(\ell)} = \bP^{(\ell)} + \frac{1}{\sqrt{N}} \bN^{(\ell)}$, we have,
\begin{equation} \label{eq:starting_equation}
\bP^{(\ell)} \bT^{(\ell) \top} \bQ^{(\ell)} + \frac{1}{\sqrt{N}} \bN^{(\ell)} \bT^{(\ell) \top} \bQ^{(\ell)} - z \bQ^{(\ell)} = \bI_{n_\ell}.
\end{equation}

\subsection{Expressions with Stein's Lemma}

Using Stein's lemma (Lemma \ref{lem:stein}), we find the following expressions.
\begin{gather}
\Esp{\bP^{(\ell)} \bT^{(\ell) \top} \bQ^{(\ell)}} = \Esp{\bP^{(\ell)} \bP^{(\ell) \top} \bQ^{(\ell)}} - \frac{1}{N} \Esp{\bP^{(\ell)} \bT^{(\ell) \top} \bQ^{(\ell)} \Tr \bQ^{(\ell)} + \bP^{(\ell)} \bT^{(\ell) \top} \bQ^{(\ell) 2}}, \label{eq:PTQ} \\
\Esp{\bN^{(\ell)} \bT^{(\ell) \top} \bQ^{(\ell)}} = \frac{\prod_{\ell' \neq \ell} n_{\ell'}}{\sqrt{N}} \Esp{\bQ^{(\ell)}} - \frac{1}{\sqrt{N}} \Esp{\left( n_\ell + 1 \right) \bQ^{(\ell)} + z \left( \bQ^{(\ell) 2} + \bQ^{(\ell)} \Tr \bQ^{(\ell)} \right)}. \label{eq:NTQ}
\end{gather}

\paragraph{Derivatives of $\bQ^{(\ell)}$}
Firstly, we need to show that
\begin{equation} \label{eq:derivQ}
\deriv{Q^{(\ell)}_{a, b}}{N^{(\ell)}_{c, d}} = -\frac{1}{\sqrt{N}} \left( Q^{(\ell)}_{a, c} \left[ \bT^{(\ell) \top} \bQ^{(\ell)} \right]_{d, b} + \left[ \bQ^{(\ell)} \bT^{(\ell)} \right]_{a, d} Q^{(\ell)}_{c, b} \right).
\end{equation}

Indeed, using the fact that $\partial \bQ^{(\ell)} = -\bQ^{(\ell)} \partial(\bT^{(\ell)} \bT^{(\ell) \top}) \bQ^{(\ell)}$, we have,
\[
\deriv{Q^{(\ell)}_{a, b}}{N^{(\ell)}_{c, d}} = -\left[ \bQ^{(\ell)} \deriv{\bT^{(\ell)}}{N^{(\ell)}_{c, d}} \bT^{(\ell) \top} \bQ^{(\ell)} \right]_{a, b} - \left[ \bQ^{(\ell)} \bT^{(\ell)} \deriv{\bT^{(\ell) \top}}{N^{(\ell)}_{c, d}} \bQ^{(\ell)} \right]_{a, b}
\]
and, since $\bT^{(\ell)} = \bP^{(\ell)} + \frac{1}{\sqrt{N}} \bN^{(\ell)}$,
\begin{gather*}
\left[ \bQ^{(\ell)} \deriv{\bT^{(\ell)}}{N^{(\ell)}_{c, d}} \bT^{(\ell) \top} \bQ^{(\ell)} \right]_{a, b} = \frac{1}{\sqrt{N}} Q^{(\ell)}_{a, c} \left[ \bT^{(\ell) \top} \bQ^{(\ell)} \right]_{d, b}, \\
\left[ \bQ^{(\ell)} \bT^{(\ell)} \deriv{\bT^{(\ell) \top}}{N^{(\ell)}_{c, d}} \bQ^{(\ell)} \right]_{a, b} = \frac{1}{\sqrt{N}} \left[ \bQ^{(\ell)} \bT^{(\ell)} \right]_{a, d} Q^{(\ell)}_{c, b}.
\end{gather*}

\paragraph{Proof of Equation \eqref{eq:PTQ}}
Since $\bT^{(\ell)} = \bP^{(\ell)} + \frac{1}{\sqrt{N}} \bN^{(\ell)}$, we have,
\[
\esp{\bP^{(\ell)} \bT^{(\ell) \top} \bQ^{(\ell)}} = \esp{\bP^{(\ell)} \bP^{(\ell) \top} \bQ^{(\ell)}} + \frac{1}{\sqrt{N}} \esp{\bP^{(\ell)} \bN^{(\ell) \top} \bQ^{(\ell)}}.
\]
To deal with the rightmost term, we successively use Stein's lemma (Lemma \ref{lem:stein}) and Equation \eqref{eq:derivQ}.
\begin{align*}
\Esp{\bP^{(\ell)} \bN^{(\ell) \top} \bQ^{(\ell)}}_{i, j} &= \sum_{k = 1}^{\prod_{\ell' \neq \ell} n_{\ell'}} \sum_{l = 1}^{n_\ell} \Esp{P^{(\ell)}_{i, k} N^{(\ell)}_{l, k} Q^{(\ell)}_{l, j}} \\
&= \sum_{k = 1}^{\prod_{\ell' \neq \ell} n_{\ell'}} \sum_{l = 1}^{n_\ell} \Esp{P^{(\ell)}_{i, k} \deriv{Q^{(\ell)}_{l, j}}{N^{(\ell)}_{l, k}}} \\
&= \begin{multlined}[t] -\frac{1}{\sqrt{N}} \sum_{k = 1}^{\prod_{\ell' \neq \ell} n_{\ell'}} \sum_{l = 1}^{n_\ell} \Esp{P^{(\ell)}_{i, k} Q^{(\ell)}_{l, l} \left[ \bT^{(\ell) \top} \bQ^{(\ell)} \right]_{k, j}} \\ - \frac{1}{\sqrt{N}} \sum_{k = 1}^{\prod_{\ell' \neq \ell} n_{\ell'}} \sum_{l = 1}^{n_\ell} \Esp{P^{(\ell)}_{i, k} \left[ \bQ^{(\ell)} \bT^{(\ell)} \right]_{l, k} Q^{(\ell)}_{l, j}} \end{multlined} \\
&= -\frac{1}{\sqrt{N}} \Esp{\bP^{(\ell)} \bT^{(\ell) \top} \bQ^{(\ell)} \Tr \bQ^{(\ell)} + \bP^{(\ell)} \bT^{(\ell) \top} \bQ^{(\ell) 2}}_{i, j}.
\end{align*}

\paragraph{Proof of Equation \eqref{eq:NTQ}}
We proceed similarly with Stein's lemma (Lemma \ref{lem:stein}) and Equation \eqref{eq:derivQ}.
\begin{align*}
&\Esp{\bN^{(\ell)} \bT^{(\ell) \top} \bQ^{(\ell)}}_{i, j} \\
&\qquad= \sum_{k = 1}^{\prod_{\ell' \neq \ell} n_{\ell'}} \sum_{l = 1}^{n_\ell} \Esp{N^{(\ell)}_{i, k} T^{(\ell)}_{l, k} Q^{(\ell)}_{l, j}} \\
&\qquad= \sum_{k = 1}^{\prod_{\ell' \neq \ell} n_{\ell'}} \sum_{l = 1}^{n_\ell} \Esp{\deriv{T^{(\ell)}_{l, k}}{N^{(\ell)}_{i, k}} Q^{(\ell)}_{l, j} + T^{(\ell)}_{l, k} \deriv{Q^{(\ell)}_{l, j}}{N^{(\ell)}_{i, k}}} \\
&\qquad= \begin{multlined}[t] \frac{\prod_{\ell' \neq \ell} n_{\ell'}}{\sqrt{N}} \Esp{\bQ^{(\ell)}}_{i, j} - \frac{1}{\sqrt{N}} \sum_{k = 1}^{\prod_{\ell' \neq \ell} n_{\ell'}} \sum_{l = 1}^{n_\ell} \Esp{T^{(\ell)}_{l, k} Q^{(\ell)}_{l, i} \left[ \bT^{(\ell) \top} \bQ^{(\ell)} \right]_{k, j}} \\ - \frac{1}{\sqrt{N}} \sum_{k = 1}^{\prod_{\ell' \neq \ell} n_{\ell'}} \sum_{l = 1}^{n_\ell} \Esp{T^{(\ell)}_{l, k} \left[ \bQ^{(\ell)} \bT^{(\ell)} \right]_{l, k} Q^{(\ell)}_{i, j}} \end{multlined} \\
&\qquad= \frac{\prod_{\ell' \neq \ell} n_{\ell'}}{\sqrt{N}} \Esp{\bQ^{(\ell)}}_{i, j} - \frac{1}{\sqrt{N}} \Esp{\bQ^{(\ell)} \bT^{(\ell)} \bT^{(\ell) \top} \bQ^{(\ell)} + \bQ^{(\ell)} \Tr \bT^{(\ell)} \bT^{(\ell) \top} \bQ^{(\ell)}}_{i, j}.
\end{align*}
Since $\bT^{(\ell)} \bT^{(\ell) \top} \bQ^{(\ell)} - z \bQ^{(\ell)} = \bI_{n_\ell}$, we find the result stated in Equation \eqref{eq:NTQ}.

\subsection{Asymptotic Behavior of the Resolvent}

Taking the expectation of Equation \eqref{eq:starting_equation} and injecting Equation \eqref{eq:NTQ} yields
\begin{multline*}
\Esp{\bP^{(\ell)} \bT^{(\ell) \top} \bQ^{(\ell)}} + \frac{\prod_{\ell' \neq \ell} n_{\ell'}}{N} \Esp{\bQ^{(\ell)}} \\
- \frac{1}{N} \Esp{\left( n_\ell + 1 \right) \bQ^{(\ell)} + z \left( \bQ^{(\ell) 2} + \bQ^{(\ell)} \Tr \bQ^{(\ell)} \right)} - z \Esp{\bQ^{(\ell)}} = \bI_{n_\ell}.
\end{multline*}
We rearrange this expression into the more convenient following form,
\begin{multline*}
z \frac{n_\ell}{N} \Esp{\bQ^{(\ell)} \frac{\Tr \bQ^{(\ell)}}{n_\ell}} + \left( z + \frac{n_\ell - \prod_{\ell' \neq \ell} n_{\ell'}}{N} \right) \Esp{\bQ^{(\ell)}} + \bI_{n_\ell} \\
= -\frac{1}{N} \Esp{\bQ^{(\ell)} + z \bQ^{(\ell) 2}} + \Esp{\bP^{(\ell)} \bT^{(\ell) \top} \bQ^{(\ell)}}.
\end{multline*}
Here, the divergence of the spectrum of $\bT^{(\ell)} \bT^{(\ell) \top}$ becomes problematic: its resolvent $\bQ^{(\ell)}$ vanishes asymptotically, allowing the presence of the diverging coefficient $\frac{1}{N} \prod_{\ell' \neq \ell} n_{\ell'}$ in the previous equation. To bypass this difficulty, we proceed to a rescaling,
\begin{gather*}
\tilde{z} \eqdef \frac{z - \mu^{(\ell)}_N}{\sigma_N}, \qquad \widetilde{\bQ}^{(\ell)}(\tilde{z}) \eqdef \left( \frac{\bT^{(\ell)} \bT^{(\ell) \top} - \mu^{(\ell)}_N \bI_{n_\ell}}{\sigma_N} - \tilde{z} \bI_{n_\ell} \right)^{-1} = \sigma_N \bQ^{(\ell)}(z), \\
\text{with} \qquad \mu^{(\ell)}_N = \frac{1}{N} \prod_{\ell' \neq \ell} n_{\ell'}, \qquad \sigma_N = \frac{1}{N} \sqrt{\prod_{\ell \in [d]} n_\ell}.
\end{gather*}
This changes our equation into
\begin{multline*}
\frac{\mu^{(\ell)}_N + \sigma_N \tilde{z}}{\sigma_N^2} \frac{n_\ell}{N} \Esp{\widetilde{\bQ}^{(\ell)} \frac{\Tr \widetilde{\bQ}^{(\ell)}}{n_\ell}} + \left( \frac{\mu^{(\ell)}_N}{\sigma_N} + \tilde{z} + \frac{n_\ell - \prod_{\ell' \neq \ell} n_{\ell'}}{\sigma_N N} \right) \Esp{\widetilde{\bQ}^{(\ell)}} + \bI_{n_\ell} \\
= -\frac{1}{\sigma_N N} \Esp{\widetilde{\bQ}^{(\ell)} + \left( \frac{\mu^{(\ell)}_N}{\sigma_N} + \tilde{z} \right) \widetilde{\bQ}^{(\ell) 2}} + \frac{1}{\sigma_N} \Esp{\bP^{(\ell)} \bT^{(\ell) \top} \widetilde{\bQ}^{(\ell)}}.
\end{multline*}
Let us define $\widetilde{m}^{(\ell)}_N : \tilde{z} \mapsto \frac{1}{n_\ell} \Tr \widetilde{\bQ}^{(\ell)}(\tilde{z})$, the Stieltjes transform of the empirical spectral distribution of $\frac{1}{\sigma_N} \left[ \bT^{(\ell)} \bT^{(\ell) \top} - \mu^{(\ell)}_N \bI_{n_\ell} \right]$. With this definition, and using the fact that $\frac{\mu^{(\ell)}_N}{\sigma_N^2} \frac{n_\ell}{N} = 1$, we can rewrite the previous relation as
\begin{multline} \label{eq:EQ}
\Esp{\widetilde{m}^{(\ell)}_N(\tilde{z}) \widetilde{\bQ}^{(\ell)}} + \tilde{z} \Esp{\widetilde{\bQ}^{(\ell)}} + \bI_{n_\ell} - \frac{1}{\sigma_N} \Esp{\bP^{(\ell)} \bT^{(\ell) \top} \widetilde{\bQ}^{(\ell)}} \\
= -\frac{1}{\sigma_N} \frac{n_\ell}{N} \Esp{\left( \tilde{z} \widetilde{m}^{(\ell)}_N(\tilde{z}) + 1 \right) \widetilde{\bQ}^{(\ell)}} - \frac{1}{N} \left( \frac{\mu^{(\ell)}_N}{\sigma_N^2} + \frac{\tilde{z}}{\sigma_N} \right) \Esp{\widetilde{\bQ}^{(\ell) 2}} - \frac{1}{\sigma_N} \frac{1}{N} \Esp{\widetilde{\bQ}^{(\ell)}}
\end{multline}
where we have kept only the non-vanishing terms on the left-hand side.

\subsection{Concentration of Bilinear Forms and Traces}
\label{proof:thm:ed:concentration}

\subsubsection{Overview}

Let us show that $\esp{\widetilde{\bQ}^{(\ell)}}$ is a deterministic equivalent of $\widetilde{\bQ}^{(\ell)}$ (Definition \ref{def:matrix_equivalent}). That is, for all bounded (sequences of) vectors $\ba, \bb \in \bbR^{n_\ell}$ and matrices $\bA \in \bbR^{n_\ell \times n_\ell}$,
\[
\ba^\top \left( \widetilde{\bQ}^{(\ell)} - \Esp{\widetilde{\bQ}^{(\ell)}} \right) \bb \xrightarrow[N \to +\infty]{\text{a.s.}} 0 \qquad \text{and} \qquad \frac{1}{n_\ell} \Tr \bA \left( \widetilde{\bQ}^{(\ell)} - \Esp{\widetilde{\bQ}^{(\ell)}} \right) \xrightarrow[N \to +\infty]{\text{a.s.}} 0.
\]
This result is derived with the Poincaré-Nash inequality (Lemma \ref{lem:poincare-nash}) and Lemma \ref{lem:as_convergence}. Indeed, relying on Lemma \ref{lem:poincare-nash} and Equation \eqref{eq:derivQ}, we find below that
\begin{gather*}
\Esp{\left\lvert \Tr \bA \left( \widetilde{\bQ}^{(\ell)}(\tilde{z}) - \Esp{\widetilde{\bQ}^{(\ell)}(\tilde{z})} \right) \right\rvert^{2}} = \bigO[\tilde{z}](N^{-2}), \\
\Esp{\left\lvert \ba^\top \left( \widetilde{\bQ}^{(\ell)}(\tilde{z}) - \Esp{\widetilde{\bQ}^{(\ell)}(\tilde{z})} \right) \bb \right\rvert^{4}} = \bigO[\tilde{z}](N^{-2})
\end{gather*}
for all deterministic matrix $\bA \in \bbR^{n_\ell \times n_\ell}$ and vectors $\ba, \bb \in \bbR^{n_\ell}$ of bounded norms (spectral and Euclidean norms respectively). By Lemma \ref{lem:as_convergence}, this implies that
\[
\Tr \bA \left( \widetilde{\bQ}^{(\ell)}(\tilde{z}) - \Esp{\widetilde{\bQ}^{(\ell)}(\tilde{z})} \right) \xrightarrow[N \to +\infty]{\text{a.s.}} 0 \quad \text{and} \quad \ba^\top \left( \widetilde{\bQ}^{(\ell)}(\tilde{z}) - \Esp{\widetilde{\bQ}^{(\ell)}(\tilde{z})} \right) \bb \xrightarrow[N \to +\infty]{\text{a.s.}} 0.
\]
Hence, $\esp{\widetilde{\bQ}^{(\ell)}(\tilde{z})}$ is a deterministic equivalent of $\widetilde{\bQ}^{(\ell)}(\tilde{z})$ according to Definition \ref{def:matrix_equivalent}.

\subsubsection{Proof of the Concentration}

\paragraph{Concentration of Traces}

Let $\bA \in \bbR^{n_\ell \times n_\ell}$ be a deterministic matrix with bounded norm. The Poincaré-Nash inequality (Lemma \ref{lem:poincare-nash}) yields
\begin{multline*}
\Esp{\left( \frac{1}{n_\ell} \Tr \bA \left( \widetilde{\bQ}^{(\ell)} - \Esp{\widetilde{\bQ}^{(\ell)}} \right) \right)^2} = \Var \left( \frac{1}{n_\ell} \Tr \bA \widetilde{\bQ}^{(\ell)} \right) \\
\leqslant \frac{1}{n_\ell^2} \sum_{i = 1}^{n_\ell} \sum_{j = 1}^{\prod_{\ell' \neq \ell} n_{\ell'}} \Esp{\left( \deriv{\Tr \bA \widetilde{\bQ}^{(\ell)}}{N^{(\ell)}_{i, j}} \right)^2}.
\end{multline*}
With Equation \eqref{eq:derivQ}, we have,
\begin{align*}
\deriv{\Tr \bA \widetilde{\bQ}^{(\ell)}}{N^{(\ell)}_{i, j}} &= \sum_{k = 1}^{n_\ell} \sum_{l = 1}^{n_\ell} A_{k, l} \deriv{\widetilde{Q}^{(\ell)}_{l, k}}{N^{(\ell)}_{i, j}} \\
&= -\frac{1}{\sigma_N \sqrt{N}} \sum_{k = 1}^{n_\ell} \sum_{l = 1}^{n_\ell} A_{k, l} \left[ \widetilde{Q}^{(\ell)}_{l, i} [\bT^{(\ell) \top} \widetilde{\bQ}^{(\ell)}]_{j, k} + [\widetilde{\bQ}^{(\ell)} \bT^{(\ell)}]_{l, j} \widetilde{Q}^{(\ell)}_{i, k} \right] \\
&= -\frac{1}{\sigma_N \sqrt{N}} \left[ \widetilde{\bQ}^{(\ell)} \bA^\top \widetilde{\bQ}^{(\ell)} \bT^{(\ell)} + \widetilde{\bQ}^{(\ell)} \bA \widetilde{\bQ}^{(\ell)} \bT^{(\ell)} \right]_{i, j}.
\end{align*}
Thus, using the fact that $(a + b)^2 \leqslant 2 (a^2 + b^2)$, we have,
\begin{multline*}
\frac{1}{n_\ell^2} \sum_{i = 1}^{n_\ell} \sum_{j = 1}^{\prod_{\ell' \neq \ell} n_{\ell'}} \Esp{\left( \deriv{\Tr \bA \widetilde{\bQ}^{(\ell)}}{N^{(\ell)}_{i, j}} \right)^2} \\
\leqslant \frac{2}{\sigma_N^2 N n_\ell^2} \sum_{i = 1}^{n_\ell} \sum_{j = 1}^{\prod_{\ell' \neq \ell} n_{\ell'}} \Esp{[\widetilde{\bQ}^{(\ell)} \bA^\top \widetilde{\bQ}^{(\ell)} \bT^{(\ell)}]_{i, j}^2 + [\widetilde{\bQ}^{(\ell)} \bA \widetilde{\bQ}^{(\ell)} \bT^{(\ell)}]_{i, j}^2} \\
= \frac{2}{\sigma_N^2 N n_\ell^2} \Esp{\Norm{\widetilde{\bQ}^{(\ell)} \bA^\top \widetilde{\bQ}^{(\ell)} \bT^{(\ell)}}_\rmF^2 + \Norm{\widetilde{\bQ}^{(\ell)} \bA \widetilde{\bQ}^{(\ell)} \bT^{(\ell)}}_\rmF^2} \\
\leqslant \frac{4}{\sigma_N^2 N n_\ell^2} \Norm{\bA}_\rmF^2 \Esp{\Norm{\widetilde{\bQ}^{(\ell)}}^4 \Norm{\bT^{(\ell)}}^2}
\end{multline*}
where the last inequality is obtained using the property $\norm{\bA \bB}_\rmF \leqslant \norm{\bA} \norm{\bB}_\rmF$ and the sub-multiplicative property $\norm{\bA \bB} \leqslant \norm{\bA} \norm{\bB}$. Since $\norm{\bA}_\rmF^2 \leqslant n_\ell \norm{\bA}^2 = \bigO(N)$ (because $n_\ell = \bigTh(N)$ and $\norm{\bA}$ is bounded), $\norm{\widetilde{\bQ}^{(\ell)}} = \bigO[\tilde{z}](1)$ and $\frac{1}{\sigma_N^2} \norm{\bT^{(\ell)}}^2 = \bigO(1)$, we find that $\Var(\frac{1}{n_\ell} \Tr \bA \widetilde{\bQ}^{(\ell)}) = \bigO[\tilde{z}](N^{-2})$; Q.E.D.

\paragraph{Concentration of Bilinear Forms}

Let $\ba, \bb \in \bbR^{n_\ell}$ be two deterministic vectors with bounded norms. The Poincaré-Nash inequality (Lemma \ref{lem:poincare-nash}) yields
\[
\Esp{\left( \ba^\top \left( \widetilde{\bQ}^{(\ell)} - \Esp{\widetilde{\bQ}^{(\ell)}} \right) \bb \right)^2} = \Var \left( \ba^\top \widetilde{\bQ}^{(\ell)} \bb \right) \leqslant \sum_{i = 1}^{n_\ell} \sum_{j = 1}^{\prod_{\ell' \neq \ell} n_{\ell'}} \Esp{\left( \deriv{\ba^\top \widetilde{\bQ}^{(\ell)} \bb}{N^{(\ell)}_{i, j}} \right)^2}.
\]
With Equation \eqref{eq:derivQ}, we have,
\begin{align*}
\deriv{\ba^\top \widetilde{\bQ}^{(\ell)} \bb}{N^{(\ell)}_{i, j}} &= \sum_{k = 1}^{n_\ell} \sum_{l = 1}^{n_\ell} a_k \deriv{\widetilde{Q}^{(\ell)}_{k, l}}{N^{(\ell)}_{i, j}} b_l \\
&= -\frac{1}{\sigma_N \sqrt{N}} \sum_{k = 1}^{n_\ell} \sum_{l = 1}^{n_\ell} a_k  \left[ \widetilde{Q}^{(\ell)}_{k, i} [\bT^{(\ell) \top} \widetilde{\bQ}^{(\ell)}]_{j, l} + [\widetilde{\bQ}^{(\ell)} \bT^{(\ell)}]_{k, j} \widetilde{Q}^{(\ell)}_{i, l} \right] b_l \\
&= -\frac{1}{\sigma_N \sqrt{N}} \left[ \widetilde{\bQ}^{(\ell)} \ba \bb^\top \widetilde{\bQ}^{(\ell)} \bT^{(\ell)} + \widetilde{\bQ}^{(\ell)} \bb \ba^\top \widetilde{\bQ}^{(\ell)} \bT^{(\ell)} \right]_{i, j}.
\end{align*}
Thus, using the fact that $(a + b)^2 \leqslant 2 (a^2 + b^2)$, we have,
\begin{multline*}
\sum_{i = 1}^{n_\ell} \sum_{j = 1}^{\prod_{\ell' \neq \ell} n_{\ell'}} \Esp{\left( \deriv{\ba^\top \widetilde{\bQ}^{(\ell)} \bb}{N^{(\ell)}_{i, j}} \right)^2} \\
\leqslant \frac{2}{\sigma_N^2 N} \sum_{i = 1}^{n_\ell} \sum_{j = 1}^{\prod_{\ell' \neq \ell} n_{\ell'}} \Esp{[\widetilde{\bQ}^{(\ell)} \ba \bb^\top \widetilde{\bQ}^{(\ell)} \bT^{(\ell)}]_{i, j}^2 + [\widetilde{\bQ}^{(\ell)} \bb \ba^\top \widetilde{\bQ}^{(\ell)} \bT^{(\ell)}]_{i, j}^2} \\
= \frac{2}{\sigma_N^2 N} \Esp{\Norm{\widetilde{\bQ}^{(\ell)} \ba \bb^\top \widetilde{\bQ}^{(\ell)} \bT^{(\ell)}}_\rmF^2 + \Norm{\widetilde{\bQ}^{(\ell)} \bb \ba^\top \widetilde{\bQ}^{(\ell)} \bT^{(\ell)}}_\rmF^2} \\
\leqslant \frac{2}{\sigma_N^2 N} \Norm{\ba}^2 \Norm{\bb}^2 \Esp{\Norm{\widetilde{\bQ}^{(\ell)}}^2 \Norm{\bT^{(\ell)}}^2} = \bigO[\tilde{z}](N^{-1})
\end{multline*}
where the last line is obtained similarly to the previous case. Hence, $\Var(\ba^\top \widetilde{\bQ}^{(\ell)} \bb) = \bigO[\tilde{z}](N^{-1})$, which is not enough to apply Lemma \ref{lem:as_convergence}. We must therefore evaluate the moment of order $4$,
\begin{multline*}
\Esp{\left( \ba^\top \left( \widetilde{\bQ}^{(\ell)} - \Esp{\widetilde{\bQ}^{(\ell)}} \right) \bb \right)^4} = \Var \left( \left( \ba^\top \left( \widetilde{\bQ}^{(\ell)} - \Esp{\widetilde{\bQ}^{(\ell)}} \right) \bb \right)^2 \right) \\
+ \Esp{\left( \ba^\top \left( \widetilde{\bQ}^{(\ell)} - \Esp{\widetilde{\bQ}^{(\ell)}} \right) \bb \right)^2}^2.
\end{multline*}
Since we already known that the rightmost term is $\bigO(N^{-2})$, we use the Poincaré-Nash inequality (Lemma \ref{lem:poincare-nash}) on the variance of $( \ba^\top (\widetilde{\bQ}^{(\ell)} - \esp{\widetilde{\bQ}^{(\ell)}}) \bb)^2$,
\begin{multline*}
\Var \left( \left( \ba^\top \left( \widetilde{\bQ}^{(\ell)} - \Esp{\widetilde{\bQ}^{(\ell)}} \right) \bb \right)^2 \right) \\
\leqslant \sum_{i, j} \Esp{\left( \ba^\top \deriv{}{N^{(\ell)}_{i, j}} \left[ \widetilde{\bQ}^{(\ell)} \bb \ba^\top \widetilde{\bQ}^{(\ell)} - \widetilde{\bQ}^{(\ell)} \bb \ba^\top \Esp{\widetilde{\bQ}^{(\ell)}} - \Esp{\widetilde{\bQ}^{(\ell)}} \bb \ba^\top \widetilde{\bQ}^{(\ell)} \right] \bb \right)^2}.
\end{multline*}
Similarly, injecting the expression of $\partial \widetilde{\bQ}^{(\ell)} / \partial N^{(\ell)}_{i, j}$, straightforward computations yield
\begin{align*}
&\Var \left( \left( \ba^\top \left( \widetilde{\bQ}^{(\ell)} - \Esp{\widetilde{\bQ}^{(\ell)}} \right) \bb \right)^2 \right) \\
&\leqslant \sum_{i, j} \Esp{\left( -\frac{2}{\sigma_N \sqrt{N}} \left( \ba^\top \left( \widetilde{\bQ}^{(\ell)} - \Esp{\widetilde{\bQ}^{(\ell)}} \right) \bb \right) \left[ \widetilde{\bQ}^{(\ell)} \left( \ba \bb^\top + \bb \ba^\top \right) \widetilde{\bQ}^{(\ell)} \bT^{(\ell)} \right]_{i, j} \right)^2} \\
&\leqslant \frac{8}{\sigma_N^2 N} \sum_{i, j} \Esp{\left( \ba^\top \left( \widetilde{\bQ}^{(\ell)} - \Esp{\widetilde{\bQ}^{(\ell)}} \right) \bb \right)^2 \left( [\widetilde{\bQ}^{(\ell)} \ba \bb^\top \widetilde{\bQ}^{(\ell)} \bT^{(\ell)}]_{i, j}^2 + [\widetilde{\bQ}^{(\ell)} \bb \ba^\top \widetilde{\bQ}^{(\ell)} \bT^{(\ell)}]_{i, j}^2 \right)} \\
& \leqslant \frac{16}{\sigma_N^2 N} \Norm{\ba}^2 \Norm{\bb}^2 \Esp{\left( \ba^\top \left( \widetilde{\bQ}^{(\ell)} - \Esp{\widetilde{\bQ}^{(\ell)}} \right) \bb \right)^2 \Norm{\widetilde{\bQ}^{(\ell)}}^4 \Norm{\bT^{(\ell)}}^2}.
\end{align*}
Since we have shown that $\esp{(\ba^\top (\widetilde{\bQ}^{(\ell)} - \esp{\widetilde{\bQ}^{(\ell)}}) \bb)^2} = \bigO[\tilde{z}](N^{-1})$, we have $\Var((\ba^\top (\widetilde{\bQ}^{(\ell)} - \esp{\widetilde{\bQ}^{(\ell)}}) \bb)^2) = \bigO[\tilde{z}](N^{-2})$; Q.E.D.

\subsection{Expansion of the Mean Empirical Stieltjes Transform}

Let $\widetilde{\bQ}^{(\ell)}_0$ denote the resolvent of $\frac{1}{\sigma_N} \left[ \frac{1}{N} \bN^{(\ell)} \bN^{(\ell) \top} - \mu^{(\ell)}_N \bI_{n_\ell} \right]$. Using the resolvent identity $\bA^{-1} - \bB^{-1} = \bA^{-1} (\bB - \bA) \bB^{-1}$, we can see that $\frac{1}{n_\ell} \Tr(\widetilde{\bQ}^{(\ell)}_0 - \widetilde{\bQ}^{(\ell)}) = \bigO[\tilde{z}](N^{-1})$ therefore the low-rank perturbation $\bP^{(\ell)}$ does not change the limiting spectral distribution and we can consider $\bP^{(\ell)} = \bzero_{n_\ell \times \prod_{\ell' \neq \ell} n_{\ell'}}$ from now on.

Applying $\frac{1}{n_\ell} \Tr$ to Equation \eqref{eq:EQ} and using the fact that $\frac{1}{n_\ell} \Tr \widetilde{\bQ}^{(\ell) 2} = \widetilde{m}_N^{(\ell) \prime}$ as well as $\esp{\widetilde{m}^{(\ell) 2}_N(\tilde{z})} = \esp{\widetilde{m}^{(\ell)}_N(\tilde{z})}^2 + \Var(\widetilde{m}^{(\ell)}_N(\tilde{z}))$ with $\Var(\widetilde{m}^{(\ell)}_N(\tilde{z})) = \bigO[\tilde{z}](N^{-2})$, we find
\begin{multline} \label{eq:esp_mN}
\Esp{\widetilde{m}^{(\ell)}_N(\tilde{z})}^2 + \tilde{z} \Esp{\widetilde{m}^{(\ell)}_N(\tilde{z})} + 1 \\
= -\frac{1}{\sigma_N} \frac{n_\ell}{N} \Esp{\tilde{z} \widetilde{m}^{(\ell) 2}_N(\tilde{z}) + \widetilde{m}^{(\ell)}_N(\tilde{z})} - \frac{1}{N} \frac{\mu^{(\ell)}_N}{\sigma_N^2} \Esp{\widetilde{m}^{(\ell) \prime}_N(\tilde{z})} + \bigO[\tilde{z}](N^{-\min(\frac{d}{2}, 2)})
\end{multline}
where the $\bigO[\tilde{z}](N^{-\min(\frac{d}{2}, 2)})$ stems from $\sigma_N = \bigTh(N^{\frac{d - 2}{2}})$. We know that the Stieltjes transform $m_\mathrm{SC}$ of the Wigner semicircle distribution on $[-2, 2]$ satisfies $m_\mathrm{SC}^2(\tilde{z}) + \tilde{z} m_\mathrm{SC}(\tilde{z}) + 1 = 0$ for all $\tilde{z} \in \bbC \setminus [-2, 2]$ (see, e.g., \citealp[Corollary 2.2.8]{pastur_eigenvalue_2011}). Let us subtract this relation to Equation \eqref{eq:esp_mN} and factorize by $\esp{\widetilde{m}^{(\ell)}_N(\tilde{z})} - m_\mathrm{SC}(\tilde{z})$ using the relation $a^2 - b^2 = (a - b) (a + b)$.
\begin{multline*}
\left( \Esp{\widetilde{m}^{(\ell)}_N(\tilde{z})} - m_\mathrm{SC}(\tilde{z}) \right) \left( \Esp{\widetilde{m}^{(\ell)}_N(\tilde{z})} + m_\mathrm{SC}(\tilde{z}) + \tilde{z} \right) \\
= -\frac{1}{\sigma_N} \frac{n_\ell}{N} \Esp{\tilde{z} \widetilde{m}^{(\ell) 2}_N(\tilde{z}) + \widetilde{m}^{(\ell)}_N(\tilde{z})} - \frac{1}{N} \frac{\mu^{(\ell)}_N}{\sigma_N^2} \Esp{\widetilde{m}^{(\ell) \prime}_N(\tilde{z})} + \bigO[\tilde{z}](N^{-\min(\frac{d}{2}, 2)}).
\end{multline*}
Let $g^{(\ell)}_N : \tilde{z} \mapsto \frac{-1}{\tilde{z} + m_\mathrm{SC}(\tilde{z}) + \esp{\widetilde{m}^{(\ell)}_N(\tilde{z})}}$. Since, $m_\mathrm{SC}$ and $\esp{\widetilde{m}^{(\ell)}_N}$ are Stieltjes transforms\footnote{$\widetilde{m}^{(\ell)}_N$ is obviously a Stieltjes transform and it can be shown that its expectation is also a Stieltjes transform with Herglotz theorem \citep[Theorem B3]{weidmann_linear_1980}.}, we have $\abs{g^{(\ell)}_N(\tilde{z})} \leqslant \abs{\Im[\tilde{z} + m_\mathrm{SC}(\tilde{z}) + \esp{\widetilde{m}^{(\ell)}_N(\tilde{z})}]}^{-1} \leqslant \abs{\Im \tilde{z}}^{-1}$ therefore
\begin{align}
&\Esp{\widetilde{m}^{(\ell)}_N(\tilde{z})} - m_\mathrm{SC}(\tilde{z}) \nonumber \\
&= g^{(\ell)}_N(\tilde{z}) \left[ \frac{1}{\sigma_N} \frac{n_\ell}{N} \Esp{\tilde{z} \widetilde{m}^{(\ell) 2}_N(\tilde{z}) + \widetilde{m}^{(\ell)}_N(\tilde{z})} + \frac{1}{N} \frac{\mu^{(\ell)}_N}{\sigma_N^2} \Esp{\widetilde{m}^{(\ell) \prime}_N(\tilde{z})} \right] + \bigO[\tilde{z}](N^{-\min(\frac{d}{2}, 2)}) \label{eq:expansion1} \\
&= g^{(\ell)}_N(\tilde{z}) \frac{1}{\sigma_N} \frac{n_\ell}{N} \Esp{\tilde{z} \widetilde{m}^{(\ell) 2}_N(\tilde{z}) + \widetilde{m}^{(\ell)}_N(\tilde{z})} + \bigO[\tilde{z}](N^{-1}) \label{eq:expansion2} \\
&= \bigO[\tilde{z}](N^{-\min(\frac{d - 2}{2}, 1)}). \label{eq:expansion3}
\end{align}
Notice that the dominant term in the difference $\esp{\widetilde{m}^{(\ell)}_N(\tilde{z})} - m_\mathrm{SC}(\tilde{z})$ differs depending on whether $d = 3$, $d = 4$ or $d \geqslant 5$ because $\sigma_N = \bigTh(N^{\frac{d - 2}{2}})$. The higher $d$, the faster the convergence of the empirical spectral distribution to the semicircle distribution. Indeed, at this point, Equation \eqref{eq:expansion3} already shows the pointwise convergence of $\esp{\widetilde{m}^{(\ell)}_N(\tilde{z})}$ to $m_\mathrm{SC}$ and therefore the weak convergence of the corresponding distributions \citep{geronimo_necessary_2003}. Yet, in order to show the confinement of the spectrum below, we need an explicit expansion of $\esp{\widetilde{m}^{(\ell)}_N(\tilde{z})}$ with all the terms not dominated by $N^{-1}$, which we state in the following lemma.

\begin{lemma} \label{lem:devEm}
$\esp{\widetilde{m}^{(\ell)}_N(\tilde{z})} = m_\mathrm{SC}(\tilde{z}) + h^{(\ell)}_N(\tilde{z}) + \bigO[\tilde{z}](N^{-\min(\frac{d}{2}, 2)})$ with
\begin{multline*}
h^{(\ell)}_N(\tilde{z}) = \tau^{(\ell)}_N(\tilde{z}) \left( \frac{1}{\sigma_N} \frac{n_\ell}{N} \left( \tilde{z} t^{(\ell) 2}_N(\tilde{z}) + t^{(\ell)}_N(\tilde{z}) \right) - \frac{1}{N} \frac{\mu^{(\ell)}_N}{\sigma_N^2} \frac{m_\mathrm{SC}(\tilde{z})}{\tilde{z} + 2 m_\mathrm{SC}(\tilde{z})} \right) \\
\begin{aligned} \text{where} \quad \tau^{(\ell)}_N(\tilde{z}) &= \frac{-1}{\tilde{z} + 2 m_\mathrm{SC}(\tilde{z})} \left( 1 + \frac{1}{\sigma_N} \frac{n_\ell}{N} \frac{\tilde{z} m_\mathrm{SC}^2(\tilde{z}) + m_\mathrm{SC}(\tilde{z})}{\left( \tilde{z} + 2 m_\mathrm{SC}(\tilde{z}) \right)^2} \right) \\
\text{and} \quad t^{(\ell)}_N(\tilde{z}) &= m_\mathrm{SC}(\tilde{z}) - \frac{1}{\sigma_N} \frac{n_\ell}{N} \frac{\tilde{z} m_\mathrm{SC}^2(\tilde{z}) + m_\mathrm{SC}(\tilde{z})}{\tilde{z} + 2 m_\mathrm{SC}(\tilde{z})}. \end{aligned}
\end{multline*}
\end{lemma}
\begin{proof}
Firstly, with the Poincaré-Nash inequality (Lemma \ref{lem:poincare-nash}), we can show that
\[
\Var(\widetilde{m}^{(\ell) \prime}_N(\tilde{z})) \leqslant \sum_{i = 1}^{n_\ell} \sum_{j = 1}^{\prod_{\ell' \neq \ell} n_{\ell'}} \Esp{\Abs{\frac{1}{n_\ell} \deriv{\Tr \widetilde{\bQ}^{(\ell) 2}}{N^{(\ell)}_{i, j}}}^2} \leqslant \frac{16 \Esp{\Norm{\widetilde{\bQ}^{(\ell)}}^6 \Norm{\bN^{(\ell)}}^2}}{n_\ell \sigma_N^2 N^2} = \bigO[\tilde{z}](N^{-2}).
\]
Taking $\frac{1}{n_\ell} \Tr$ of Equation \eqref{eq:EQ} with $\bP^{(\ell)} = \bzero_{n_\ell \times \prod_{\ell' \neq \ell} n_{\ell'}}$ and differentiating with respect to the complex variable $\tilde{z}$ under $\bbE$ (which is possible because the integrand can be upper bounded on every compact subset of $\bbC \setminus \bbR$), we find
\begin{multline} \label{eq:diffEm}
\Esp{\left( 2 \widetilde{m}^{(\ell)}_N(\tilde{z}) + \tilde{z} \right) \widetilde{m}^{(\ell) \prime}_N(\tilde{z}) + \widetilde{m}^{(\ell)}_N(\tilde{z})} \\
= \mathbb{E} \left[ -\frac{1}{\sigma_N} \frac{n_\ell}{N} \left( \widetilde{m}^{(\ell) 2}_N(\tilde{z}) + 2 \tilde{z} \widetilde{m}^{(\ell)}_N (\tilde{z}) \widetilde{m}^{(\ell) \prime}_N(\tilde{z}) + \widetilde{m}^{(\ell) \prime}_N(\tilde{z}) \right) \right. \\
\left. - \frac{1}{N} \left( \frac{\mu^{(\ell)}_N}{\sigma_N^2} + \frac{\tilde{z}}{\sigma_N} \right) \widetilde{m}^{(\ell) \prime \prime}_N(\tilde{z}) - \frac{2 \widetilde{m}^{(\ell) \prime}(\tilde{z})}{\sigma_N N} \right].
\end{multline}
Then, since $\Var \widetilde{m}^{(\ell)}_N(\tilde{z}) = \bigO[\tilde{z}](N^{-2})$ and $\Var \widetilde{m}^{(\ell) \prime}_N(\tilde{z}) = \bigO[\tilde{z}](N^{-2})$, we have $\esp{\widetilde{m}^{(\ell) 2}_N(\tilde{z})} = \esp{\widetilde{m}^{(\ell)}_N(\tilde{z})}^2 + \bigO[\tilde{z}](N^{-2})$ and $\esp{\widetilde{m}^{(\ell)}_N(\tilde{z}) \widetilde{m}^{(\ell) \prime}_N(\tilde{z})} = \esp{\widetilde{m}^{(\ell)}_N(\tilde{z})} \esp{\widetilde{m}^{(\ell) \prime}_N(\tilde{z})} + \bigO[\tilde{z}](N^{-2})$ (using the Cauchy-Schwarz inequality to upper bound the covariance). And, with the fact that $\esp{\widetilde{m}^{(\ell)}_N(\tilde{z})} = m_\mathrm{SC}(\tilde{z}) + \bigO[\tilde{z}](N^{-\min(\frac{d - 2}{2}, 1)})$ (Equation \eqref{eq:expansion3}), we obtain from Equation \eqref{eq:diffEm} that
\[
\Esp{\widetilde{m}^{(\ell) \prime}_N(\tilde{z})} = \frac{-m_\mathrm{SC}(\tilde{z})}{\tilde{z} + 2 m_\mathrm{SC}(\tilde{z})} + \bigO[\tilde{z}](N^{-\min(\frac{d - 2}{2}, 1)})
\]
since $\abs{\tilde{z} + 2 m_\mathrm{SC}(\tilde{z})}^{-1} \leqslant \abs{\tilde{z}}^{-1} \leqslant \abs{\Im \tilde{z}}^{-1}$. Moreover, using Equation \eqref{eq:expansion3} in Equation \eqref{eq:expansion2}, we find
\begin{equation} \label{eq:expansion2p}
\Esp{\widetilde{m}^{(\ell)}_N(\tilde{z})} = m_\mathrm{SC}(\tilde{z}) + g^{(\ell)}_N(\tilde{z}) \frac{1}{\sigma_N} \frac{n_\ell}{N} \left( \tilde{z} m_\mathrm{SC}^2(\tilde{z}) + m_\mathrm{SC}(\tilde{z}) \right) + \bigO[\tilde{z}](N^{-1}).
\end{equation}
We can now inject the last two relations into Equation \eqref{eq:expansion1}:
\begin{multline} \label{eq:expansion1p}
\Esp{\widetilde{m}^{(\ell)}_N(\tilde{z})} = m_\mathrm{SC}(\tilde{z}) + g^{(\ell)}_N(\tilde{z}) \frac{\tilde{z}}{\sigma_N} \frac{n_\ell}{N} \left( m_\mathrm{SC}(\tilde{z}) + g^{(\ell)}_N(\tilde{z}) \frac{1}{\sigma_N} \frac{n_\ell}{N} \left( \tilde{z} m_\mathrm{SC}^2(\tilde{z}) + m_\mathrm{SC}(\tilde{z}) \right) \right)^2 \\
+ g^{(\ell)}_N(\tilde{z}) \frac{1}{\sigma_N} \frac{n_\ell}{N} \left( m_\mathrm{SC}(\tilde{z}) + g^{(\ell)}_N(\tilde{z}) \frac{1}{\sigma_N} \frac{n_\ell}{N} \left( \tilde{z} m_\mathrm{SC}^2(\tilde{z}) + m_\mathrm{SC}(\tilde{z}) \right) \right) \\
- g^{(\ell)}_N(\tilde{z}) \frac{1}{N} \frac{\mu^{(\ell)}_N}{\sigma_N^2} \frac{m_\mathrm{SC}(\tilde{z})}{\tilde{z} + 2 m_\mathrm{SC}(\tilde{z})} + \bigO[\tilde{z}](N^{-\min(\frac{d}{2}, 2)}).
\end{multline}
We only need to handle the asymptotic behavior of $g^{(\ell)}_N(\tilde{z})$ to conclude the proof. With Equation \eqref{eq:expansion3} and the fact that $\abs{\tilde{z} + 2 m_\mathrm{SC}(\tilde{z})}^{-1} \leqslant \abs{\Im \tilde{z}}^{-1}$, we have $g^{(\ell)}_N(\tilde{z}) = \frac{-1}{\tilde{z} + 2 m_\mathrm{SC}(\tilde{z})} + \bigO[\tilde{z}](N^{-\min(\frac{d - 2}{2}, 1)})$. We can then use this relation in Equation \eqref{eq:expansion2p}:
\[
\Esp{\widetilde{m}^{(\ell)}_N(\tilde{z})} - m_\mathrm{SC}(\tilde{z}) = -\frac{1}{\sigma_N} \frac{n_\ell}{N} \frac{\tilde{z} m_\mathrm{SC}^2(\tilde{z}) + m_\mathrm{SC}(\tilde{z})}{\tilde{z} + 2 m_\mathrm{SC}(\tilde{z})} + \bigO[\tilde{z}](N^{-1}).
\]
Therefore, we have
\begin{align*}
g^{(\ell)}_N(\tilde{z}) &= \frac{-1}{\tilde{z} + m_\mathrm{SC}(\tilde{z}) + \Esp{\widetilde{m}^{(\ell)}_N(\tilde{z})}} \\
&= \frac{-1}{\tilde{z} + 2 m_\mathrm{SC}(\tilde{z})} \left( 1 + \frac{\Esp{\widetilde{m}^{(\ell)}_N(\tilde{z})} - m_\mathrm{SC}(\tilde{z})}{\tilde{z} + 2 m_\mathrm{SC}(\tilde{z})} \right)^{-1} \\
&= \frac{-1}{\tilde{z} + 2 m_\mathrm{SC}(\tilde{z})} \left( 1 - \frac{\Esp{\widetilde{m}^{(\ell)}_N(\tilde{z})} - m_\mathrm{SC}(\tilde{z})}{\tilde{z} + 2 m_\mathrm{SC}(\tilde{z})} + \bigO[\tilde{z}](N^{-\min(d - 2, 2)}) \right) \\
&= \frac{-1}{\tilde{z} + 2 m_\mathrm{SC}(\tilde{z})} \left( 1 + \frac{1}{\sigma_N} \frac{n_\ell}{N} \frac{\tilde{z} m_\mathrm{SC}^2(\tilde{z}) + m_\mathrm{SC}(\tilde{z})}{\left( \tilde{z} + 2 m_\mathrm{SC}(\tilde{z}) \right)^2} \right) + \bigO[\tilde{z}](N^{-1}).
\end{align*}
Eventually, Equation \eqref{eq:expansion1p} becomes
\begin{multline*}
\Esp{\widetilde{m}^{(\ell)}_N(\tilde{z})} = m_\mathrm{SC}(\tilde{z}) + \tau^{(\ell)}_N(\tilde{z}) \frac{\tilde{z}}{\sigma_N} \frac{n_\ell}{N} \left( m_\mathrm{SC}(\tilde{z}) - \frac{1}{\sigma_N} \frac{n_\ell}{N} \frac{\tilde{z} m_\mathrm{SC}^2(\tilde{z}) + m_\mathrm{SC}(\tilde{z})}{\tilde{z} + 2 m_\mathrm{SC}(\tilde{z})} \right)^2 \\
+ \tau^{(\ell)}_N(\tilde{z}) \frac{1}{\sigma_N} \frac{n_\ell}{N} \left( m_\mathrm{SC}(\tilde{z}) - \frac{1}{\sigma_N} \frac{n_\ell}{N} \frac{\tilde{z} m_\mathrm{SC}^2(\tilde{z}) + m_\mathrm{SC}(\tilde{z})}{\tilde{z} + 2 m_\mathrm{SC}(\tilde{z})} \right) \\
- \tau^{(\ell)}_N(\tilde{z}) \frac{1}{N} \frac{\mu^{(\ell)}_N}{\sigma_N^2} \frac{m_\mathrm{SC}(\tilde{z})}{\tilde{z} + 2 m_\mathrm{SC}(\tilde{z})} + \bigO[\tilde{z}](N^{-\min(\frac{d}{2}, 2)}).
\end{multline*}
with $\tau^{(\ell)}_N : \tilde{z} \mapsto \frac{-1}{\tilde{z} + 2 m_\mathrm{SC}(\tilde{z})} \left( 1 + \frac{1}{\sigma_N} \frac{n_\ell}{N} \frac{\tilde{z} m_\mathrm{SC}^2(\tilde{z}) + m_\mathrm{SC}(\tilde{z})}{\left( \tilde{z} + 2 m_\mathrm{SC}(\tilde{z}) \right)^2} \right)$.
\end{proof}

\subsection{Confinement of the Spectrum}
\label{proof:thm:ed:confinement}

We have just found, in Lemma \ref{lem:devEm}, the following expansion: $\esp{\widetilde{m}^{(\ell)}_N(\tilde{z})} = m_\mathrm{SC}(\tilde{z}) + h^{(\ell)}_N(\tilde{z}) + \bigO[\tilde{z}](N^{-\min(\frac{d}{2}, 2)})$. Note that this shows that the limiting spectral distribution of \linebreak $\frac{1}{\sigma_N} \left[ \frac{1}{N} \bN^{(\ell)} \bN^{(\ell) \top} - \mu^{(\ell)}_N \bI_{n_\ell} \right]$ (and therefore that of $\frac{1}{\sigma_N} \left[ \bT^{(\ell)} \bT^{(\ell) \top} - \mu^{(\ell)}_N \bI_{n_\ell} \right]$) is the semicircle distribution on $[-2, 2]$. We will now use this expansion to prove that, almost surely, no eigenvalue of $\frac{1}{\sigma_N} \left[ \frac{1}{N} \bN^{(\ell)} \bN^{(\ell) \top} - \mu^{(\ell)}_N \bI_{n_\ell} \right]$ stays outside $[-2, 2]$ as $N \to +\infty$. That is, for all $\varepsilon > 0$, there exists an integer $N_0$ such that, for all $\lambda \in \Sp \left( \frac{1}{\sigma_N} \left[ \frac{1}{N} \bN^{(\ell)} \bN^{(\ell) \top} - \mu^{(\ell)}_N \bI_{n_\ell} \right] \right)$, $\Dist(\lambda, \Supp \mu_\mathrm{MP}) \leqslant \varepsilon$ as soon as $N \geqslant N_0$.

Let $\varepsilon > 0$, $\varphi : \bbR \mapsto [0, 1]$ be an infinitely differentiable function which equals $1$ on $[-2, 2]$ and $0$ on $\bbR \setminus [-2 - \varepsilon, 2 + \varepsilon]$ and $\psi = 1 - \varphi$. We want to show that
\[
\Tr \left( \psi \left( \frac{1}{\sigma_N} \left[ \frac{1}{N} \bN^{(\ell)} \bN^{(\ell) \top} - \mu^{(\ell)}_N \bI_{n_\ell} \right] \right) \right) \to 0 \quad \text{almost surely as}~ N \to +\infty.
\]

First of all, we show the convergence in mean with the Helffer-Sjöstrand formula (Proposition \ref{prop:helffer-sjostrand}).
\begin{multline*}
\Esp{\frac{1}{n_\ell} \Tr \left( \varphi \left( \frac{1}{\sigma_N} \left[ \frac{1}{N} \bN^{(\ell)} \bN^{(\ell) \top} - \mu^{(\ell)}_N \bI_{n_\ell} \right] \right) \right)} \\
\begin{aligned}
&= \frac{2}{\pi} \Re \int_{\bbC^+} \deriv{\Phi_q[\varphi]}{\bar{\tilde{z}}}(\tilde{z}) \Esp{\widetilde{m}^{(\ell)}_N(\tilde{z})} ~\rmd \tilde{z} \\
&= \begin{multlined}[t] \frac{2}{\pi} \Re \int_{\bbC^+} \deriv{\Phi_q[\varphi]}{\bar{\tilde{z}}}(\tilde{z}) m_\mathrm{SC}(\tilde{z}) ~\rmd \tilde{z} + \frac{2}{\pi} \Re \int_{\bbC^+} \deriv{\Phi_q[\varphi]}{\bar{\tilde{z}}}(\tilde{z}) h^{(\ell)}_N(\tilde{z}) ~\rmd \tilde{z} \\
+ \frac{2}{\pi} \Re \int_{\bbC^+} \deriv{\Phi_q[\varphi]}{\bar{\tilde{z}}}(\tilde{z}) \times \bigO[\tilde{z}](N^{-\min(\frac{d}{2}, 2)}) ~\rmd \tilde{z}. \end{multlined}
\end{aligned}
\end{multline*}
where we have used the expression of $\esp{\widetilde{m}^{(\ell)}_N(\tilde{z})}$ given by Lemma \ref{lem:devEm}. The first integral is $\int_\bbR \varphi ~\rmd \mu_\mathrm{SC} = 1$ while the last one is $\bigO(N^{-\min(\frac{d}{2}, 2)})$ with $q$ chosen sufficiently large so that $\deriv{\Phi_q[\varphi]}{\bar{\tilde{z}}}(\tilde{z})$ cancels the divergence of $\bigO[\tilde{z}](N^{-\min(\frac{d}{2}, 2)})$ near the real axis. In order to evaluate the second integral, we perform an integration by parts:
\begin{multline*}
\frac{2}{\pi} \Re \int_{\bbC^+} \deriv{\Phi_q[\varphi]}{\bar{\tilde{z}}}(\tilde{z}) h^{(\ell)}_N(\tilde{z}) ~\rmd \tilde{z} \\
\begin{aligned}
&= \begin{multlined}[t] \frac{2}{\pi} \Re \left[ \frac{1}{2} \int_0^{+\infty} \left( \int_{-\infty}^{+\infty} \deriv{\Phi_q[\varphi]}{x}(x + \rmi y) h^{(\ell)}_N(x + \rmi y) ~\rmd x \right) \rmd y \right. \\ \left. + \frac{\rmi}{2} \int_{-\infty}^{+\infty} \left( \int_0^{+\infty} \deriv{\Phi_q[\varphi]}{y}(x + \rmi y) h^{(\ell)}_N(x + \rmi y) ~\rmd y \right) \rmd x \right] \end{multlined} \\
&= \frac{2}{\pi} \Re \left[ \frac{-\rmi}{2} \int_\bbR \lim_{y \downarrow 0} \left\{ \Phi_q[\varphi](x + \rmi y) h^{(\ell)}_N(x + \rmi y) \right\} ~\rmd x - \int_{\bbC^+} \Phi_q[\varphi](\tilde{z}) \deriv{h}{\bar{\tilde{z}}}(\tilde{z}) ~\rmd \tilde{z} \right].
\end{aligned}
\end{multline*}
Since $h^{(\ell)}_N$ is an analytic function (as sums and products of Stieltjes transforms), we have $\partial h / \partial \bar{\tilde{z}} = 0$ by the Cauchy-Riemann equations. Moreover, we have $\lim_{y \downarrow 0} \Phi_q[\varphi](x + \rmi y) = \varphi(x)$ and, from the definition of $h^{(\ell)}_N$, $\lim_{y \downarrow 0} \Im[h^{(\ell)}_N(x + \rmi y)] = 0$ for all $x \in \bbR \setminus [-2, 2]$. Therefore,
\[
\frac{2}{\pi} \Re \int_{\bbC^+} \deriv{\Phi_q[\varphi]}{\bar{\tilde{z}}}(\tilde{z}) h^{(\ell)}_N(\tilde{z}) ~\rmd \tilde{z} = \lim_{y \downarrow 0} \frac{1}{\pi} \int_{-2}^2 \Im[h^{(\ell)}_N(x + \rmi y)] ~\rmd x.
\]
We use Lemma \ref{lem:stieltjes_distrib} to show that this integral equals $\lim_{y \to +\infty} -\rmi y h^{(\ell)}_N(\rmi y) = 0$. The function $h^{(\ell)}_N$ is analytic on $\bbC \setminus [-2, 2]$, $\lim_{\abs{\tilde{z}} \to +\infty} h(\tilde{z}) = 0$ and $h^{(\ell)}_N(\bar{\tilde{z}}) = \overline{h(\tilde{z})}$ for all $\tilde{z} \in \bbC \setminus [-2, 2]$. Thus, we just need to show that there exist an integer $n_0$ and a constant $C > 0$ such that $\abs{h^{(\ell)}_N(\tilde{z})} \leqslant C \max(\Dist(\tilde{z}, [-2, 2])^{-n_0}, 1)$ for all $\tilde{z} \in \bbC \setminus [-2, 2]$. Firstly, we find an upper bound for $t^{(\ell)}_N(\tilde{z})$. Since $\tilde{z} \mapsto \frac{-1}{\tilde{z} + 2 m_\mathrm{SC}(\tilde{z})}$ is the Stieltjes transform of a probability measure on $[-2, 2]$ and $-(\tilde{z} m_\mathrm{SC}(\tilde{z}) + 1) = m_\mathrm{SC}^2(\tilde{z})$, we have
\begin{align*}
\Abs{t^{(\ell)}_N(\tilde{z})} &= \Abs{m_\mathrm{SC}(\tilde{z})} \Abs{1 - \frac{1}{\sigma_N} \frac{n_\ell}{N} \frac{\tilde{z} m_\mathrm{SC}(\tilde{z}) + 1}{\tilde{z} + 2 m_\mathrm{SC}(\tilde{z})}} \\
&\leqslant \frac{1}{\Dist(\tilde{z}, [-2, 2])} \left( 1 + \frac{1}{\sigma_N} \frac{n_\ell}{N} \frac{1}{\Dist(\tilde{z}, [-2, 2])^3} \right) \\
&\leqslant \left( 1 + \frac{1}{\sigma_N} \frac{n_\ell}{N} \right) \max(\Dist(\tilde{z}, [-2, 2])^{-4}, 1).
\end{align*}
Moreover, since $\abs{\tilde{z}} \leqslant \Dist(\tilde{z}, [-2, 2]) + 2$, we also have
\begin{align*}
\Abs{\tilde{z} t^{(\ell)}_N(\tilde{z})} &\leqslant \left( 1 + \frac{2}{\Dist(\tilde{z}, [-2, 2])} \right) \left( 1 + \frac{1}{\sigma_N} \frac{n_\ell}{N} \frac{1}{\Dist(\tilde{z}, [-2, 2])^3} \right) \\
&\leqslant 3 \left( 1 + \frac{1}{\sigma_N} \frac{n_\ell}{N} \right) \max(\Dist(\tilde{z}, [-2, 2])^{-4}, 1).
\end{align*}
Similarly,
\begin{align*}
\Abs{\tau^{(\ell)}_N(\tilde{z})} &= \Abs{\frac{-1}{\tilde{z} + 2 m_\mathrm{SC}(\tilde{z})}} \Abs{1 + \frac{1}{\sigma_N} \frac{n_\ell}{N} \frac{m_\mathrm{SC}(\tilde{z}) \left( \tilde{z} m_\mathrm{SC}(\tilde{z}) + 1 \right)}{\left( \tilde{z} + 2 m_\mathrm{SC}(\tilde{z}) \right)^2}} \\
&\leqslant \left( 1 + \frac{1}{\sigma_N} \frac{n_\ell}{N} \right) \max(\Dist(\tilde{z}, [-2, 2])^{-6}, 1).
\end{align*}
Hence, we can upper bound $\abs{h^{(\ell)}_N(\tilde{z})}$:
\begin{align*}
\Abs{h^{(\ell)}_N(\tilde{z})} &\leqslant \Abs{\tau^{(\ell)}_N(\tilde{z})} \left( \frac{1}{\sigma_N} \frac{n_\ell}{N} \Abs{t^{(\ell)}_N(\tilde{z})} \left( \Abs{\tilde{z} t^{(\ell)}_N(\tilde{z})} + 1 \right) + \frac{1}{N} \frac{\mu^{(\ell)}_N}{\sigma_N^2} \Abs{\frac{m_\mathrm{SC}(\tilde{z})}{\tilde{z} + 2 m_\mathrm{SC}(\tilde{z})}} \right) \\
&\leqslant \begin{multlined}[t] \left( 1 + \frac{1}{\sigma_N} \frac{n_\ell}{N} \right) \left[  \frac{1}{\sigma_N} \frac{n_\ell}{N} \left( 1 + \frac{1}{\sigma_N} \frac{n_\ell}{N} \right) \left[ 3 \left( 1 + \frac{1}{\sigma_N} \frac{n_\ell}{N} \right) + 1 \right] + \frac{1}{N} \frac{\mu^{(\ell)}_N}{\sigma_N^2} \right] \\ \times \max(\Dist(\tilde{z}, [-2, 2])^{-14}, 1). \end{multlined}
\end{align*}
Therefore we can conclude that $\frac{2}{\pi} \Re \int_{\bbC^+} \deriv{\Phi_q[\varphi]}{\bar{\tilde{z}}}(\tilde{z}) h^{(\ell)}_N(\tilde{z}) ~\rmd \tilde{z} = 0$ (Lemma \ref{lem:stieltjes_distrib}) and
\begin{multline*}
\Esp{\frac{1}{n_\ell} \Tr \left( \varphi \left( \frac{1}{\sigma_N} \left[ \frac{1}{N} \bN^{(\ell)} \bN^{(\ell) \top} - \mu^{(\ell)}_N \bI_{n_\ell} \right] \right) \right)} = 1 + \bigO(N^{-\min(\frac{d}{2}, 2)}), \\
\text{i.e.,} \quad \Esp{\Tr \left( \psi \left( \frac{1}{\sigma_N} \left[ \frac{1}{N} \bN^{(\ell)} \bN^{(\ell) \top} - \mu^{(\ell)}_N \bI_{n_\ell} \right] \right) \right)} = \bigO(N^{-\min(\frac{d - 2}{2}, 1)}).
\end{multline*}

Secondly, we prove the almost sure convergence of $\Tr \left( \psi \left( \frac{1}{\sigma_N} \left[ \frac{1}{N} \bN^{(\ell)} \bN^{(\ell) \top} - \mu^{(\ell)}_N \bI_{n_\ell} \right] \right) \right)$ to $0$ by showing that its variance is $\bigO(N^{-\min(\frac{d}{2}, 2)})$ (and Lemma \ref{lem:as_convergence} implies the result). With the Poincaré-Nash inequality (Lemma \ref{lem:poincare-nash}), we have
\begin{multline*}
\Var \left( \Tr \left( \psi \left( \frac{1}{\sigma_N} \left[ \frac{1}{N} \bN^{(\ell)} \bN^{(\ell) \top} - \mu^{(\ell)}_N \bI_{n_\ell} \right] \right) \right) \right) \\
\begin{aligned}
&= \Var \left( \Tr \left( \varphi \left( \frac{1}{\sigma_N} \left[ \frac{1}{N} \bN^{(\ell)} \bN^{(\ell) \top} - \mu^{(\ell)}_N \bI_{n_\ell} \right] \right) \right) \right) \\
&\leqslant \sum_{i = 1}^{n_\ell} \sum_{j = 1}^{\prod_{\ell' \neq \ell} n_{\ell'}} \Esp{\Abs{\deriv{\Tr \left( \varphi \left( \frac{1}{\sigma_N} \left[ \frac{1}{N} \bN^{(\ell)} \bN^{(\ell) \top} - \mu^{(\ell)}_N \bI_{n_\ell} \right] \right) \right)}{N^{(\ell)}_{i, j}}}^2} \\
&= \sum_{i = 1}^{n_\ell} \sum_{j = 1}^{\prod_{\ell' \neq \ell} n_{\ell'}} \Esp{\Abs{\Tr \left( \varphi' \left( \frac{1}{\sigma_N} \left[ \frac{1}{N} \bN^{(\ell)} \bN^{(\ell) \top} - \mu^{(\ell)}_N \bI_{n_\ell} \right] \right) \deriv{}{N^{(\ell)}_{i, j}} \left[ \frac{1}{\sigma_N} \frac{1}{N} \bN^{(\ell)} \bN^{(\ell) \top} \right] \right)}^2} \\
&= \frac{1}{\sigma_N^2} \frac{1}{N^2} \sum_{i = 1}^{n_\ell} \sum_{j = 1}^{\prod_{\ell' \neq \ell} n_{\ell'}} \Esp{\Abs{\left[ 2 \varphi' \left( \frac{1}{\sigma_N} \left[ \frac{1}{N} \bN^{(\ell)} \bN^{(\ell) \top} - \mu^{(\ell)}_N \bI_{n_\ell} \right] \right) \bN^{(\ell)} \right]_{i, j}}^2} \\
&= \begin{multlined}[t] \frac{1}{\sigma_N} \frac{4}{N} \Esp{\Tr u \left( \frac{1}{\sigma_N} \left[ \frac{1}{N} \bN^{(\ell)} \bN^{(\ell) \top} - \mu^{(\ell)}_N \bI_{n_\ell} \right] \right)} \\ + \frac{4}{N} \frac{\mu^{(\ell)}_N}{\sigma_N^2} \Esp{\Tr \varphi^{\prime 2} \left( \frac{1}{\sigma_N} \left[ \frac{1}{N} \bN^{(\ell)} \bN^{(\ell) \top} - \mu^{(\ell)}_N \bI_{n_\ell} \right] \right)} \end{multlined}
\end{aligned}
\end{multline*}
where $u : x \mapsto x \varphi^{\prime 2}(x)$ and $\varphi^{\prime 2}$ are infinitely differentiable functions with compact support which equal $0$ on $[-2, 2]$. Hence, applying similarly the Helffer-Sjöstrand formula (Proposition \ref{prop:helffer-sjostrand}), we find
\begin{align*}
\frac{1}{\sigma_N} \frac{4}{N} \Esp{\Tr u \left( \frac{1}{\sigma_N} \left[ \frac{1}{N} \bN^{(\ell)} \bN^{(\ell) \top} - \mu^{(\ell)}_N \bI_{n_\ell} \right] \right)} &= \frac{4}{\sigma_N} \frac{n_\ell}{N} \frac{2}{\pi} \Re \int_{\bbC^+} \deriv{\Phi_q[u]}{\bar{\tilde{z}}}(\tilde{z}) \Esp{\widetilde{m}^{(\ell)}_N(\tilde{z})} ~\rmd \tilde{z} \\
&= \bigO(N^{-\min(d - 1, \frac{d + 2}{2})})
\end{align*}
and
\begin{multline*}
\frac{4}{N} \frac{\mu^{(\ell)}_N}{\sigma_N^2} \Esp{\Tr \varphi^{\prime 2} \left( \frac{1}{\sigma_N} \left[ \frac{1}{N} \bN^{(\ell)} \bN^{(\ell) \top} - \mu^{(\ell)}_N \bI_{n_\ell} \right] \right)} \\
= 4 \frac{n_\ell}{N} \frac{\mu^{(\ell)}_N}{\sigma_N^2} \frac{2}{\pi} \Re \int_{\bbC^+} \deriv{\Phi_q[\varphi']}{\bar{\tilde{z}}}(\tilde{z}) \Esp{\widetilde{m}^{(\ell)}_N(\tilde{z})} ~\rmd \tilde{z} = \bigO(N^{-\min(\frac{d}{2}, 2)})
\end{multline*}
for $q$ chosen sufficiently large. Thus,
\[
\Var \left( \Tr \left( \psi \left( \frac{1}{\sigma_N} \left[ \frac{1}{N} \bN^{(\ell)} \bN^{(\ell) \top} - \mu^{(\ell)}_N \bI_{n_\ell} \right] \right) \right) \right) = \bigO(N^{-\min(\frac{d}{2}, 2)})
\]
and we can conclude on the almost sure convergence with Lemma \ref{lem:as_convergence}.

\subsection{Deterministic Equivalent}

With the rescaling $(z, \bQ^{(\ell)}(z)) \curvearrowright (\tilde{z}, \widetilde{\bQ}^{(\ell)}(\tilde{z}))$, Equation \eqref{eq:PTQ} becomes
\begin{multline*}
\frac{1}{\sigma_N} \Esp{\bP^{(\ell)} \bT^{(\ell) \top} \widetilde{\bQ}^{(\ell)}} = \frac{1}{\sigma_N} \Esp{\bP^{(\ell)} \bP^{(\ell) \top} \widetilde{\bQ}^{(\ell)}} \\
- \frac{1}{\sigma_N^2} \Esp{\frac{n_\ell}{N} \widetilde{m}^{(\ell)}_N(\tilde{z}) \bP^{(\ell)} \bT^{(\ell) \top} \widetilde{\bQ}^{(\ell)} + \frac{1}{N} \bP^{(\ell)} \bT^{(\ell) \top} \widetilde{\bQ}^{(\ell) 2}}
\end{multline*}
where $\Norm{\frac{1}{\sigma_N^2} \Esp{\frac{n_\ell}{N} \widetilde{m}^{(\ell)}_N(\tilde{z}) \bP^{(\ell)} \bT^{(\ell) \top} \widetilde{\bQ}^{(\ell)} + \frac{1}{N} \bP^{(\ell)} \bT^{(\ell) \top} \widetilde{\bQ}^{(\ell) 2}}} \to 0$ as $N \to +\infty$ since $\norm{\bP^{(\ell)}} = \bigO(N^{\frac{d - 2}{4}})$ and $\norm{\bT^{(\ell)}} = \bigO(N^{\frac{d - 2}{2}})$. Hence, with Equation \eqref{eq:EQ} and Equation \eqref{eq:expansion3}, we have
\[
\Norm{m_\mathrm{SC}(\tilde{z}) \Esp{\widetilde{\bQ}^{(\ell)}} + \tilde{z} \Esp{\widetilde{\bQ}^{(\ell)}} + \bI_{n_\ell} - \frac{1}{\sigma_N} \bP^{(\ell)} \bP^{(\ell) \top} \Esp{\widetilde{\bQ}^{(\ell)}}} \xrightarrow[N \to +\infty]{} 0
\]
and, since $m_\mathrm{SC}(\tilde{z}) + \tilde{z} = \frac{-1}{m_\mathrm{SC}(\tilde{z})}$, we can define the following deterministic equivalent (Definition \ref{def:matrix_equivalent}):
\[
\widetilde{\bQ}^{(\ell)}(\tilde{z}) \longleftrightarrow \bar{\bQ}^{(\ell)}(\tilde{z}) \eqdef \left( \frac{1}{\sigma_N} \bP^{(\ell)} \bP^{(\ell) \top} + \frac{1}{m_\mathrm{SC}(\tilde{z})} \bI_{n_\ell} \right)^{-1}.
\]

\section{Proof of Theorem \ref{thm:spike}}
\label{proof:thm:spike}

Recall that $\widetilde{\bQ}^{(\ell)}$ is the resolvent of $\frac{1}{\sigma_N} \left[ \bT^{(\ell)} \bT^{(\ell) \top} - \mu^{(\ell)}_N \bI_{n_\ell} \right]$ while $\widetilde{\bQ}^{(\ell)}_0$ denotes the resolvent of the same model without signal, $\frac{1}{\sigma_N} \left[ \frac{1}{N} \bN^{(\ell)} \bN^{(\ell) \top} - \mu^{(\ell)}_N \bI_{n_\ell} \right]$.

\subsection{Convergence of Bilinear Forms}
\label{app:proof_spike_low_mlrank:bilinear_forms}

First of all, we must show the \emph{almost sure} convergence $\ba^\top \widetilde{\bQ}^{(\ell)}_0 \bb - m_\mathrm{SC}(\tilde{z}) \scal{\ba}{\bb} \to 0$ for all bounded (sequences of) vectors $\ba, \bb \in \bbR^{n_\ell}$. Given the concentration result proven in Section \ref{proof:thm:ed:concentration}, we just need to show that $\ba^\top \esp{\widetilde{\bQ}^{(\ell)}_0} \bb - m_\mathrm{SC}(\tilde{z}) \scal{\ba}{\bb} \to 0$ as $N \to +\infty$.

Let us multiply Equation \eqref{eq:EQ} when $\bP^{(\ell)} = \bzero_{n_\ell \times \prod_{\ell' \neq \ell} n_{\ell'}}$ by $\ba^\top$ on the left and $\bb$ on the right.
\[
\Esp{\widetilde{m}^{(\ell)}_N(\tilde{z}) \ba^\top \widetilde{\bQ}^{(\ell)}_0 \bb} + \tilde{z} \Esp{\ba^\top \widetilde{\bQ}^{(\ell)}_0 \bb} + \scal{\ba}{\bb} = \bigO[z](N^{-\min(\frac{d - 2}{2}, 1)}).
\]
Then, using the fact that $\esp{\widetilde{m}^{(\ell)}_N(\tilde{z})} = m_\mathrm{SC}(\tilde{z}) + \bigO[\tilde{z}](N^{-\min(\frac{d - 2}{2}, 1)})$ and $m_\mathrm{SC}(\tilde{z}) + \tilde{z} = \frac{-1}{m_\mathrm{SC}(\tilde{z})}$, we obtain the desired result: $\esp{\ba^\top \widetilde{\bQ}^{(\ell)}_0 \bb} = m_\mathrm{SC}(\tilde{z}) \scal{\ba}{\bb} + \bigO[\tilde{z}](N^{-\min(\frac{d - 2}{2}, 1)})$.

\subsection{Isolated Eigenvalues}

We seek eigenvalues of $\frac{1}{\sigma_N} \left[ \bT^{(\ell)} \bT^{(\ell) \top} - \mu^{(\ell)}_N \bI_{n_\ell} \right]$ which stay outside the support of the semicircle distribution $[-2, 2]$. That is, we seek $\tilde{\xi} \in \bbR \setminus [-2, 2]$ such that
\[
\det \left( \frac{1}{\sigma_N} \left[ \bT^{(\ell)} \bT^{(\ell) \top} - \mu^{(\ell)}_N \bI_{n_\ell} \right] - \tilde{\xi} \bI_{n_\ell} \right) = 0.
\]
Using the expansion $\bT^{(\ell)} = \bP^{(\ell)} + \frac{1}{\sqrt{N}} \bN^{(\ell)}$, this is equivalent to
\begin{multline*}
\det \left( \frac{1}{\sigma_N} \left( \bP^{(\ell)} \bP^{(\ell) \top} + \frac{1}{\sqrt{N}} \bP^{(\ell)} \bN^{(\ell) \top} + \frac{1}{\sqrt{N}} \bN^{(\ell)} \bP^{(\ell) \top} \right) \widetilde{\bQ}^{(\ell)}_0(\tilde{\xi}) + \bI_{n_\ell} \right) \\
\times \det \left( \frac{1}{\sigma_N} \left[ \frac{1}{N} \bN^{(\ell)} \bN^{(\ell) \top} - \mu^{(\ell)}_N \bI_{n_\ell} \right] - \tilde{\xi} \bI_{n_\ell} \right) = 0
\end{multline*}
where the second determinant is non-zero for $N$ large enough from the confinement of the spectrum proven in Section \ref{proof:thm:ed:confinement}. Then, we know that $\bscrP = \tucker{\bscrH}{\bX^{(1)}, \ldots, \bX^{(d)}}$ therefore $\bP^{(\ell)} = \bX^{(\ell)} \bL^{(\ell)}$ with $\bL^{(\ell)} = \bH^{(\ell)} \left( \bigotimes_{\ell' \neq \ell} \bX^{(\ell') \top} \right)$ and we can write
\[
\frac{1}{\sigma_N} \left( \bP^{(\ell)} \bP^{(\ell) \top} + \frac{1}{\sqrt{N}} \bP^{(\ell)} \bN^{(\ell) \top} + \frac{1}{\sqrt{N}} \bN^{(\ell)} \bP^{(\ell) \top} \right)
\]
as the matrix product $\begin{bsmallmatrix} \bX^{(\ell)} & \bX^{(\ell)} & \frac{1}{\sigma_N \sqrt{N}} \bN^{(\ell)} \bL^{(\ell) \top} \end{bsmallmatrix} \begin{bsmallmatrix} \frac{1}{\sigma_N} \bH^{(\ell)} \bH^{(\ell) \top} \bX^{(\ell) \top} \\ \frac{1}{\sigma_N \sqrt{N}} \bL^{(\ell)} \bN^{(\ell) \top} \\ \bX^{(\ell) \top} \end{bsmallmatrix}$. Thus, with Sylvester's identity ($\det(\bI_n + \bA \bB) = \det(\bI_k + \bB \bA)$), we are left to evaluate a $3 r_\ell \times 3 r_\ell$ determinant:
\[
\det(\bI_{3 r_\ell} + \bM) = 0
\]
where $\bM$ is the following matrix
\[
\begin{bsmallmatrix}
\frac{1}{\sigma_N} \bH^{(\ell)} \bH^{(\ell) \top} \bX^{(\ell) \top} \widetilde{\bQ}^{(\ell)}_0(\tilde{\xi}) \bX^{(\ell)} & \frac{1}{\sigma_N} \bH^{(\ell)} \bH^{(\ell) \top} \bX^{(\ell) \top} \widetilde{\bQ}^{(\ell)}_0(\tilde{\xi}) \bX^{(\ell)} & \frac{1}{\sigma_N^2 \sqrt{N}} \bH^{(\ell)} \bH^{(\ell) \top} \bX^{(\ell) \top} \widetilde{\bQ}^{(\ell)}_0(\tilde{\xi}) \bN^{(\ell)} \bL^{(\ell) \top} \\
\frac{1}{\sigma_N \sqrt{N}} \bL^{(\ell)} \bN^{(\ell) \top} \widetilde{\bQ}^{(\ell)}_0(\tilde{\xi}) \bX^{(\ell)} & \frac{1}{\sigma_N \sqrt{N}} \bL^{(\ell)} \bN^{(\ell) \top} \widetilde{\bQ}^{(\ell)}_0(\tilde{\xi}) \bX^{(\ell)} & \frac{1}{\sigma_N^2 N} \bL^{(\ell)} \bN^{(\ell) \top} \widetilde{\bQ}^{(\ell)}_0(\tilde{\xi}) \bN^{(\ell)} \bL^{(\ell) \top} \\
\bX^{(\ell) \top} \widetilde{\bQ}^{(\ell)}_0(\tilde{\xi}) \bX^{(\ell)} & \bX^{(\ell) \top} \widetilde{\bQ}^{(\ell)}_0(\tilde{\xi}) \bX^{(\ell)} & \frac{1}{\sigma_N \sqrt{N}} \bX^{(\ell) \top} \widetilde{\bQ}^{(\ell)}_0(\tilde{\xi}) \bN^{(\ell)} \bL^{(\ell) \top}
\end{bsmallmatrix}.
\]
From the convergence of bilinear forms and the orthonormality of the columns of $\bX^{(\ell)}$, we have $\bX^{(\ell) \top} \widetilde{\bQ}^{(\ell)}_0(\tilde{\xi}) \bX^{(\ell)} \to m_\mathrm{SC}(\tilde{\xi}) \bI_{r_\ell}$ almost surely. Moreover, we can see that the matrix $\frac{1}{\sigma_N \sqrt{N}} \bL^{(\ell)} \bN^{(\ell) \top} \widetilde{\bQ}^{(\ell)}_0(\tilde{\xi}) \bX^{(\ell)}$ vanishes almost surely as $N \to +\infty$ since $\norm{\bL^{(\ell)}} = \bigO(N^{\frac{d - 2}{4}})$ and $\norm{\bN^{(\ell)} \ba} = \bigO(\sqrt{N})$ almost surely\footnote{This fact is not so easy to see (note that it does not depend on $d$!). If we naively upper bound $\norm{\bN^{(\ell)} \ba}^2$ by $\norm{\bN^{(\ell)}}^2 \norm{\ba}^2 = \bigO(N^{\frac{d - 1}{2}})$, we do not find the desired upper bound. Instead, we can remark that $\norm{\bN^{(\ell)} \ba}^2 = \ba^\top \bN^{(\ell) \top} \bN^{(\ell)} \ba = \ba^\top \bV \bD \bV^\top \ba$ where $\bD = \Diag(\lambda_1(\bN^{(\ell) \top} \bN^{(\ell)}), \ldots, \lambda_{n_\ell}(\bN^{(\ell) \top} \bN^{(\ell)}))$ (we assume $N$ large enough so that $\prod_{\ell' \neq \ell} n_{\ell'} > n_\ell$ thus $\lambda_i(\bN^{(\ell) \top} \bN^{(\ell)}) = 0$ for $i > n_\ell$) and $\bV$ follows a uniform distribution on the Stiefel manifold $V_{n_\ell}(\bbR^{\prod_{\ell' \neq \ell} n_{\ell'}})$ \citep[Theorem 2.2.1]{chikuse_statistics_2003}. Therefore $\frac{1}{N} \norm{\bN^{(\ell)} \ba}^2 \leqslant (\mu^{(\ell)}_N + 2 \sigma_N) \ba^\top \bV \bV^\top \ba$. Without loss of generality, we can assume that $a_i = \norm{\ba} \delta_{1, i}$ (replace $\ba$ and $\bV$ by $\bO \ba$ and $\bO \bV$ for a well-chosen orthogonal matrix $\bO$). From \citet{mardia_uniform_1977}, we know that $[\bV \bV^\top]_{1, 1}$ follows a beta distribution with parameters $\frac{n_\ell}{2}, \frac{\prod_{\ell' \neq \ell} n_{\ell'} - n_\ell}{2}$ so its moments are given by $\esp{[\bV \bV^\top]_{1, 1}^k} = \prod_{r = 0}^{k - 1} \frac{n_\ell + 2 r}{\prod_{\ell' \neq \ell} n_{\ell'} + 2 r}$ for all $k \geqslant 1$. This is enough to see that $(\mu^{(\ell)}_N + 2 \sigma_N) \esp{[\bV \bV^\top]_{1, 1}} = (\mu^{(\ell)}_N + 2 \sigma_N) \frac{n_\ell}{\prod_{\ell' \neq \ell} n_{\ell'}} = \bigO(1)$ and $(\mu^{(\ell)}_N + 2 \sigma_N)^4 \esp{([\bV \bV^\top]_{1, 1} - \esp{[\bV \bV^\top]_{1, 1}})^4} = \bigO(N^{-2})$, whence the almost sure statement $\frac{1}{N} \norm{\bN^{(\ell)} \ba}^2 = \bigO(1)$. \label{ftn:norm_Na}} for all bounded (sequences of) vectors $\ba \in \bbR^{\prod_{\ell' \neq \ell} n_{\ell'}}$. The only term which remains to evaluate is the block $(2, 3)$:
\begin{multline*}
\frac{1}{\sigma_N^2 N} \Norm{\bL^{(\ell)} \bN^{(\ell) \top} \widetilde{\bQ}^{(\ell)}_0(\tilde{\xi}) \bN^{(\ell)} \bL^{(\ell) \top}} \\
\leqslant \frac{1}{\sigma_N} \frac{\Norm{\bH^{(\ell)}}^2}{\sigma_N} \frac{\Norm{\bN^{(\ell)} \left( \bigotimes_{\ell' \neq \ell} \bX^{(\ell')} \right)}^2}{N} \Norm{\widetilde{\bQ}^{(\ell)}_0(\tilde{\xi})} \xrightarrow[N \to +\infty]{\text{a.s.}} 0.
\end{multline*}

Eventually, as the determinant is a continuous function in the entries of the matrix, we have, in the large $N$ limit,
\[
\det \begin{bmatrix}
\frac{m_\mathrm{SC}(\tilde{\xi})}{\sigma_N} \bH^{(\ell)} \bH^{(\ell) \top} + \bI_{r_\ell} & \frac{m_\mathrm{SC}(\tilde{\xi})}{\sigma_N} \bH^{(\ell)} \bH^{(\ell) \top} & \bzero_{r_\ell \times r_\ell} \\
\bzero_{r_\ell \times r_\ell} & \bI_{r_\ell} & \bzero_{r_\ell \times r_\ell} \\
m_\mathrm{SC}(\tilde{\xi}) \bI_{r_\ell} & m_\mathrm{SC}(\tilde{\xi}) \bI_{r_\ell} & \bI_{r_\ell}
\end{bmatrix} = 0.
\]
Using twice the relation $\det \begin{bsmallmatrix} \bA & \bB \\ \bC & \bD \end{bsmallmatrix} = \det(\bA \bD^{-1} - \bB \bD^{-1} \bC \bD)$ when $\bD$ is invertible, this simplifies into
\[
\det \left( \frac{m_\mathrm{SC}(\tilde{\xi})}{\sigma_N} \bH^{(\ell)} \bH^{(\ell) \top} + \bI_{r_\ell} \right) = 0
\]
Thus, we seek $\tilde{\xi}^{(\ell)}_{q_\ell} \in \bbR \setminus [-2, 2]$ such that
\[
\frac{m_\mathrm{SC}(\tilde{\xi}^{(\ell)}_{q_\ell})}{\sigma_N} s_{q_\ell}^2(\bP^{(\ell)}) + 1 = 0, \qquad q_\ell \in [r_\ell].
\]
Injecting the expression $m_\mathrm{SC}(\tilde{\xi}^{(\ell)}_{q_\ell}) = -\frac{\sigma_N}{s_{q_\ell}^2(\bP^{(\ell)})}$ into the equation $m_\mathrm{SC}^2(\tilde{\xi}^{(\ell)}_{q_\ell}) + \tilde{\xi}^{(\ell)}_{q_\ell} m_\mathrm{SC}(\tilde{\xi}^{(\ell)}_{q_\ell}) + 1 = 0$ yields
\[
\frac{\sigma_N^2}{s_{q_\ell}^4(\bP^{(\ell)})} - \tilde{\xi}^{(\ell)}_{q_\ell} \frac{\sigma_N}{s_{q_\ell}^2(\bP^{(\ell)})} + 1 = 0 \iff \tilde{\xi}^{(\ell)}_{q_\ell} = \frac{s_{q_\ell}^2(\bP^{(\ell)})}{\sigma_N} + \frac{\sigma_N}{s_{q_\ell}^2(\bP^{(\ell)})}.
\]
As $\tilde{\xi}^{(\ell)}_{q_\ell} > 0$ by definition, it must be strictly greater than $2$ (the right edge of the semicircle). This is true only if $\rho^{(\ell)}_{q_\ell} \eqdef \frac{s_{q_\ell}^2(\bP^{(\ell)})}{\sigma_N} > 1$.

\subsection{Eigenvector Alignments}

Let $\hat{\bu}^{(\ell)}_{i_\ell}$, $i_\ell \in [n_\ell]$, denote the $i_\ell$-th left singular vector of $\bT^{(\ell)}$ (sorted in non-increasing order of its corresponding singular value). From the spectral decomposition $\bT^{(\ell)} \bT^{(\ell) \top} = \sum_{i_\ell = 1}^{n_\ell} s_{i_\ell}^2(\bT^{(\ell)}) \hat{\bu}^{(\ell)}_{i_\ell} \hat{\bu}^{(\ell) \top}_{i_\ell}$, we have,
\[
\widetilde{\bQ}^{(\ell)}(\tilde{z}) = \sum_{i_\ell = 1}^{n_\ell} \frac{\hat{\bu}^{(\ell)}_{q_\ell} \hat{\bu}^{(\ell) \top}_{q_\ell}}{\frac{1}{\sigma_N} \left[ s_{i_\ell}^2(\bT^{(\ell)}) - \mu^{(\ell)}_N \right] - \tilde{z}}.
\]
If $\rho^{(\ell)}_{q_\ell} > 1$, $q_\ell \in [r_\ell]$, then $s_{q_\ell}^2(\bT^{(\ell)})$ is an isolated eigenvalue in the spectrum of $\bT^{(\ell)} \bT^{(\ell) \top}$ and \linebreak $\frac{1}{\sigma_N} \left[ s_{q_\ell}^2(\bT^{(\ell)}) - \mu^{(\ell)}_N \right] \xrightarrow[N \to +\infty]{\text{a.s.}} \tilde{\xi}^{(\ell)}_{q_\ell}$. Hence, for any positively-oriented simple closed complex contour $\gamma^{(\ell)}_{q_\ell}$ circling around $\tilde{\xi}^{(\ell)}_{q_\ell}$, leaving all the other $\tilde{\xi}^{(\ell)}_{q_\ell'}$, $q_\ell' \neq q_\ell$, outside and not crossing $[-2, 2]$, Cauchy's integral formula yields, for $N$ large enough and any $\ba \in \bbR^{n_\ell}$,
\[
\Abs{\ba^\top \hat{\bu}^{(\ell)}_{q_\ell}}^2 = -\frac{1}{2 \rmi \pi} \oint_{\gamma^{(\ell)}_{q_\ell}} \ba^\top \widetilde{\bQ}^{(\ell)}(\tilde{z}) \ba ~\rmd \tilde{z} \xrightarrow[N \to +\infty]{\text{a.s.}} -\frac{1}{2 \rmi \pi} \oint_{\gamma^{(\ell)}_{q_\ell}} \ba^\top \bar{\bQ}^{(\ell)}(\tilde{z}) \ba ~\rmd \tilde{z}
\]
by the dominated convergence theorem since, for all $\tilde{z} \in \gamma^{(\ell)}_{q_\ell}$, $\ba^\top \widetilde{\bQ}^{(\ell)}(\tilde{z}) \ba \to \ba^\top \bar{\bQ}^{(\ell)}(\tilde{z}) \ba$ almost surely as $N \to +\infty$ by definition of the deterministic equivalent (Definition \ref{def:matrix_equivalent}) and $\tilde{z} \in \gamma^{(\ell)}_{q_\ell} \mapsto \abs{\ba^\top \widetilde{\bQ}^{(\ell)}(\tilde{z}) \ba}$ is almost surely bounded for $N$ large enough because we can choose $\gamma^{(\ell)}_{q_\ell}$ such that $\Dist(\tilde{\xi}^{(\ell)}_{q_\ell}, \gamma^{(\ell)}_{q_\ell}) \geqslant \varepsilon > 0$ and therefore $\abs{\ba^\top \widetilde{\bQ}^{(\ell)}(\tilde{z}) \ba} \leqslant \norm{\ba}^2 \norm{\widetilde{\bQ}^{(\ell)}(\tilde{z})} \leqslant \norm{\ba}^2 / \varepsilon$ almost surely.

Using residue calculus, we can compute,
\[
-\frac{1}{2 \rmi \pi} \oint_{\gamma^{(\ell)}_{q_\ell}} \ba^\top \bar{\bQ}^{(\ell)}(\tilde{z}) \ba ~\rmd \tilde{z} = -\lim_{\tilde{z} \to \tilde{\xi}^{(\ell)}_{q_\ell}} \left( \tilde{z} - \tilde{\xi}^{(\ell)}_{q_\ell} \right) \ba^\top \left( \frac{1}{\sigma_N} \bP^{(\ell)} \bP^{(\ell) \top} + \frac{1}{m_\mathrm{SC}(\tilde{z})} \bI_{n_\ell} \right)^{-1} \ba.
\]
Note that $\bP^{(\ell)} \bP^{(\ell) \top} = \bX^{(\ell)} \bH^{(\ell)} \bH^{(\ell) \top} \bX^{(\ell) \top}$ and there exist an $r_\ell \times r_\ell$ orthogonal matrix $\bO^{(\ell)}$ such that $\bH^{(\ell)} \bH^{(\ell) \top} = \bO^{(\ell)} \bLambda^{(\ell)} \bO^{(\ell) \top}$ with $\bLambda^{(\ell)} = \Diag(s_1^2(\bP^{(\ell)}), \ldots, s_{r_\ell}^2(\bP^{(\ell)}))$. Hence,
\begin{multline*}
\ba^\top \left( \frac{1}{\sigma_N} \bP^{(\ell)} \bP^{(\ell) \top} + \frac{1}{m_\mathrm{SC}(\tilde{z})} \bI_{n_\ell} \right)^{-1} \ba \\
= \ba^\top \bX^{(\ell)} \bO^{(\ell)} \left( \frac{1}{\sigma_N} \bLambda^{(\ell)} + \frac{1}{m_\mathrm{SC}(\tilde{z})} \bI_{n_\ell} \right)^{-1} \bO^{(\ell) \top} \bX^{(\ell) \top} \ba.
\end{multline*}
Let us therefore compute the following quantity,
\[
-\lim_{\tilde{z} \to \tilde{\xi}^{(\ell)}_{q_\ell}} \left( \tilde{z} - \tilde{\xi}^{(\ell)}_{q_\ell} \right) \left( \frac{s_{q_\ell'}^2(\bP^{(\ell)})}{\sigma_N} + \frac{1}{m_\mathrm{SC}(\tilde{z})} \right)^{-1} = \left\{ \begin{array}{ll}
0 & \text{if}~ q_\ell' \neq q_\ell \\
\zeta^{(\ell)}_{q_\ell} & \text{if}~ q_\ell' = q_\ell
\end{array} \right., \qquad q_\ell' \in [r_\ell],
\]
where we have used the fact that $m_\mathrm{SC}(\tilde{\xi}^{(\ell)}_{q_\ell'}) = -\frac{\sigma_N}{s_{q_\ell'}^2(\bP^{(\ell)})}$. In order to handle the case $q_\ell' = q_\ell$, we use L'Hôpital's rule,
\begin{align*}
\zeta^{(\ell)}_{q_\ell} &= -\left( \frac{\rmd}{\rmd \tilde{z}} \left[ \frac{s_{q_\ell}^2(\bP^{(\ell)})}{\sigma_N} + \frac{1}{m_\mathrm{SC}(\tilde{z})} \right]_{\tilde{z} = \tilde{\xi}^{(\ell)}_{q_\ell}} \right)^{-1} \\
&= \frac{m_\mathrm{SC}^2(\tilde{\xi}^{(\ell)}_{q_\ell})}{m_\mathrm{SC}'(\tilde{\xi}^{(\ell)}_{q_\ell})} \\
&= \frac{\sigma_N^2}{s_{q_\ell}^4(\bP^{(\ell)}) m_\mathrm{SC}'(\tilde{\xi}^{(\ell)}_{q_\ell})}.
\end{align*}
In order to compute $m_\mathrm{SC}'(\tilde{\xi}^{(\ell)}_{q_\ell})$, let us differentiate the relation $m_\mathrm{SC}^2(\tilde{z}) + \tilde{z} m_\mathrm{SC}(\tilde{z}) + 1 = 0$,
\begin{gather*}
2 m_\mathrm{SC}'(\tilde{z}) m_\mathrm{SC}(\tilde{z}) + m_\mathrm{SC}(\tilde{z}) + \tilde{z} m_\mathrm{SC}'(\tilde{z}) = 0, \\
m_\mathrm{SC}'(\tilde{z}) = -\frac{m_\mathrm{SC}(\tilde{z})}{2 m_\mathrm{SC}(\tilde{z}) + \tilde{z}}.
\end{gather*}
Hence,
\begin{align*}
m_\mathrm{SC}'(\tilde{\xi}^{(\ell)}_{q_\ell}) &= -\frac{-\frac{\sigma_N}{s_{q_\ell}^2(\bP^{(\ell)})}}{-2 \frac{\sigma_N}{s_{q_\ell}^2(\bP^{(\ell)})} + \frac{s_{q_\ell}^2(\bP^{(\ell)})}{\sigma_N} + \frac{\sigma_N}{s_{q_\ell}^2(\bP^{(\ell)})}} \\
&= \frac{1}{\frac{s_{q_\ell}^4(\bP^{(\ell)})}{\sigma_N^2} - 1}.
\end{align*}
Back to our previous expression of $\zeta^{(\ell)}_{q_\ell}$, we now have,
\[
\zeta^{(\ell)}_{q_\ell} = 1 - \frac{\sigma_N^2}{s_{q_\ell}^4(\bP^{(\ell)})}.
\]

Therefore, for all $\ba \in \bbR^{n_\ell}$,
\[
-\frac{1}{2 \rmi \pi} \oint_{\gamma^{(\ell)}_{q_\ell}} \ba^\top \bar{\bQ}^{(\ell)}(\tilde{z}) \ba ~\rmd \tilde{z} = \ba^\top \bX^{(\ell)} \bO^{(\ell)} \bZ^{(\ell)}_{q_\ell} \bO^{(\ell) \top} \bX^{(\ell) \top} \ba
\]
where $\bZ^{(\ell)}_{q_\ell}$ is an $r_\ell \times r_\ell$ matrix with all its entries equal to $0$ except $[\bZ^{(\ell)}_{q_\ell}]_{q_\ell, q_\ell} = \zeta^{(\ell)}_{q_\ell}$. Thus, summing the alignments of $\hat{\bu}^{(\ell)}_{q_\ell}$ with each column of $\bX^{(\ell)}$ yields
\[
\Norm{\bX^{(\ell) \top} \hat{\bu}^{(\ell)}_{q_\ell}}^2 \xrightarrow[N \to +\infty]{\text{a.s.}} \zeta^{(\ell)}_{q_\ell} \sum_{q_\ell' = 1}^{r_\ell} O^{(\ell) 2}_{q_\ell', q_\ell} = \zeta^{(\ell)}_{q_\ell}.
\]

\section{Proof of Lemma \ref{lem:bound}}
\label{proof:lem:bound}

Our proof of Lemma \ref{lem:bound} uses the notion of $\varepsilon$-covering. An $\varepsilon$-covering of a \emph{compact} set $\calK$ for the norm $\norm{\cdot}$ is a \emph{finite} set $\calC \subset \calK$ such that for all $x \in \calK$, there exists $\bar{x} \in \calC$ such that $\norm{x - \bar{x}} \leqslant \varepsilon$. We also define the covering number $N(\varepsilon, \calK, \norm{\cdot})$ as the smallest possible number of elements in $\calC$.

Moreover, we recall the definition of the Gamma function $\Gamma(s) = \int_0^{+\infty} t^{s - 1} e^{-t} \rmd t$ and the (upper) incomplete Gamma function $\Gamma(s, x) = \int_x^{+\infty} t^{s - 1} e^{-t} \rmd t$ for $s > 0$ and $x \geqslant 0$.

For our proof, we need to introduce a few preliminary results which are stated and proven below (except Lemma \ref{lem:covering} for which a reference is given).

\begin{lemma} \label{lem:norm_epsilon}
For $\ell \in [d]$ and $\varepsilon > 0$, let $\bDelta^{(\ell)} \in \bbR^{n_\ell \times r_\ell}$ be such that $\norm{\bDelta^{(\ell)}} \leqslant \varepsilon$ and $\bV^{(\ell)} \in V_{r_\ell}(\bbR^{n_\ell})$ be the matrix of its left singular vectors. For all $\bA^{(\ell')} \in \bbR^{n_{\ell'} \times r_{\ell'}}$, $\ell' \neq \ell$,
\[
\Norm{\bscrN(\bA^{(1)}, \ldots, \bDelta^{(\ell)}, \ldots, \bA^{(d)})}_\rmF \leqslant \varepsilon \Norm{\bscrN(\bA^{(1)}, \ldots, \bV^{(\ell)}, \ldots, \bA^{(d)})}_\rmF.
\]
\end{lemma}
\begin{proof}
Let $\bV^{(\ell)} \bSigma^{(\ell)} \bW^{(\ell) \top}$ be the singular value decomposition of $\bDelta^{(\ell)}$. We have,
\begin{multline*}
\Norm{\bscrN(\bA^{(1)}, \ldots, \bDelta^{(\ell)}, \ldots, \bA^{(d)})}_\rmF = \Norm{\bDelta^{(\ell) \top} \bN^{(\ell)} \bigkron_{\ell' \neq \ell} \bA^{(\ell')}}_\rmF \\
= \Norm{\bW^{(\ell) \top} \bSigma^{(\ell)} \bV^{(\ell) \top} \bN^{(\ell)} \bigkron_{\ell' \neq \ell} \bA^{(\ell')}}_\rmF \leqslant \underbrace{\Norm{\bW^{(\ell) \top} \bSigma^{(\ell)}}}_{= \varepsilon} \underbrace{\Norm{\bV^{(\ell) \top} \bN^{(\ell)} \bigkron_{\ell' \neq \ell} \bA^{(\ell')}}_\rmF}_{= \Norm{\bscrN(\bA^{(1)}, \ldots, \bV^{(\ell)}, \ldots, \bA^{(d)})}_\rmF}
\end{multline*}
using the fact that $\norm{\bA \bB}_\rmF \leqslant \norm{\bA} \norm{\bB}_\rmF$.
\end{proof}

\begin{lemma} \label{lem:chi2}
Given $\bA{(\ell)} \in V_{r_\ell}(\bbR^{n_\ell})$, $\ell \in [d]$,
\[
\Norm{\bscrN(\bA^{(1)}, \ldots, \bA^{(d)})}_\rmF^2 \sim \chi^2 \left( \prod_{\ell \in [d]} r_\ell \right).
\]
\end{lemma}
\begin{proof}
Firstly, observe that, for all $(q_1, \ldots, q_d) \in \bigtimes_{\ell \in [d]} [r_\ell]$,
\[
[\bscrN(\bA^{(1)}, \ldots, \bA^{(d)})]_{q_1, \ldots, q_d} = \sum_{i_1, \ldots, i_d = 1}^{n_1, \ldots, n_d} \scrN_{i_1, \ldots, i_d} A^{(1)}_{i_1, q_1} \ldots A^{(d)}_{i_d, q_d} \sim \calN(0, 1)
\]
since $\sum_{i_\ell = 1}^{n_\ell} A^{(\ell) 2}_{i_\ell, q_\ell} = 1$ for all $\ell \in [d]$. Then, we show that all the entries of $\bscrN(\bA^{(1)}, \ldots, \bA^{(d)})$ are independent because their covariance is identity,
\begin{align*}
&\Esp{[\bscrN(\bA^{(1)}, \ldots, \bA^{(d)})]_{q_1, \ldots, q_d} [\bscrN(\bA^{(1)}, \ldots, \bA^{(d)})]_{q_1', \ldots, q_d'}} \\
&= \sum_{i_1, \ldots, i_d = 1}^{n_1, \ldots, n_d} \sum_{i_1', \ldots, i_d' = 1}^{n_1, \ldots, n_d} \Esp{\scrN_{i_1, \ldots, i_d} \scrN_{i_1', \ldots, i_d'}} A^{(1)}_{i_1, q_1} A^{(1)}_{i_1', q_1'} \ldots A^{(d)}_{i_d, q_d} A^{(d)}_{i_d', q_d'} \\
&= \sum_{i_1, \ldots, i_d = 1}^{n_1, \ldots, n_d} A^{(1)}_{i_1, q_1} A^{(1)}_{i_1, q_1'} \ldots A^{(d)}_{i_d, q_d} A^{(d)}_{i_d, q_d'} \\
&= \left\{ \begin{array}{ll}
1 & \text{if}~ (q_1, \ldots, q_d) = (q_1', \ldots, q_d') \\
0 & \text{otherwise}
\end{array} \right. .
\end{align*}
Hence, the result follows from the fact that $\norm{\bscrN(\bA^{(1)}, \ldots, \bA^{(d)})}_\rmF^2$ is the sum of $\prod_{\ell \in [d]} r_\ell$ squared independent $\calN(0, 1)$ variables.
\end{proof}

\begin{lemma}[\citealp{hinrichs_entropy_2017}, Lemma 4.1] \label{lem:covering}
For $0 < \varepsilon < 1$, we have the following upper bound on the $\varepsilon$-covering number of the Stiefel manifold $V_r(\bbR^n)$ for the spectral norm $\norm{\cdot}$,
\[
N(\varepsilon, V_r(\bbR^n), \Norm{\cdot}) \leqslant \left[ \frac{C}{\varepsilon} \right]^{r \left( n - \frac{r + 1}{2} \right)}.
\]
where $C > 0$ is a universal constant.
\end{lemma}

\begin{lemma} \label{lem:gamma}
$\Gamma(s, x) \leqslant \max(1, e^{s - 1}) \Gamma(s) e^{-x / 2}$ for all $x \geqslant 0$ and $s > 0$.
\end{lemma}
\begin{proof}
Given $s > 0$, consider the function $f : x \in [0, +\infty[ \mapsto \frac{\Gamma(s, x)}{C e^{-x / 2}}$ with $C > 0$. Our goal is to show that $0 < f \leqslant 1$ when $C$ is well chosen. $f$ is continuously differentiable on $[0, +\infty[$ and
\[
f'(x) = \frac{1}{C e^{-x / 2}} \left( -x^{s - 1} e^{-x} + \frac{1}{2} \Gamma(s, x) \right) \geqslant 0 \iff \Gamma(s, x) - 2 x^{s - 1} e^{-x} \geqslant 0.
\]
Consider the function $g : x \in [0, +\infty[ \mapsto \Gamma(s, x) - 2 x^{s - 1} e^{-x}$. $g$ is also continuously differentiable on $[0, +\infty[$ and
\[
g'(x) = x^{s - 1} e^{-x} - 2 \left( s - 1 \right) x^{s - 2} e^{-x} \geqslant 0 \iff x \geqslant 2 \left( s - 1 \right).
\]
We distinguish two cases.
\begin{enumerate}
\item If $0 < s \leqslant 1$, then $g$ increases monotonically on $[0, +\infty[$. Since $\lim_{x \to +\infty} g(x) = 0$, we necessarily have $g(x) \leqslant 0$ for all $x \in [0, +\infty[$. Hence, $f(x) \leqslant f(0) = \frac{\Gamma(s)}{C}$ and we can choose $C = \Gamma(s)$.
\item If $s > 1$, our conclusion stems from the following table.
\begin{center} \begin{tikzpicture}
\tkzTabInit{$x$/.5, $g'(x)$/.5, $g(x)$/1.5}{$0$, $2 \left( s - 1 \right)$, $+\infty$}
\tkzTabLine{, -, z, +,}
\tkzTabVar{+/$\Gamma(s)$, -/$g(2 \left( s - 1 \right))$, +/0}
\end{tikzpicture} \end{center}
Since $g$ is strictly increasing on $\left[ 2 \left( s - 1 \right), +\infty \right[$ and $\lim_{x \to +\infty} g(x) = 0$, we necessarily have $g(2 \left( s - 1 \right)) < 0$. Hence, since $\Gamma(s) > 0$, the equation $g(x) = 0$ has a unique solution on $[0, +\infty[$ and it lies between $0$ and $2 \left( s - 1 \right)$. Let $x_0(s) \in \left[ 0, 2 \left( s - 1 \right) \right]$ be this unique solution. Then,
\[
\sup_{[0, +\infty[} f = f(x_0(s)) = \frac{\Gamma(s, x_0(s))}{C e^{-x_0(s) / 2}}
\]
and we can choose $C = \Gamma(s, x_0(s)) e^{x_0(s) / 2}$. The final result follows from $\Gamma(s, x_0(s)) \leqslant \Gamma(s)$ and $e^{x_0(s) / 2} \leqslant e^{s - 1}$.
\end{enumerate}
\end{proof}

We are now ready to prove Lemma \ref{lem:bound}.

Let $\varepsilon > 0$ and $\calC_\ell$ be an $\varepsilon$-covering of $V_{r_\ell}(\bbR^{n_\ell})$ for the spectral norm $\norm{\cdot}$, $\ell \in [d]$. Since $\bigtimes_{\ell \in [d]} V_{r_\ell}(\bbR^{n_\ell})$ is compact, it contains an element $(\bA^{(1)}_\star, \ldots, \bA^{(d)}_\star)$ such that
\[
\sup_{\bA^{(\ell)} \in V_{r_\ell}(\bbR^{n_\ell}),~ \ell \in [d]} \Norm{\bscrN(\bA^{(1)}, \ldots, \bA^{(d)})}_\rmF = \Norm{\bscrN(\bA^{(1)}_\star, \ldots, \bA^{(d)}_\star)}_\rmF.
\]
Let $\bar{\bA}^{(\ell)} \in \calC_\ell$ be such that $\bA^{(\ell)}_\star = \bar{\bA}^{(\ell)} + \bDelta^{(\ell)}$ with $\norm{\bDelta^{(\ell)}} \leqslant \varepsilon$. Then, using the triangle inequality, Lemma \ref{lem:norm_epsilon} and the optimality of $(\bA^{(1)}_\star, \ldots, \bA^{(d)}_\star)$, we have,
\[
\Norm{\bscrN(\bar{\bA}^{(1)} + \bDelta^{(1)}, \ldots, \bar{\bA}^{(d)} + \bDelta^{(d)})}_\rmF \leqslant \Norm{\bscrN(\bar{\bA}^{(1)}, \ldots, \bar{\bA}^{(d)})}_\rmF + S \Norm{\bscrN(\bA^{(1)}_\star, \ldots, \bA^{(d)}_\star)}_\rmF
\]
with $S \eqdef \sum_{k = 1}^d \binom{d}{k} \varepsilon^k \leqslant \sum_{k = 1}^d \varepsilon^k \frac{d^k}{k!} \leqslant e^{\varepsilon d} -1$. Hence, choosing $\varepsilon = \frac{1}{d} \log \frac{3}{2}$, we get,
\[
\Norm{\bscrN(\bA^{(1)}_\star, \ldots, \bA^{(d)}_\star)}_\rmF \leqslant 2 \Norm{\bscrN(\bar{\bA}^{(1)}, \ldots, \bar{\bA}^{(d)})}_\rmF
\]
and, from the union bound, for any $t \geqslant 0$,
\begin{align*}
\bbP \left( \Norm{\bscrN(\bA^{(1)}_\star, \ldots, \bA^{(d)}_\star)}_\rmF \geqslant t \right) &\leqslant \bbP \left( \bigcup_{\bA^{(\ell)} \in \calC_\ell,~ \ell \in [d]} \left\{ \Norm{\bscrN(\bA^{(1)}, \ldots, \bA^{(d)})}_\rmF \geqslant \frac{t}{2} \right\} \right) \\
&\leqslant \sum_{\bA^{(\ell)} \in \calC_\ell,~ \ell \in [d]} \bbP \left( \Norm{\bscrN(\bA^{(1)}, \ldots, \bA^{(d)})}_\rmF \geqslant \frac{t}{2} \right).
\end{align*}
Thus, combining Lemma \ref{lem:chi2} and \ref{lem:covering}, we have,
\[
\bbP \left( \Norm{\bscrN(\bA^{(1)}_\star, \ldots, \bA^{(d)}_\star)}_\rmF \geqslant t \right) \leqslant \left[ \frac{C d}{\log \frac{3}{2}} \right]^{\sum_{\ell = 1}^d r_\ell \left( n_\ell - \frac{r_\ell + 1}{2} \right)} \bbP \left( X \geqslant \frac{t^2}{4} \right)
\]
where $X$ is a random variable following a $\chi^2(\prod_{\ell \in [d]} r_\ell)$ distribution. Eventually, the probability on the right-hand side can be bounded using Lemma \ref{lem:gamma},
\[
\bbP \left( X \geqslant \frac{t^2}{4} \right) = \frac{\Gamma \left( \frac{1}{2} \prod_{\ell \in [d]} r_\ell, \frac{t^2}{8} \right)}{\Gamma \left( \frac{1}{2} \prod_{\ell \in [d]} r_\ell \right)} \leqslant \max(1, e^{\frac{1}{2} \prod_{\ell \in [d]} r_\ell - 1}) e^{-t^2 / 16}.
\]
We get the result stated in Lemma \ref{lem:bound} with
\[
t^2 = 16 \left[ \left( \sum_{\ell = 1}^d r_\ell \left( n_\ell - \frac{r_\ell + 1}{2} \right) \right) \log \frac{C d}{\log \frac{3}{2}} + \log \left( \frac{1}{\delta} \max \left( 1, e^{\frac{1}{2} \prod_{\ell = 1}^d r_\ell - 1} \right) \right) \right].
\]

\section{Proof of Theorem \ref{thm:hooi}}
\label{proof:thm:hooi}

Recall the decomposition $\bscrP = \tucker{\bscrH}{\bX^{(1)}, \ldots, \bX^{(d)}}$. We use the following lemma whenever we state that $\norm{\bscrP(\bU^{(1)}, \ldots, \bU^{(d)})}_\rmF = \bigO(\norm{\bscrP}_\rmF)$ as $N \to +\infty$.
\begin{lemma} \label{lem:ThetaP}
For all $(\bA^{(1)}, \ldots, \bA^{(d)}) \in \bigtimes_{\ell \in [d]} V_{r_\ell}(\bbR^{n_\ell})$,
\[
\Norm{\bscrP(\bA^{(1)}, \ldots, \bA^{(d)})}_\rmF \leqslant \Norm{\bscrP}_\rmF \prod_{\ell \in [d]} \Norm{\bX^{(\ell) \top} \bA^{(\ell)}}.
\]
\end{lemma}
\begin{proof}
The proof relies on the property $\norm{\bA \bB}_\rmF \leqslant \norm{\bA} \norm{\bB}_\rmF$.
\begin{align*}
\Norm{\bscrP(\bA^{(1)}, \ldots, \bA^{(d)})}_\rmF &= \Norm{\bA^{(1) \top} \bX^{(1)} \bH^{(1)} \bigkron_{\ell = 2}^d \bX^{(\ell) \top} \bA^{(\ell)}}_\rmF \\
&\leqslant \Norm{\bA^{(1) \top} \bX^{(1)}} \Norm{\bH^{(1)} \bigkron_{\ell = 2}^d \bX^{(\ell) \top} \bA^{(\ell)}}_\rmF \\
&= \Norm{\bA^{(1) \top} \bX^{(1)}} \Norm{\bA^{(2) \top} \bX^{(2)} \bH^{(2)} \left( \bI_{r_1} \kron \bigkron_{\ell = 3}^d \bX^{(\ell) \top} \bA^{(\ell)} \right)}_\rmF \\
&\leqslant \Norm{\bA^{(1) \top} \bX^{(1)}} \Norm{\bA^{(2) \top} \bX^{(2)}} \Norm{\bH^{(2)} \left( \bI_{r_1} \kron \bigkron_{\ell = 3}^d \bX^{(\ell) \top} \bA^{(\ell)} \right)}_\rmF \\
&\ldots \\
&\leqslant \left( \prod_{\ell = 1}^d \Norm{\bA^{(\ell) \top} \bX^{(\ell)}} \right) \underbrace{\Norm{\bH^{(d)} \bigkron_{\ell = 1}^d \bI_{r_\ell}}_\rmF}_{= \Norm{\bscrH}_\rmF = \Norm{\bscrP}_\rmF}.
\end{align*}
\end{proof}

Given $\ell \in [d]$, $\bU^{(\ell)}_1$ gathers the $r_\ell$ dominant left singular vectors of $\bT^{(\ell)} \bigkron_{\ell' \neq \ell} \bU^{(\ell')}_0$, i.e., it is solution to
\begin{equation} \label{eq:max_t}
\max_{\bU^{(\ell)} \in V_{r_\ell}(\bbR^{n_\ell})} \Norm{\bU^{(\ell) \top} \bT^{(\ell)} \bigkron_{\ell' \neq \ell} \bU^{(\ell')}_0}_\rmF^2.
\end{equation}
Consider also a solution $\widetilde{\bU}^{(\ell)}_1$ to the following related problem
\begin{equation} \label{eq:max_p}
\max_{\bU^{(\ell)} \in V_{r_\ell}(\bbR^{n_\ell})} \Norm{\bU^{(\ell) \top} \bP^{(\ell)} \bigkron_{\ell' \neq \ell} \bU^{(\ell')}_0}_\rmF^2.
\end{equation}
Observe that, using the property $\Norm{\bA \bB}_\rmF \leqslant \Norm{\bA} \Norm{\bB}_\rmF$, we have,
\begin{multline*}
\Norm{\bU^{(\ell) \top} \bP^{(\ell)} \bigkron_{\ell' \neq \ell} \bU^{(\ell')}_0}_\rmF^2 = \Norm{\bU^{(\ell) \top} \bX^{(\ell)} \bH^{(\ell)} \bigkron_{\ell' \neq \ell} \bX^{(\ell') \top} \bU^{(\ell')}_0}_\rmF^2 \\
\leqslant \Norm{\bU^{(\ell) \top} \bX^{(\ell)}}^2 \Norm{\bH^{(\ell)} \bigkron_{\ell' \neq \ell} \bX^{(\ell') \top} \bU^{(\ell')}_0}_\rmF^2 \leqslant \Norm{\bH^{(\ell)} \bigkron_{\ell' \neq \ell} \bX^{(\ell') \top} \bU^{(\ell')}_0}_\rmF^2
\end{multline*}
and this upper bound is only reached with $\bU^{(\ell)} = \bX^{(\ell)} \bO^{(\ell)}$, for any $r_\ell \times r_\ell$ orthogonal matrix $\bO^{(\ell)}$. Hence, $\widetilde{\bU}^{(\ell)}_1 = \bX^{(\ell)} \bO^{(\ell)}$. The strategy of our proof is to show that, as $N \to +\infty$, Problem \eqref{eq:max_t} has the same solutions as Problem \eqref{eq:max_p}, which are known.

With the decomposition $\bscrT = \bscrP + \frac{1}{\sqrt{N}} \bscrN$, we have,
\begin{multline*}
\Norm{\bU^{(\ell) \top} \bT^{(\ell)} \bigkron_{\ell' \neq \ell} \bU^{(\ell')}_0}_\rmF^2 = \Norm{\bU^{(\ell) \top} \bP^{(\ell)} \bigkron_{\ell' \neq \ell} \bU^{(\ell')}_0}_\rmF^2 + \frac{1}{N} \Norm{\bU^{(\ell) \top} \bN^{(\ell)} \bigkron_{\ell' \neq \ell} \bU^{(\ell')}_0}_\rmF^2 \\
+ \frac{2}{\sqrt{N}} \Scal{\bU^{(\ell) \top} \bP^{(\ell)} \bigkron_{\ell' \neq \ell} \bU^{(\ell')}_0}{\bU^{(\ell) \top} \bN^{(\ell)} \bigkron_{\ell' \neq \ell} \bU^{(\ell')}_0}_\rmF.
\end{multline*}
From Lemma \ref{lem:bound}, $\frac{1}{N} \norm{\bU^{(\ell) \top} \bN^{(\ell)} \bigkron_{\ell' \neq \ell} \bU^{(\ell')}_0}_\rmF^2 = \bigO(1)$ almost surely and
\begin{multline*}
\frac{2}{\sqrt{N}} \left\lvert \Scal{\bU^{(\ell) \top} \bP^{(\ell)} \bigkron_{\ell' \neq \ell} \bU^{(\ell')}_0}{\bU^{(\ell) \top} \bN^{(\ell)} \bigkron_{\ell' \neq \ell} \bU^{(\ell')}_0}_\rmF \right\rvert \\
\leqslant \frac{2}{\sqrt{N}} \underbrace{\Norm{\bU^{(\ell) \top} \bP^{(\ell)} \bigkron_{\ell' \neq \ell} \bU^{(\ell')}_0}_\rmF}_{= \bigO(\Norm{\bscrP}_\rmF)} \underbrace{\Norm{\bU^{(\ell) \top} \bN^{(\ell)} \bigkron_{\ell' \neq \ell} \bU^{(\ell')}_0}_\rmF}_{= \bigO(\sqrt{N})}.
\end{multline*}
Therefore, for all $\bU^{(\ell)} \in V_{r_\ell}(\bbR^{n_\ell})$,
\begin{equation} \label{eq:approx_t}
\Norm{\bU^{(\ell) \top} \bT^{(\ell)} \bigkron_{\ell' \neq \ell} \bU^{(\ell')}_0}_\rmF^2 = \Norm{\bU^{(\ell) \top} \bP^{(\ell)} \bigkron_{\ell' \neq \ell} \bU^{(\ell')}_0}_\rmF^2 + \bigO(\Norm{\bscrP}_\rmF) \qquad \text{almost surely.}
\end{equation}
In particular,
\begin{gather*}
\Norm{\bU^{(\ell) \top}_1 \bT^{(\ell)} \bigkron_{\ell' \neq \ell} \bU^{(\ell')}_0}_\rmF^2 = \Norm{\bU^{(\ell) \top}_1 \bP^{(\ell)} \bigkron_{\ell' \neq \ell} \bU^{(\ell')}_0}_\rmF^2 + \bigO(\Norm{\bscrP}_\rmF) \qquad \text{almost surely,} \\
\Norm{\bU^{(\ell) \top}_1 \bT^{(\ell)} \bigkron_{\ell' \neq \ell} \bU^{(\ell')}_0}_\rmF^2 = \Norm{\widetilde{\bU}^{(\ell) \top}_1 \bP^{(\ell)} \bigkron_{\ell' \neq \ell} \bU^{(\ell')}_0}_\rmF^2 + \bigO(\Norm{\bscrP}_\rmF) \qquad \text{almost surely,}
\end{gather*}
where the first equation is simply Equation \eqref{eq:approx_t} with $\bU^{(\ell)} = \bU^{(\ell)}_1$ and the second equation stems from the maximum over $\bU^{(\ell)} \in V_{r_\ell}(\bbR^{n_\ell})$ of both sides of Equation \eqref{eq:approx_t}\footnote{$\max \{ \norm{\bU^{(\ell) \top} \bT^{(\ell)} \bigkron_{\ell' \neq \ell} \bU^{(\ell')}_0}_\rmF^2 \} \leqslant \max \{ \norm{\bU^{(\ell) \top} \bP^{(\ell)} \bigkron_{\ell' \neq \ell} \bU^{(\ell')}_0}_\rmF^2 \} + \max \{ \bigO(\norm{\bscrP}_\rmF) \}$ where each $\max$ is over $\bU^{(\ell)} \in V_{r_\ell}(\bbR^{n_\ell})$. Thus, $\norm{\bU^{(\ell) \top}_1 \bT^{(\ell)} \bigkron_{\ell' \neq \ell} \bU^{(\ell')}_0}_\rmF^2 - \norm{\widetilde{\bU}^{(\ell) \top}_1 \bP^{(\ell)} \bigkron_{\ell' \neq \ell} \bU^{(\ell')}_0}_\rmF^2 = \bigO(\norm{\bscrP}_\rmF)$ almost surely.}. Hence,
\begin{equation} \label{eq:approx_p}
\Norm{\bU^{(\ell) \top}_1 \bP^{(\ell)} \bigkron_{\ell' \neq \ell} \bU^{(\ell')}_0}_\rmF^2 = \Norm{\widetilde{\bU}^{(\ell) \top}_1 \bP^{(\ell)} \bigkron_{\ell' \neq \ell} \bU^{(\ell')}_0}_\rmF^2 + \bigO(\Norm{\bscrP}_\rmF) \qquad \text{almost surely.}
\end{equation}
Then, consider the singular value decomposition $\bX^{(\ell) \top} \bU^{(\ell)}_1 = \sum_{q_\ell = 1}^{r_\ell} s^{(\ell)}_{q_\ell} \bv^{(\ell)}_{q_\ell} \bw^{(\ell) \top}_{q_\ell}$.
\begin{align*}
&\Norm{\bU^{(\ell) \top}_1 \bP^{(\ell)} \bigkron_{\ell' \neq \ell} \bU^{(\ell')}_0}_\rmF^2 = \Norm{\bU^{(\ell) \top}_1 \bX^{(\ell)} \bH^{(\ell)} \bigkron_{\ell' \neq \ell} \bX^{(\ell') \top} \bU^{(\ell')}_0}_\rmF^2 \\
&= \Norm{\sum_{q_\ell = 1}^{r_\ell} s^{(\ell)}_{q_\ell} \bw^{(\ell)}_{q_\ell} \bv^{(\ell) \top}_{q_\ell} \bH^{(\ell)} \bigkron_{\ell' \neq \ell} \bX^{(\ell') \top} \bU^{(\ell')}_0}_\rmF^2 \\
&= \sum_{q_\ell = 1}^{r_\ell} s^{(\ell) 2}_{q_\ell} \Norm{\bv^{(\ell) \top}_{q_\ell} \bH^{(\ell)} \bigkron_{\ell' \neq \ell} \bX^{(\ell') \top} \bU^{(\ell')}_0}_\rmF^2 \\
&= \sum_{q_\ell = 1}^{r_\ell} \Norm{\bv^{(\ell) \top}_{q_\ell} \bH^{(\ell)} \bigkron_{\ell' \neq \ell} \bX^{(\ell') \top} \bU^{(\ell')}_0}_\rmF^2 - \sum_{q_\ell = 1}^{r_\ell} \left( 1 - s^{(\ell) 2}_{q_\ell} \right) \Norm{\bv^{(\ell) \top}_{q_\ell} \bH^{(\ell)} \bigkron_{\ell' \neq \ell} \bX^{(\ell') \top} \bU^{(\ell')}_0}_\rmF^2 \\
&= \Norm{\bH^{(\ell)} \bigkron_{\ell' \neq \ell} \bX^{(\ell') \top} \bU^{(\ell')}_0}_\rmF^2 - \sum_{q_\ell = 1}^{r_\ell} \left( 1 - s^{(\ell) 2}_{q_\ell} \right) \Norm{\bv^{(\ell) \top}_{q_\ell} \bH^{(\ell)} \bigkron_{\ell' \neq \ell} \bX^{(\ell') \top} \bU^{(\ell')}_0}_\rmF^2.
\end{align*}
Therefore, because $\norm{\bH^{(\ell)} \bigkron_{\ell' \neq \ell} \bX^{(\ell') \top} \bU^{(\ell')}_0}_\rmF^2 = \norm{\widetilde{\bU}^{(\ell) \top}_1 \bP^{(\ell)} \bigkron_{\ell' \neq \ell} \bU^{(\ell')}_0}_\rmF^2$, Equation \eqref{eq:approx_p} yields,
\[
\sum_{q_\ell = 1}^{r_\ell} \left( 1 - s^{(\ell) 2}_{q_\ell} \right) \Norm{\bv^{(\ell) \top}_{q_\ell} \bH^{(\ell)} \bigkron_{\ell' \neq \ell} \bX^{(\ell') \top} \bU^{(\ell')}_0}_\rmF^2 = \bigO(\Norm{\bscrP}_\rmF) \qquad \text{almost surely.}
\]
Using the decomposition $\bv^{(\ell)}_{q_\ell} = \sum_{q_\ell' = 1}^{r_\ell} [\bv^{(\ell)}_{q_\ell}]_{q_\ell'} \bX^{(\ell) \top} \bx^{(\ell)}_{q_\ell'}$, we can see that
\[
\Norm{\bv^{(\ell) \top}_{q_\ell} \bH^{(\ell)} \bigkron_{\ell' \neq \ell} \bX^{(\ell') \top} \bU^{(\ell')}_0}_\rmF^2 = \sum_{q_\ell' = 1}^{r_\ell} [\bv^{(\ell)}_{q_\ell}]_{q_\ell'}^2 \Norm{\bscrP(\bU^{(1)}_0, \ldots, \bx^{(\ell)}_{q_\ell'}, \ldots, \bU^{(d)}_0)}_\rmF^2 = \bigTh(L_N^2).
\]
Hence,
\[
\sum_{q_\ell = 1}^{r_\ell} \left( 1 - s^{(\ell) 2}_{q_\ell} \right) = \bigO \left( \frac{\Norm{\bscrP}_\rmF}{L_N^2} \right) \qquad \text{almost surely,}
\]
which is the result stated in Theorem \ref{thm:hooi}:
\[
\frac{1}{r_\ell} \Norm{\bX^{(\ell) \top} \bU^{(\ell)}_1}_\rmF^2 = \frac{1}{r_\ell} \sum_{q_\ell = 1}^{r_\ell} s^{(\ell) 2}_{q_\ell} = 1 + \bigO \left( \frac{\Norm{\bscrP}_\rmF}{L_N^2} \right) \qquad \text{almost surely.}
\]

\bibliography{bibliography}

\begin{thebibliography}{109}
\providecommand{\natexlab}[1]{#1}
\providecommand{\url}[1]{\texttt{#1}}
\expandafter\ifx\csname urlstyle\endcsname\relax
  \providecommand{\doi}[1]{doi: #1}\else
  \providecommand{\doi}{doi: \begingroup \urlstyle{rm}\Url}\fi

\bibitem[Abadir and Magnus(2005)]{abadir_matrix_2005}
Karim~M. Abadir and Jan~R. Magnus.
\newblock \emph{Matrix {Algebra}}.
\newblock Econometric {Exercises}. Cambridge University Press, Cambridge, 2005.
\newblock ISBN 978-0-521-53746-9.
\newblock \doi{10.1017/CBO9780511810800}.
\newblock URL
  \url{https://www.cambridge.org/core/books/matrix-algebra/BCE8FD2D62006D4061F88E02615B5622}.

\bibitem[Absil et~al.(2009)Absil, Mahony, and
  Sepulchre]{absil_optimization_2009}
Pierre-Antoine Absil, Robert Mahony, and Rodolphe Sepulchre.
\newblock \emph{Optimization {Algorithms} on {Matrix} {Manifolds}}.
\newblock Princeton University Press, April 2009.
\newblock ISBN 978-1-4008-3024-4.
\newblock \doi{10.1515/9781400830244}.
\newblock URL
  \url{https://www.degruyter.com/document/doi/10.1515/9781400830244/html}.
\newblock Publication Title: Optimization Algorithms on Matrix Manifolds.

\bibitem[Acar et~al.(2007)Acar, Aykut-Bingol, Bingol, Bro, and
  Yener]{acar_multiway_2007}
Evrim Acar, Canan Aykut-Bingol, Haluk Bingol, Rasmus Bro, and Bülent Yener.
\newblock Multiway analysis of epilepsy tensors.
\newblock \emph{Bioinformatics (Oxford, England)}, 23\penalty0 (13):\penalty0
  i10--18, July 2007.
\newblock ISSN 1367-4811.
\newblock \doi{10.1093/bioinformatics/btm210}.

\bibitem[Anandkumar et~al.(2013)Anandkumar, Ge, Hsu, and
  Kakade]{anandkumar_tensor_2013}
Animashree Anandkumar, Rong Ge, Daniel Hsu, and Sham Kakade.
\newblock A {Tensor} {Spectral} {Approach} to {Learning} {Mixed} {Membership}
  {Community} {Models}.
\newblock In \emph{Proceedings of the 26th {Annual} {Conference} on {Learning}
  {Theory}}, pages 867--881. PMLR, June 2013.
\newblock URL \url{https://proceedings.mlr.press/v30/Anandkumar13.html}.
\newblock ISSN: 1938-7228.

\bibitem[Anandkumar et~al.(2014)Anandkumar, Ge, Hsu, Kakade, and
  Telgarsky]{anandkumar_tensor_2014}
Animashree Anandkumar, Rong Ge, Daniel Hsu, Sham~M. Kakade, and Matus
  Telgarsky.
\newblock Tensor {Decompositions} for {Learning} {Latent} {Variable} {Models}.
\newblock \emph{Journal of Machine Learning Research}, 15\penalty0
  (80):\penalty0 2773--2832, 2014.
\newblock ISSN 1533-7928.
\newblock URL \url{http://jmlr.org/papers/v15/anandkumar14b.html}.

\bibitem[Bai and Silverstein(2010)]{bai_spectral_2010}
Zhidong Bai and Jack~W. Silverstein.
\newblock \emph{Spectral analysis of large dimensional random matrices},
  volume~20.
\newblock Springer, 2010.

\bibitem[Bai and Yin(1988)]{bai_convergence_1988}
Zhidong Bai and Yanqing Yin.
\newblock Convergence to the {Semicircle} {Law}.
\newblock \emph{The Annals of Probability}, 16\penalty0 (2):\penalty0 863--875,
  April 1988.
\newblock ISSN 0091-1798, 2168-894X.
\newblock \doi{10.1214/aop/1176991792}.
\newblock URL
  \url{https://projecteuclid.org/journals/annals-of-probability/volume-16/issue-2/Convergence-to-the-Semicircle-Law/10.1214/aop/1176991792.full}.
\newblock Publisher: Institute of Mathematical Statistics.

\bibitem[Baik et~al.(2005)Baik, Ben~Arous, and Péché]{baik_phase_2005}
Jinho Baik, Gérard Ben~Arous, and Sandrine Péché.
\newblock Phase transition of the largest eigenvalue for nonnull complex sample
  covariance matrices.
\newblock \emph{The Annals of Probability}, 33\penalty0 (5):\penalty0
  1643--1697, September 2005.
\newblock ISSN 0091-1798, 2168-894X.
\newblock \doi{10.1214/009117905000000233}.
\newblock URL
  \url{https://projecteuclid.org/journals/annals-of-probability/volume-33/issue-5/Phase-transition-of-the-largest-eigenvalue-for-nonnull-complex-sample/10.1214/009117905000000233.full}.
\newblock Publisher: Institute of Mathematical Statistics.

\bibitem[Bandeira et~al.(2018)Bandeira, Perry, and Wein]{bandeira_notes_2018}
Afonso~S. Bandeira, Amelia Perry, and Alexander~S. Wein.
\newblock Notes on computational-to-statistical gaps: predictions using
  statistical physics.
\newblock \emph{Portugaliae Mathematica}, 75\penalty0 (2):\penalty0 159--186,
  December 2018.
\newblock ISSN 0032-5155.
\newblock \doi{10.4171/pm/2014}.
\newblock URL \url{https://ems.press/journals/pm/articles/15908}.

\bibitem[Banna et~al.(2015)Banna, Merlevède, and
  Peligrad]{banna_limiting_2015}
Marwa Banna, Florence Merlevède, and Magda Peligrad.
\newblock On the limiting spectral distribution for a large class of symmetric
  random matrices with correlated entries.
\newblock \emph{Stochastic Processes and their Applications}, 125\penalty0
  (7):\penalty0 2700--2726, July 2015.
\newblock ISSN 0304-4149.
\newblock \doi{10.1016/j.spa.2015.01.010}.
\newblock URL
  \url{https://www.sciencedirect.com/science/article/pii/S0304414915000290}.

\bibitem[Ben~Arous et~al.(2019)Ben~Arous, Mei, Montanari, and
  Nica]{ben_arous_landscape_2019}
Gérard Ben~Arous, Song Mei, Andrea Montanari, and Mihai Nica.
\newblock The {Landscape} of the {Spiked} {Tensor} {Model}.
\newblock \emph{Communications on Pure and Applied Mathematics}, 72\penalty0
  (11):\penalty0 2282--2330, 2019.
\newblock ISSN 1097-0312.
\newblock \doi{10.1002/cpa.21861}.
\newblock URL \url{https://onlinelibrary.wiley.com/doi/abs/10.1002/cpa.21861}.

\bibitem[Ben~Arous et~al.(2020)Ben~Arous, Gheissari, and
  Jagannath]{ben_arous_algorithmic_2020}
Gérard Ben~Arous, Reza Gheissari, and Aukosh Jagannath.
\newblock Algorithmic thresholds for tensor {PCA}.
\newblock \emph{The Annals of Probability}, 48\penalty0 (4):\penalty0
  2052--2087, July 2020.
\newblock ISSN 0091-1798, 2168-894X.
\newblock \doi{10.1214/19-AOP1415}.
\newblock URL
  \url{https://projecteuclid.org/journals/annals-of-probability/volume-48/issue-4/Algorithmic-thresholds-for-tensor-PCA/10.1214/19-AOP1415.full}.
\newblock Publisher: Institute of Mathematical Statistics.

\bibitem[Ben~Arous et~al.(2023)Ben~Arous, Huang, and
  Huang]{ben_arous_long_2023}
Gérard Ben~Arous, Daniel~Zhengyu Huang, and Jiaoyang Huang.
\newblock Long random matrices and tensor unfolding.
\newblock \emph{The Annals of Applied Probability}, 33\penalty0 (6B):\penalty0
  5753--5780, December 2023.
\newblock ISSN 1050-5164, 2168-8737.
\newblock \doi{10.1214/23-AAP1958}.
\newblock URL
  \url{https://projecteuclid.org/journals/annals-of-applied-probability/volume-33/issue-6B/Long-random-matrices-and-tensor-unfolding/10.1214/23-AAP1958.full}.
\newblock Publisher: Institute of Mathematical Statistics.

\bibitem[Benaych-Georges and
  Nadakuditi(2011)]{benaych-georges_eigenvalues_2011}
Florent Benaych-Georges and Raj~Rao Nadakuditi.
\newblock The eigenvalues and eigenvectors of finite, low rank perturbations of
  large random matrices.
\newblock \emph{Advances in Mathematics}, 227\penalty0 (1):\penalty0 494--521,
  2011.
\newblock ISSN 0001-8708.
\newblock \doi{10.1016/j.aim.2011.02.007}.
\newblock URL
  \url{https://www.sciencedirect.com/science/article/pii/S0001870811000570}.

\bibitem[Benaych-Georges and Nadakuditi(2012)]{benaych-georges_singular_2012}
Florent Benaych-Georges and Raj~Rao Nadakuditi.
\newblock The singular values and vectors of low rank perturbations of large
  rectangular random matrices.
\newblock \emph{Journal of Multivariate Analysis}, 111:\penalty0 120--135,
  October 2012.
\newblock ISSN 0047-259X.
\newblock \doi{10.1016/j.jmva.2012.04.019}.
\newblock URL
  \url{https://www.sciencedirect.com/science/article/pii/S0047259X12001108}.

\bibitem[Bi et~al.(2021)Bi, Tang, Yuan, Zhang, and Qu]{bi_tensors_2021}
Xuan Bi, Xiwei Tang, Yubai Yuan, Yanqing Zhang, and Annie Qu.
\newblock Tensors in {Statistics}.
\newblock \emph{Annual Review of Statistics and Its Application}, 8\penalty0
  (Volume 8, 2021):\penalty0 345--368, March 2021.
\newblock ISSN 2326-8298, 2326-831X.
\newblock \doi{10.1146/annurev-statistics-042720-020816}.
\newblock URL
  \url{https://www.annualreviews.org/content/journals/10.1146/annurev-statistics-042720-020816}.
\newblock Publisher: Annual Reviews.

\bibitem[Billingsley(2012)]{billingsley_probability_2012}
Patrick Billingsley.
\newblock \emph{Probability and {Measure}}.
\newblock Wiley {Series} in {Probability} and {Statistics}. John Wiley \& Sons,
  Inc., 2012.

\bibitem[Björck and Golub(1973)]{bjorck_numerical_1973}
Åke Björck and Gene~H. Golub.
\newblock Numerical methods for computing angles between linear subspaces.
\newblock \emph{Mathematics of Computation}, 27\penalty0 (123):\penalty0
  579--594, 1973.
\newblock ISSN 0025-5718, 1088-6842.
\newblock \doi{10.1090/S0025-5718-1973-0348991-3}.
\newblock URL
  \url{https://www.ams.org/mcom/1973-27-123/S0025-5718-1973-0348991-3/}.

\bibitem[Bro and Andersson(1998)]{bro_improving_1998}
Rasmus Bro and Claus~A. Andersson.
\newblock Improving the speed of multiway algorithms: {Part} {II}:
  {Compression}.
\newblock \emph{Chemometrics and Intelligent Laboratory Systems}, 42\penalty0
  (1):\penalty0 105--113, August 1998.
\newblock ISSN 0169-7439.
\newblock \doi{10.1016/S0169-7439(98)00011-2}.
\newblock URL
  \url{https://www.sciencedirect.com/science/article/pii/S0169743998000112}.

\bibitem[Capitaine et~al.(2009)Capitaine, Donati-Martin, and
  Féral]{capitaine_largest_2009}
Mireille Capitaine, Catherine Donati-Martin, and Delphine Féral.
\newblock The largest eigenvalues of finite rank deformation of large {Wigner}
  matrices: convergence and nonuniversality of the fluctuations.
\newblock \emph{The Annals of Probability}, 37\penalty0 (1), January 2009.
\newblock ISSN 0091-1798.
\newblock \doi{10.1214/08-AOP394}.
\newblock URL \url{http://arxiv.org/abs/0706.0136}.
\newblock arXiv:0706.0136 [math].

\bibitem[Chen et~al.(2022)Chen, Mourrat, and Xia]{chen_statistical_2022}
Hongbin Chen, Jean-Christophe Mourrat, and Jiaming Xia.
\newblock Statistical inference of finite-rank tensors.
\newblock \emph{Annales Henri Lebesgue}, 5:\penalty0 1161--1189, 2022.
\newblock ISSN 2644-9463.
\newblock \doi{10.5802/ahl.146}.
\newblock URL \url{https://ahl.centre-mersenne.org/articles/10.5802/ahl.146/}.

\bibitem[Chen et~al.(2021)Chen, Handschy, and Lerman]{chen_phase_2021}
Wei-Kuo Chen, Madeline Handschy, and Gilad Lerman.
\newblock Phase transition in random tensors with multiple independent spikes.
\newblock \emph{The Annals of Applied Probability}, 31\penalty0 (4):\penalty0
  1868--1913, August 2021.
\newblock ISSN 1050-5164, 2168-8737.
\newblock \doi{10.1214/20-AAP1636}.
\newblock URL
  \url{https://projecteuclid.org/journals/annals-of-applied-probability/volume-31/issue-4/Phase-transition-in-random-tensors-with-multiple-independent-spikes/10.1214/20-AAP1636.full}.
\newblock Publisher: Institute of Mathematical Statistics.

\bibitem[Chevreuil and Loubaton(2018)]{chevreuil_non-detectability_2018}
Antoine Chevreuil and Philippe Loubaton.
\newblock On {The} {Non}-{Detectability} of {Spiked} {Large} {Random}
  {Tensors}.
\newblock In \emph{2018 {IEEE} {Statistical} {Signal} {Processing} {Workshop}
  ({SSP})}, pages 443--447, June 2018.
\newblock \doi{10.1109/SSP.2018.8450752}.
\newblock URL \url{https://ieeexplore.ieee.org/document/8450752}.

\bibitem[Chikuse(2003)]{chikuse_statistics_2003}
Yasuko Chikuse.
\newblock \emph{Statistics on {Special} {Manifolds}}, volume 174 of
  \emph{Lecture {Notes} in {Statistics}}.
\newblock Springer, New York, NY, 2003.
\newblock ISBN 978-0-387-00160-9 978-0-387-21540-2.
\newblock \doi{10.1007/978-0-387-21540-2}.
\newblock URL \url{http://link.springer.com/10.1007/978-0-387-21540-2}.

\bibitem[Cichocki et~al.(2015)Cichocki, Mandic, Phan, Caiafa, Zhou, Zhao, and
  De~Lathauwer]{cichocki_tensor_2015}
Andrzej Cichocki, Danilo~P. Mandic, Anh~Huy Phan, Cesar~F. Caiafa, Guoxu Zhou,
  Qibin Zhao, and Lieven De~Lathauwer.
\newblock Tensor {Decompositions} for {Signal} {Processing} {Applications}:
  {From} two-way to multiway component analysis.
\newblock \emph{IEEE Signal Processing Magazine}, 32\penalty0 (2):\penalty0
  145--163, March 2015.
\newblock ISSN 1558-0792.
\newblock \doi{10.1109/MSP.2013.2297439}.
\newblock URL \url{https://ieeexplore.ieee.org/document/7038247}.
\newblock Conference Name: IEEE Signal Processing Magazine.

\bibitem[Comon(2009)]{comon_tensors_2009}
Pierre Comon.
\newblock Tensors versus {Matrices}, usefulness and unexpected properties.
\newblock In {IEEE}, editor, \emph{{IEEE} {Workshop} on {Statistical} {Signal}
  {Processing}}, pages 780--788, Cardiff, United Kingdom, September 2009. IEEE.
\newblock URL \url{https://hal.archives-ouvertes.fr/hal-00417258}.

\bibitem[Comon(2014)]{comon_tensors_2014}
Pierre Comon.
\newblock Tensors: a {Brief} {Introduction}.
\newblock \emph{IEEE Signal Processing Magazine}, 31\penalty0 (3):\penalty0
  44--53, May 2014.
\newblock \doi{10.1109/MSP.2014.2298533}.
\newblock URL \url{https://hal.archives-ouvertes.fr/hal-00923279}.
\newblock Publisher: Institute of Electrical and Electronics Engineers.

\bibitem[Couillet and Liao(2022)]{couillet_random_2022}
Romain Couillet and Zhenyu Liao.
\newblock \emph{Random {Matrix} {Methods} for {Machine} {Learning}}.
\newblock Cambridge University Press, Cambridge, 2022.
\newblock ISBN 978-1-00-912323-5.
\newblock \doi{10.1017/9781009128490}.
\newblock URL
  \url{https://www.cambridge.org/core/books/random-matrix-methods-for-machine-learning/6B681EB69E58B5F888EDB689C160C682}.

\bibitem[De~Lathauwer et~al.(2000{\natexlab{a}})De~Lathauwer, De~Moor, and
  Vandewalle]{de_lathauwer_best_2000}
Lieven De~Lathauwer, Bart De~Moor, and Joos Vandewalle.
\newblock On the {Best} {Rank}-1 and {Rank}-({R1}, {R2}, ..., {RN})
  {Approximation} of {Higher}-{Order} {Tensors}.
\newblock \emph{SIAM Journal on Matrix Analysis and Applications}, 21\penalty0
  (4):\penalty0 1324--1342, January 2000{\natexlab{a}}.
\newblock ISSN 0895-4798.
\newblock \doi{10.1137/S0895479898346995}.
\newblock URL \url{https://epubs.siam.org/doi/10.1137/S0895479898346995}.
\newblock Publisher: Society for Industrial and Applied Mathematics.

\bibitem[De~Lathauwer et~al.(2000{\natexlab{b}})De~Lathauwer, De~Moor, and
  Vandewalle]{de_lathauwer_multilinear_2000}
Lieven De~Lathauwer, Bart De~Moor, and Joos Vandewalle.
\newblock A {Multilinear} {Singular} {Value} {Decomposition}.
\newblock \emph{SIAM Journal on Matrix Analysis and Applications}, 21\penalty0
  (4):\penalty0 1253--1278, January 2000{\natexlab{b}}.
\newblock ISSN 0895-4798.
\newblock \doi{10.1137/S0895479896305696}.
\newblock URL \url{https://epubs.siam.org/doi/10.1137/S0895479896305696}.
\newblock Publisher: Society for Industrial and Applied Mathematics.

\bibitem[Eckart and Young(1936)]{eckart_approximation_1936}
Carl Eckart and Gale Young.
\newblock The approximation of one matrix by another of lower rank.
\newblock \emph{Psychometrika}, 1\penalty0 (3):\penalty0 211--218, September
  1936.
\newblock ISSN 1860-0980.
\newblock \doi{10.1007/BF02288367}.
\newblock URL \url{https://doi.org/10.1007/BF02288367}.

\bibitem[Edwards and Jones(1976)]{edwards_eigenvalue_1976}
Samuel~F. Edwards and Raymund~C. Jones.
\newblock The eigenvalue spectrum of a large symmetric random matrix.
\newblock \emph{Journal of Physics A: Mathematical and General}, 9\penalty0
  (10):\penalty0 1595, October 1976.
\newblock ISSN 0305-4470.
\newblock \doi{10.1088/0305-4470/9/10/011}.
\newblock URL \url{https://dx.doi.org/10.1088/0305-4470/9/10/011}.

\bibitem[Fanaee-T and Gama(2015)]{fanaee-t_eigenevent_2015}
Hadi Fanaee-T and João Gama.
\newblock {EigenEvent}: {An} {Algorithm} for {Event} {Detection} from {Complex}
  {Data} {Streams} in {Syndromic} {Surveillance}.
\newblock \emph{Intelligent Data Analysis}, 19\penalty0 (3):\penalty0 597--616,
  June 2015.
\newblock ISSN 1088467X, 15714128.
\newblock \doi{10.3233/IDA-150734}.
\newblock URL \url{http://arxiv.org/abs/1406.3496}.
\newblock arXiv:1406.3496 [cs, stat].

\bibitem[Feldman(2023)]{feldman_spiked_2023}
Michael~J. Feldman.
\newblock Spiked singular values and vectors under extreme aspect ratios.
\newblock \emph{Journal of Multivariate Analysis}, 196:\penalty0 105187, July
  2023.
\newblock ISSN 0047-259X.
\newblock \doi{10.1016/j.jmva.2023.105187}.
\newblock URL
  \url{https://www.sciencedirect.com/science/article/pii/S0047259X23000337}.

\bibitem[Feldman and Donoho(2023)]{feldman_sharp_2023}
Michael~J. Feldman and David~L. Donoho.
\newblock Sharp {Recovery} {Thresholds} of {Tensor} {PCA} {Spectral}
  {Algorithms}.
\newblock \emph{Advances in Neural Information Processing Systems},
  36:\penalty0 56628--56640, December 2023.
\newblock URL
  \url{https://proceedings.neurips.cc/paper_files/paper/2023/hash/b14d76c7266be21b338527cd25deac45-Abstract-Conference.html}.

\bibitem[Frolov and Oseledets(2017)]{frolov_tensor_2017}
Evgeny Frolov and Ivan Oseledets.
\newblock Tensor methods and recommender systems.
\newblock \emph{WIREs Data Mining and Knowledge Discovery}, 7\penalty0
  (3):\penalty0 e1201, 2017.
\newblock ISSN 1942-4795.
\newblock \doi{10.1002/widm.1201}.
\newblock URL \url{https://onlinelibrary.wiley.com/doi/abs/10.1002/widm.1201}.
\newblock \_eprint: https://onlinelibrary.wiley.com/doi/pdf/10.1002/widm.1201.

\bibitem[Féral and Péché(2007)]{feral_largest_2007}
Delphine Féral and Sandrine Péché.
\newblock The {Largest} {Eigenvalue} of {Rank} {One} {Deformation} of {Large}
  {Wigner} {Matrices}.
\newblock \emph{Communications in Mathematical Physics}, 272\penalty0
  (1):\penalty0 185--228, May 2007.
\newblock ISSN 1432-0916.
\newblock \doi{10.1007/s00220-007-0209-3}.
\newblock URL \url{https://doi.org/10.1007/s00220-007-0209-3}.

\bibitem[Füredi and Komlós(1981)]{furedi_eigenvalues_1981}
Zoltán Füredi and János Komlós.
\newblock The eigenvalues of random symmetric matrices.
\newblock \emph{Combinatorica}, 1\penalty0 (3):\penalty0 233--241, September
  1981.
\newblock ISSN 1439-6912.
\newblock \doi{10.1007/BF02579329}.
\newblock URL \url{https://doi.org/10.1007/BF02579329}.

\bibitem[Geronimo and Hill(2003)]{geronimo_necessary_2003}
Jeffrey~S. Geronimo and Theodore~P. Hill.
\newblock Necessary and sufficient condition that the limit of {Stieltjes}
  transforms is a {Stieltjes} transform.
\newblock \emph{Journal of Approximation Theory}, 121\penalty0 (1):\penalty0
  54--60, March 2003.
\newblock ISSN 0021-9045.
\newblock \doi{10.1016/S0021-9045(02)00042-4}.
\newblock URL
  \url{https://www.sciencedirect.com/science/article/pii/S0021904502000424}.

\bibitem[Goulart et~al.(2022)Goulart, Couillet, and Comon]{goulart_random_2022}
José Henrique de~M. Goulart, Romain Couillet, and Pierre Comon.
\newblock A {Random} {Matrix} {Perspective} on {Random} {Tensors}.
\newblock \emph{Journal of Machine Learning Research}, 23\penalty0
  (264):\penalty0 1--36, 2022.
\newblock ISSN 1533-7928.
\newblock URL \url{http://jmlr.org/papers/v23/21-1038.html}.

\bibitem[Grasedyck et~al.(2013)Grasedyck, Kressner, and
  Tobler]{grasedyck_literature_2013}
Lars Grasedyck, Daniel Kressner, and Christine Tobler.
\newblock A literature survey of low-rank tensor approximation techniques.
\newblock \emph{GAMM-Mitteilungen}, 36\penalty0 (1):\penalty0 53--78, 2013.
\newblock ISSN 1522-2608.
\newblock \doi{10.1002/gamm.201310004}.
\newblock URL
  \url{https://onlinelibrary.wiley.com/doi/abs/10.1002/gamm.201310004}.
\newblock \_eprint:
  https://onlinelibrary.wiley.com/doi/pdf/10.1002/gamm.201310004.

\bibitem[Gurau(2014)]{gurau_universality_2014}
Razvan Gurau.
\newblock Universality for random tensors.
\newblock \emph{Annales de l'Institut Henri Poincaré, Probabilités et
  Statistiques}, 50\penalty0 (4):\penalty0 1474--1525, November 2014.
\newblock ISSN 0246-0203.
\newblock \doi{10.1214/13-AIHP567}.
\newblock URL
  \url{https://projecteuclid.org/journals/annales-de-linstitut-henri-poincare-probabilites-et-statistiques/volume-50/issue-4/Universality-for-random-tensors/10.1214/13-AIHP567.full}.
\newblock Publisher: Institut Henri Poincaré.

\bibitem[Hackbusch(2012)]{hackbusch_tensor_2012}
Wolfgang Hackbusch.
\newblock \emph{Tensor {Spaces} and {Numerical} {Tensor} {Calculus}}.
\newblock Springer {Series} in {Computational} {Mathematics}. Springer, 2012.
\newblock URL \url{https://link.springer.com/book/10.1007/978-3-642-28027-6}.

\bibitem[Hillar and Lim(2013)]{hillar_most_2013}
Christopher~J. Hillar and Lek-Heng Lim.
\newblock Most {Tensor} {Problems} {Are} {NP}-{Hard}.
\newblock \emph{Journal of the ACM}, 60\penalty0 (6):\penalty0 45:1--45:39,
  November 2013.
\newblock ISSN 0004-5411.
\newblock \doi{10.1145/2512329}.
\newblock URL \url{https://doi.org/10.1145/2512329}.

\bibitem[Hinrichs et~al.(2017)Hinrichs, Prochno, and
  Vybíral]{hinrichs_entropy_2017}
Aicke Hinrichs, Joscha Prochno, and Jan Vybíral.
\newblock Entropy numbers of embeddings of {Schatten} classes.
\newblock \emph{Journal of Functional Analysis}, 273\penalty0 (10):\penalty0
  3241--3261, November 2017.
\newblock ISSN 0022-1236.
\newblock \doi{10.1016/j.jfa.2017.08.008}.
\newblock URL
  \url{https://www.sciencedirect.com/science/article/pii/S0022123617303221}.

\bibitem[Hitchcock(1927)]{hitchcock_expression_1927}
Frank~L. Hitchcock.
\newblock The {Expression} of a {Tensor} or a {Polyadic} as a {Sum} of
  {Products}.
\newblock \emph{Journal of Mathematics and Physics}, 6\penalty0 (1-4):\penalty0
  164--189, 1927.
\newblock ISSN 1467-9590.
\newblock \doi{10.1002/sapm192761164}.
\newblock URL
  \url{https://onlinelibrary.wiley.com/doi/abs/10.1002/sapm192761164}.
\newblock \_eprint:
  https://onlinelibrary.wiley.com/doi/pdf/10.1002/sapm192761164.

\bibitem[Hopkins et~al.(2015)Hopkins, Shi, and Steurer]{hopkins_tensor_2015}
Samuel~B. Hopkins, Jonathan Shi, and David Steurer.
\newblock Tensor principal component analysis via sum-of-square proofs.
\newblock In \emph{Proceedings of the 28th {Conference} on {Learning}
  {Theory}}, pages 956--1006. PMLR, June 2015.
\newblock URL \url{https://proceedings.mlr.press/v40/Hopkins15.html}.
\newblock ISSN: 1938-7228.

\bibitem[Hopkins et~al.(2016)Hopkins, Schramm, Shi, and
  Steurer]{hopkins_fast_2016}
Samuel~B. Hopkins, Tselil Schramm, Jonathan Shi, and David Steurer.
\newblock Fast spectral algorithms from sum-of-squares proofs: tensor
  decomposition and planted sparse vectors.
\newblock In \emph{Proceedings of the forty-eighth annual {ACM} symposium on
  {Theory} of {Computing}}, {STOC} '16, pages 178--191, New York, NY, USA,
  2016. Association for Computing Machinery.
\newblock ISBN 978-1-4503-4132-5.
\newblock \doi{10.1145/2897518.2897529}.
\newblock URL \url{https://doi.org/10.1145/2897518.2897529}.

\bibitem[Hopkins et~al.(2017)Hopkins, Kothari, Potechin, Raghavendra, Schramm,
  and Steurer]{hopkins_power_2017}
Samuel~B. Hopkins, Pravesh~K. Kothari, Aaron Potechin, Prasad Raghavendra,
  Tselil Schramm, and David Steurer.
\newblock The {Power} of {Sum}-of-{Squares} for {Detecting} {Hidden}
  {Structures}.
\newblock In \emph{2017 {IEEE} 58th {Annual} {Symposium} on {Foundations} of
  {Computer} {Science} ({FOCS})}, pages 720--731. IEEE Computer Society,
  October 2017.
\newblock ISBN 978-1-5386-3464-6.
\newblock \doi{10.1109/FOCS.2017.72}.
\newblock URL
  \url{https://www.computer.org/csdl/proceedings-article/focs/2017/3464a720/12OmNxFsmCD}.
\newblock ISSN: 0272-5428.

\bibitem[Huang et~al.(2022)Huang, Huang, Yang, and Cheng]{huang_power_2022}
Jiaoyang Huang, Daniel~Z. Huang, Qing Yang, and Guang Cheng.
\newblock Power {Iteration} for {Tensor} {PCA}.
\newblock \emph{Journal of Machine Learning Research}, 23\penalty0
  (128):\penalty0 1--47, 2022.
\newblock ISSN 1533-7928.
\newblock URL \url{http://jmlr.org/papers/v23/21-1290.html}.

\bibitem[Hunyadi et~al.(2017)Hunyadi, Dupont, Van~Paesschen, and
  Van~Huffel]{hunyadi_tensor_2017}
Borbála Hunyadi, Patrick Dupont, Wim Van~Paesschen, and Sabine Van~Huffel.
\newblock Tensor decompositions and data fusion in epileptic
  electroencephalography and functional magnetic resonance imaging data.
\newblock \emph{WIREs Data Mining and Knowledge Discovery}, 7\penalty0
  (1):\penalty0 e1197, 2017.
\newblock ISSN 1942-4795.
\newblock \doi{10.1002/widm.1197}.
\newblock URL \url{https://onlinelibrary.wiley.com/doi/abs/10.1002/widm.1197}.
\newblock \_eprint: https://onlinelibrary.wiley.com/doi/pdf/10.1002/widm.1197.

\bibitem[Jagannath et~al.(2020)Jagannath, Lopatto, and
  Miolane]{jagannath_statistical_2020}
Aukosh Jagannath, Patrick Lopatto, and Léo Miolane.
\newblock Statistical thresholds for tensor {PCA}.
\newblock \emph{The Annals of Applied Probability}, 30\penalty0 (4):\penalty0
  1910--1933, August 2020.
\newblock ISSN 1050-5164, 2168-8737.
\newblock \doi{10.1214/19-AAP1547}.
\newblock URL
  \url{https://projecteuclid.org/journals/annals-of-applied-probability/volume-30/issue-4/Statistical-thresholds-for-tensor-PCA/10.1214/19-AAP1547.full}.
\newblock Publisher: Institute of Mathematical Statistics.

\bibitem[Kadmon and Ganguli(2018)]{kadmon_statistical_2018}
Jonathan Kadmon and Surya Ganguli.
\newblock Statistical mechanics of low-rank tensor decomposition.
\newblock In \emph{Advances in {Neural} {Information} {Processing} {Systems}},
  volume~31. Curran Associates, Inc., 2018.
\newblock URL
  \url{https://papers.nips.cc/paper_files/paper/2018/hash/b3848d61bbbc6207c6668a8a9e2730ed-Abstract.html}.

\bibitem[Kanatsoulis et~al.(2018)Kanatsoulis, Fu, Sidiropoulos, and
  Ma]{kanatsoulis_hyperspectral_2018}
Charilaos~I. Kanatsoulis, Xiao Fu, Nicholas~D. Sidiropoulos, and Wing-Kin Ma.
\newblock Hyperspectral {Super}-{Resolution}: {A} {Coupled} {Tensor}
  {Factorization} {Approach}.
\newblock \emph{IEEE Transactions on Signal Processing}, 66\penalty0
  (24):\penalty0 6503--6517, December 2018.
\newblock ISSN 1941-0476.
\newblock \doi{10.1109/TSP.2018.2876362}.
\newblock URL \url{https://ieeexplore.ieee.org/abstract/document/8494792}.
\newblock Conference Name: IEEE Transactions on Signal Processing.

\bibitem[Kapteyn et~al.(1986)Kapteyn, Neudecker, and
  Wansbeek]{kapteyn_approach_1986}
Arie Kapteyn, Heinz Neudecker, and Tom Wansbeek.
\newblock An approach ton-mode components analysis.
\newblock \emph{Psychometrika}, 51\penalty0 (2):\penalty0 269--275, June 1986.
\newblock ISSN 1860-0980.
\newblock \doi{10.1007/BF02293984}.
\newblock URL \url{https://doi.org/10.1007/BF02293984}.

\bibitem[Karatzoglou et~al.(2010)Karatzoglou, Amatriain, Baltrunas, and
  Oliver]{karatzoglou_multiverse_2010}
Alexandros Karatzoglou, Xavier Amatriain, Linas Baltrunas, and Nuria Oliver.
\newblock Multiverse recommendation: n-dimensional tensor factorization for
  context-aware collaborative filtering.
\newblock In \emph{Proceedings of the fourth {ACM} conference on {Recommender}
  systems}, {RecSys} '10, pages 79--86, New York, NY, USA, September 2010.
  Association for Computing Machinery.
\newblock ISBN 978-1-60558-906-0.
\newblock \doi{10.1145/1864708.1864727}.
\newblock URL \url{https://doi.org/10.1145/1864708.1864727}.

\bibitem[Kim et~al.(2017)Kim, Bandeira, and Goemans]{kim_community_2017}
Chiheon Kim, Afonso~S. Bandeira, and Michel~X. Goemans.
\newblock Community detection in hypergraphs, spiked tensor models, and
  {Sum}-of-{Squares}.
\newblock In \emph{2017 {International} {Conference} on {Sampling} {Theory} and
  {Applications} ({SampTA})}, pages 124--128, July 2017.
\newblock \doi{10.1109/SAMPTA.2017.8024470}.
\newblock URL \url{https://ieeexplore.ieee.org/document/8024470}.

\bibitem[Kolda(2003)]{kolda_counterexample_2003}
Tamara~G. Kolda.
\newblock A {Counterexample} to the {Possibility} of an {Extension} of the
  {Eckart}--{Young} {Low}-{Rank} {Approximation} {Theorem} for the {Orthogonal}
  {Rank} {Tensor} {Decomposition}.
\newblock \emph{SIAM Journal on Matrix Analysis and Applications}, 24\penalty0
  (3):\penalty0 762--767, January 2003.
\newblock ISSN 0895-4798.
\newblock \doi{10.1137/S0895479801394465}.
\newblock URL \url{https://epubs.siam.org/doi/10.1137/S0895479801394465}.
\newblock Publisher: Society for Industrial and Applied Mathematics.

\bibitem[Kolda and Bader(2009)]{kolda_tensor_2009}
Tamara~G. Kolda and Brett~W. Bader.
\newblock Tensor {Decompositions} and {Applications}.
\newblock \emph{SIAM Review}, August 2009.
\newblock \doi{10.1137/07070111X}.
\newblock URL \url{https://epubs.siam.org/doi/10.1137/07070111X}.
\newblock Publisher: Society for Industrial and Applied Mathematics.

\bibitem[Kroonenberg and de~Leeuw(1980)]{kroonenberg_principal_1980}
Pieter~M. Kroonenberg and Jan de~Leeuw.
\newblock Principal component analysis of three-mode data by means of
  alternating least squares algorithms.
\newblock \emph{Psychometrika}, 45\penalty0 (1):\penalty0 69--97, March 1980.
\newblock ISSN 1860-0980.
\newblock \doi{10.1007/BF02293599}.
\newblock URL \url{https://doi.org/10.1007/BF02293599}.

\bibitem[Landsberg(2011)]{landsberg_tensors_2011}
Joseph~M. Landsberg.
\newblock \emph{Tensors: {Geometry} and {Applications}}, volume 128 of
  \emph{Graduate {Studies} in {Mathematics}}.
\newblock American Mathematical Society, December 2011.
\newblock ISBN 978-0-8218-6907-9 978-0-8218-8481-2 978-0-8218-8483-6
  978-1-4704-0923-4.
\newblock \doi{10.1090/gsm/128}.
\newblock URL \url{http://www.ams.org/gsm/128}.
\newblock ISSN: 1065-7339.

\bibitem[Lesieur et~al.(2017)Lesieur, Miolane, Lelarge, Krzakala, and
  Zdeborová]{lesieur_statistical_2017}
Thibault Lesieur, Léo Miolane, Marc Lelarge, Florent Krzakala, and Lenka
  Zdeborová.
\newblock Statistical and computational phase transitions in spiked tensor
  estimation.
\newblock In \emph{2017 {IEEE} {International} {Symposium} on {Information}
  {Theory} ({ISIT})}, pages 511--515, June 2017.
\newblock \doi{10.1109/ISIT.2017.8006580}.
\newblock URL \url{http://arxiv.org/abs/1701.08010}.
\newblock arXiv:1701.08010 [cond-mat, stat].

\bibitem[Li and Li(2010)]{li_tensor_2010}
Nan Li and Baoxin Li.
\newblock Tensor completion for on-board compression of hyperspectral images.
\newblock In \emph{2010 {IEEE} {International} {Conference} on {Image}
  {Processing}}, pages 517--520, September 2010.
\newblock \doi{10.1109/ICIP.2010.5651225}.
\newblock ISSN: 2381-8549.

\bibitem[Liu et~al.(2022)Liu, Yuan, and Zhao]{liu_characterizing_2022}
Tianqi Liu, Ming Yuan, and Hongyu Zhao.
\newblock Characterizing {Spatiotemporal} {Transcriptome} of the {Human}
  {Brain} {Via} {Low}-{Rank} {Tensor} {Decomposition}.
\newblock \emph{Statistics in Biosciences}, 14\penalty0 (3):\penalty0 485--513,
  December 2022.
\newblock ISSN 1867-1772.
\newblock \doi{10.1007/s12561-021-09331-5}.
\newblock URL \url{https://doi.org/10.1007/s12561-021-09331-5}.

\bibitem[Loubaton(2016)]{loubaton_almost_2016}
Philippe Loubaton.
\newblock On the {Almost} {Sure} {Location} of the {Singular} {Values} of
  {Certain} {Gaussian} {Block}-{Hankel} {Large} {Random} {Matrices}.
\newblock \emph{Journal of Theoretical Probability}, 29\penalty0 (4):\penalty0
  1339--1443, December 2016.
\newblock ISSN 1572-9230.
\newblock \doi{10.1007/s10959-015-0614-z}.
\newblock URL \url{https://doi.org/10.1007/s10959-015-0614-z}.

\bibitem[Lytova and Pastur(2009)]{lytova_central_2009}
Anna Lytova and Leonid Pastur.
\newblock Central limit theorem for linear eigenvalue statistics of random
  matrices with independent entries.
\newblock \emph{The Annals of Probability}, 37\penalty0 (5):\penalty0
  1778--1840, 2009.
\newblock Publisher: Institute of Mathematical Statistics.

\bibitem[Mardia and Khatri(1977)]{mardia_uniform_1977}
Kanti~V. Mardia and Chinubhai~G. Khatri.
\newblock Uniform distribution on a {Stiefel} manifold.
\newblock \emph{Journal of Multivariate Analysis}, 7\penalty0 (3):\penalty0
  468--473, September 1977.
\newblock ISSN 0047-259X.
\newblock \doi{10.1016/0047-259X(77)90087-2}.
\newblock URL
  \url{https://www.sciencedirect.com/science/article/pii/0047259X77900872}.

\bibitem[Marčenko and Pastur(1967)]{marcenko_distribution_1967}
Vladimir~A. Marčenko and Leonid~A. Pastur.
\newblock Distribution of eigenvalues for some sets of random matrices.
\newblock \emph{Mathematics of the USSR-Sbornik}, 72(114)\penalty0
  (4):\penalty0 457--483, 1967.
\newblock Publisher: IOP Publishing.

\bibitem[Merlevède et~al.(2015)Merlevède, Peligrad, and
  Peligrad]{merlevede_universality_2015}
Florence Merlevède, Costel Peligrad, and Magda Peligrad.
\newblock On the universality of spectral limit for random matrices with
  martingale differences entries.
\newblock \emph{Random Matrices: Theory and Applications}, 04\penalty0
  (01):\penalty0 1550003, January 2015.
\newblock ISSN 2010-3263.
\newblock \doi{10.1142/S2010326315500033}.
\newblock URL
  \url{https://www.worldscientific.com/doi/abs/10.1142/S2010326315500033}.
\newblock Publisher: World Scientific Publishing Co.

\bibitem[Mirsky(1960)]{mirsky_symmetric_1960}
Leonid Mirsky.
\newblock Symmetric gauge functions and unitarily invariant norms.
\newblock \emph{The Quarterly Journal of Mathematics}, 11\penalty0
  (1):\penalty0 50--59, January 1960.
\newblock ISSN 0033-5606.
\newblock \doi{10.1093/qmath/11.1.50}.
\newblock URL \url{https://doi.org/10.1093/qmath/11.1.50}.

\bibitem[Montanari and Richard(2014)]{montanari_statistical_2014}
Andrea Montanari and Emile Richard.
\newblock A statistical model for tensor {PCA}.
\newblock In \emph{Advances in {Neural} {Information} {Processing} {Systems}},
  volume~27. Curran Associates, Inc., 2014.
\newblock URL
  \url{https://proceedings.neurips.cc/paper/2014/hash/b5488aeff42889188d03c9895255cecc-Abstract.html}.

\bibitem[Muralidhara et~al.(2011)Muralidhara, Gross, Gutell, and
  Alter]{muralidhara_tensor_2011}
Chaitanya Muralidhara, Andrew~M. Gross, Robin~R. Gutell, and Orly Alter.
\newblock Tensor {Decomposition} {Reveals} {Concurrent} {Evolutionary}
  {Convergences} and {Divergences} and {Correlations} with {Structural}
  {Motifs} in {Ribosomal} {RNA}.
\newblock \emph{PLOS ONE}, 6\penalty0 (4):\penalty0 e18768, 2011.
\newblock ISSN 1932-6203.
\newblock \doi{10.1371/journal.pone.0018768}.
\newblock URL
  \url{https://journals.plos.org/plosone/article?id=10.1371/journal.pone.0018768}.
\newblock Publisher: Public Library of Science.

\bibitem[Omberg et~al.(2007)Omberg, Golub, and Alter]{omberg_tensor_2007}
Larsson Omberg, Gene~H. Golub, and Orly Alter.
\newblock A tensor higher-order singular value decomposition for integrative
  analysis of {DNA} microarray data from different studies.
\newblock \emph{Proceedings of the National Academy of Sciences}, 104\penalty0
  (47):\penalty0 18371--18376, November 2007.
\newblock \doi{10.1073/pnas.0709146104}.
\newblock URL \url{https://www.pnas.org/doi/full/10.1073/pnas.0709146104}.
\newblock Publisher: Proceedings of the National Academy of Sciences.

\bibitem[Omberg et~al.(2009)Omberg, Meyerson, Kobayashi, Drury, Diffley, and
  Alter]{omberg_global_2009}
Larsson Omberg, Joel~R. Meyerson, Kayta Kobayashi, Lucy~S. Drury, John F.~X.
  Diffley, and Orly Alter.
\newblock Global effects of {DNA} replication and {DNA} replication origin
  activity on eukaryotic gene expression.
\newblock \emph{Molecular Systems Biology}, 5\penalty0 (1):\penalty0 312,
  January 2009.
\newblock ISSN 1744-4292.
\newblock \doi{10.1038/msb.2009.70}.
\newblock URL \url{https://www.embopress.org/doi/full/10.1038/msb.2009.70}.
\newblock Publisher: John Wiley \& Sons, Ltd.

\bibitem[Pastur and Shcherbina(2011)]{pastur_eigenvalue_2011}
Leonid~Andreevich Pastur and Mariya Shcherbina.
\newblock \emph{Eigenvalue {Distribution} of {Large} {Random} {Matrices}}.
\newblock Number 171 in Mathematical {Surveys} and {Monographs}. American
  Mathematical Society, 2011.

\bibitem[Perry et~al.(2020)Perry, Wein, and Bandeira]{perry_statistical_2020}
Amelia Perry, Alexander~S. Wein, and Afonso~S. Bandeira.
\newblock Statistical limits of spiked tensor models.
\newblock \emph{Annales de l'Institut Henri Poincaré, Probabilités et
  Statistiques}, 56\penalty0 (1):\penalty0 230--264, February 2020.
\newblock ISSN 0246-0203.
\newblock \doi{10.1214/19-AIHP960}.
\newblock URL
  \url{https://projecteuclid.org/journals/annales-de-linstitut-henri-poincare-probabilites-et-statistiques/volume-56/issue-1/Statistical-limits-of-spiked-tensor-models/10.1214/19-AIHP960.full}.
\newblock Publisher: Institut Henri Poincaré.

\bibitem[Potters and Bouchaud(2020)]{potters_first_2020}
Marc Potters and Jean-Philippe Bouchaud.
\newblock \emph{A {First} {Course} in {Random} {Matrix} {Theory}}.
\newblock Cambridge University Press, 2020.

\bibitem[Péché(2006)]{peche_largest_2006}
Sandrine Péché.
\newblock The largest eigenvalue of small rank perturbations of {Hermitian}
  random matrices.
\newblock \emph{Probability Theory and Related Fields}, 134\penalty0
  (1):\penalty0 127--173, January 2006.
\newblock ISSN 1432-2064.
\newblock \doi{10.1007/s00440-005-0466-z}.
\newblock URL \url{https://doi.org/10.1007/s00440-005-0466-z}.

\bibitem[Rabanser et~al.(2017)Rabanser, Shchur, and
  Günnemann]{rabanser_introduction_2017}
Stephan Rabanser, Oleksandr Shchur, and Stephan Günnemann.
\newblock Introduction to {Tensor} {Decompositions} and their {Applications} in
  {Machine} {Learning}.
\newblock \emph{arXiv:1711.10781 [cs, stat]}, November 2017.
\newblock URL \url{http://arxiv.org/abs/1711.10781}.
\newblock arXiv: 1711.10781.

\bibitem[Rabinowitz et~al.(2015)Rabinowitz, Goris, Cohen, and
  Simoncelli]{rabinowitz_attention_2015}
Neil~C. Rabinowitz, Robbe~L. Goris, Marlene Cohen, and Eero~P. Simoncelli.
\newblock Attention stabilizes the shared gain of {V4} populations.
\newblock \emph{eLife}, 4:\penalty0 e08998, November 2015.
\newblock ISSN 2050-084X.
\newblock \doi{10.7554/eLife.08998}.
\newblock URL \url{https://doi.org/10.7554/eLife.08998}.
\newblock Publisher: eLife Sciences Publications, Ltd.

\bibitem[Rendle and Schmidt-Thieme(2010)]{rendle_pairwise_2010}
Steffen Rendle and Lars Schmidt-Thieme.
\newblock Pairwise interaction tensor factorization for personalized tag
  recommendation.
\newblock In \emph{Proceedings of the third {ACM} international conference on
  {Web} search and data mining}, {WSDM} '10, pages 81--90, New York, NY, USA,
  2010. Association for Computing Machinery.
\newblock ISBN 978-1-60558-889-6.
\newblock \doi{10.1145/1718487.1718498}.
\newblock URL \url{https://doi.org/10.1145/1718487.1718498}.

\bibitem[Ros et~al.(2019)Ros, Ben~Arous, Biroli, and
  Cammarota]{ros_complex_2019}
Valentina Ros, Gerard Ben~Arous, Giulio Biroli, and Chiara Cammarota.
\newblock Complex {Energy} {Landscapes} in {Spiked}-{Tensor} and {Simple}
  {Glassy} {Models}: {Ruggedness}, {Arrangements} of {Local} {Minima}, and
  {Phase} {Transitions}.
\newblock \emph{Physical Review X}, 9\penalty0 (1):\penalty0 011003, January
  2019.
\newblock \doi{10.1103/PhysRevX.9.011003}.
\newblock URL \url{https://link.aps.org/doi/10.1103/PhysRevX.9.011003}.
\newblock Publisher: American Physical Society.

\bibitem[Rudin(1991)]{rudin_functional_1991}
Walter Rudin.
\newblock \emph{Functional {Analysis}}.
\newblock International {Series} in {Pure} and {Applied} {Mathematics}.
  McGraw-Hill, second edition edition, 1991.
\newblock ISBN 0-07-100944-2.

\bibitem[Savas and Eldén(2007)]{savas_handwritten_2007}
Berkant Savas and Lars Eldén.
\newblock Handwritten digit classification using higher order singular value
  decomposition.
\newblock \emph{Pattern Recognition}, 40\penalty0 (3):\penalty0 993--1003,
  March 2007.
\newblock ISSN 0031-3203.
\newblock \doi{10.1016/j.patcog.2006.08.004}.
\newblock URL
  \url{https://www.sciencedirect.com/science/article/pii/S0031320306003542}.

\bibitem[Schultz(2005)]{schultz_non-commutative_2005}
Hanne Schultz.
\newblock Non-commutative polynomials of independent {Gaussian} random
  matrices. {The} real and symplectic cases.
\newblock \emph{Probability Theory and Related Fields}, 131\penalty0
  (2):\penalty0 261--309, February 2005.
\newblock ISSN 1432-2064.
\newblock \doi{10.1007/s00440-004-0366-7}.
\newblock URL \url{https://doi.org/10.1007/s00440-004-0366-7}.

\bibitem[Seddik et~al.(2022)Seddik, Guillaud, and Couillet]{seddik_when_2022}
Mohamed El~Amine Seddik, Maxime Guillaud, and Romain Couillet.
\newblock When {Random} {Tensors} meet {Random} {Matrices}, November 2022.
\newblock URL \url{http://arxiv.org/abs/2112.12348}.
\newblock arXiv:2112.12348 [math, stat].

\bibitem[Seddik et~al.(2023)Seddik, Tiomoko, Decurninge, Panov, and
  Guillaud]{seddik_learning_2023}
Mohamed El~Amine Seddik, Malik Tiomoko, Alexis Decurninge, Maxim Panov, and
  Maxime Guillaud.
\newblock Learning from {Low} {Rank} {Tensor} {Data}: {A} {Random} {Tensor}
  {Theory} {Perspective}.
\newblock In \emph{Proceedings of the {Thirty}-{Ninth} {Conference} on
  {Uncertainty} in {Artificial} {Intelligence}}, pages 1858--1867. PMLR, July
  2023.
\newblock URL \url{https://proceedings.mlr.press/v216/seddik23a.html}.
\newblock ISSN: 2640-3498.

\bibitem[Seely et~al.(2016)Seely, Kaufman, Ryu, Shenoy, Cunningham, and
  Churchland]{seely_tensor_2016}
Jeffrey~S. Seely, Matthew~T. Kaufman, Stephen~I. Ryu, Krishna~V. Shenoy,
  John~P. Cunningham, and Mark~M. Churchland.
\newblock Tensor {Analysis} {Reveals} {Distinct} {Population} {Structure} that
  {Parallels} the {Different} {Computational} {Roles} of {Areas} {M1} and {V1}.
\newblock \emph{PLOS Computational Biology}, 12\penalty0 (11):\penalty0
  e1005164, November 2016.
\newblock ISSN 1553-7358.
\newblock \doi{10.1371/journal.pcbi.1005164}.
\newblock URL
  \url{https://journals.plos.org/ploscompbiol/article?id=10.1371/journal.pcbi.1005164}.
\newblock Publisher: Public Library of Science.

\bibitem[Sidiropoulos et~al.(2017)Sidiropoulos, De~Lathauwer, Fu, Huang,
  Papalexakis, and Faloutsos]{sidiropoulos_tensor_2017}
Nicholas~D. Sidiropoulos, Lieven De~Lathauwer, Xiao Fu, Kejun Huang,
  Evangelos~E. Papalexakis, and Christos Faloutsos.
\newblock Tensor {Decomposition} for {Signal} {Processing} and {Machine}
  {Learning}.
\newblock \emph{IEEE Transactions on Signal Processing}, 65\penalty0
  (13):\penalty0 3551--3582, July 2017.
\newblock ISSN 1941-0476.
\newblock \doi{10.1109/TSP.2017.2690524}.
\newblock Conference Name: IEEE Transactions on Signal Processing.

\bibitem[Stein(1981)]{stein_estimation_1981}
Charles~M. Stein.
\newblock Estimation of the {Mean} of a {Multivariate} {Normal} {Distribution}.
\newblock \emph{The Annals of Statistics}, 9\penalty0 (6):\penalty0 1135--1151,
  November 1981.
\newblock ISSN 0090-5364, 2168-8966.
\newblock \doi{10.1214/aos/1176345632}.
\newblock URL
  \url{https://projecteuclid.org/journals/annals-of-statistics/volume-9/issue-6/Estimation-of-the-Mean-of-a-Multivariate-Normal-Distribution/10.1214/aos/1176345632.full}.
\newblock Publisher: Institute of Mathematical Statistics.

\bibitem[Stewart and Sun(1990)]{stewart_matrix_1990}
Gilbert~Wright Stewart and Ji-guang Sun.
\newblock \emph{Matrix perturbation theory}.
\newblock Computer science and scientific computing. Academic Press, Boston,
  1990.
\newblock ISBN 978-0-12-670230-9.
\newblock URL \url{http://catdir.loc.gov/catdir/toc/els031/90033378.html}.
\newblock OCLC: 21227976.

\bibitem[Sun et~al.(2021)Sun, Hao, and Li]{sun_tensors_2021}
Will~Wei Sun, Botao Hao, and Lexin Li.
\newblock Tensors in {Modern} {Statistical} {Learning}.
\newblock In \emph{Wiley {StatsRef}: {Statistics} {Reference} {Online}}, pages
  1--25. John Wiley \& Sons, Ltd, 2021.
\newblock ISBN 978-1-118-44511-2.
\newblock \doi{10.1002/9781118445112.stat08319}.
\newblock URL
  \url{https://onlinelibrary.wiley.com/doi/abs/10.1002/9781118445112.stat08319}.
\newblock eprint:
  https://onlinelibrary.wiley.com/doi/pdf/10.1002/9781118445112.stat08319.

\bibitem[Tao(2012)]{tao_topics_2012}
Terence Tao.
\newblock \emph{Topics in {Random} {Matrix} {Theory}}.
\newblock Number 132 in Graduate {Studies} in {Mathematics}. American
  Mathematical Society, 2012.

\bibitem[Tomioka and Suzuki(2014)]{tomioka_spectral_2014}
Ryota Tomioka and Taiji Suzuki.
\newblock Spectral norm of random tensors, July 2014.
\newblock URL \url{http://arxiv.org/abs/1407.1870}.
\newblock arXiv:1407.1870 [math, stat].

\bibitem[Tucker(1966)]{tucker_mathematical_1966}
Ledyard~R. Tucker.
\newblock Some mathematical notes on three-mode factor analysis.
\newblock \emph{Psychometrika}, 31\penalty0 (3):\penalty0 279--311, September
  1966.
\newblock ISSN 1860-0980.
\newblock \doi{10.1007/BF02289464}.
\newblock URL \url{https://doi.org/10.1007/BF02289464}.

\bibitem[Vasilescu(2002)]{vasilescu_human_2002}
M.~Alex~O. Vasilescu.
\newblock Human motion signatures: analysis, synthesis, recognition.
\newblock In \emph{2002 {International} {Conference} on {Pattern}
  {Recognition}}, volume~3, pages 456--460 vol.3, August 2002.
\newblock \doi{10.1109/ICPR.2002.1047975}.
\newblock ISSN: 1051-4651.

\bibitem[Vasilescu and Terzopoulos(2003)]{vasilescu_multilinear_2003}
M.~Alex~O. Vasilescu and Demetri Terzopoulos.
\newblock Multilinear subspace analysis of image ensembles.
\newblock In \emph{2003 {IEEE} {Computer} {Society} {Conference} on {Computer}
  {Vision} and {Pattern} {Recognition}, 2003. {Proceedings}.}, volume~2, pages
  II--93, June 2003.
\newblock \doi{10.1109/CVPR.2003.1211457}.
\newblock ISSN: 1063-6919.

\bibitem[Vervliet et~al.(2014)Vervliet, Debals, Sorber, and
  De~Lathauwer]{vervliet_breaking_2014}
Nico Vervliet, Otto Debals, Laurent Sorber, and Lieven De~Lathauwer.
\newblock Breaking the {Curse} of {Dimensionality} {Using} {Decompositions} of
  {Incomplete} {Tensors}: {Tensor}-based scientific computing in big data
  analysis.
\newblock \emph{IEEE Signal Processing Magazine}, 31\penalty0 (5):\penalty0
  71--79, September 2014.
\newblock ISSN 1558-0792.
\newblock \doi{10.1109/MSP.2014.2329429}.
\newblock Conference Name: IEEE Signal Processing Magazine.

\bibitem[Vervliet et~al.(2016)Vervliet, Debals, Sorber, Van~Barel, and
  De~Lathauwer]{vervliet_tensorlab_2016}
Nico Vervliet, Otto Debals, Laurent Sorber, Marc Van~Barel, and Lieven
  De~Lathauwer.
\newblock Tensorlab 3.0, March 2016.
\newblock URL \url{https://www.tensorlab.net/}.

\bibitem[Weidmann(1980)]{weidmann_linear_1980}
Joachim Weidmann.
\newblock \emph{Linear {Operators} in {Hilbert} {Spaces}}, volume~68 of
  \emph{Graduate {Texts} in {Mathematics}}.
\newblock Springer, New York, NY, 1980.
\newblock ISBN 978-1-4612-6029-5 978-1-4612-6027-1.
\newblock \doi{10.1007/978-1-4612-6027-1}.
\newblock URL \url{http://link.springer.com/10.1007/978-1-4612-6027-1}.

\bibitem[Wein et~al.(2019)Wein, El~Alaoui, and Moore]{wein_kikuchi_2019}
Alexander~S. Wein, Ahmed El~Alaoui, and Cristopher Moore.
\newblock The {Kikuchi} {Hierarchy} and {Tensor} {PCA}.
\newblock In \emph{2019 {IEEE} 60th {Annual} {Symposium} on {Foundations} of
  {Computer} {Science} ({FOCS})}, pages 1446--1468. IEEE Computer Society,
  November 2019.
\newblock ISBN 978-1-72814-952-3.
\newblock \doi{10.1109/FOCS.2019.000-2}.
\newblock URL
  \url{https://www.computer.org/csdl/proceedings-article/focs/2019/495200b446/1grNBDzixFe}.

\bibitem[Wigner(1955)]{wigner_characteristic_1955}
Eugene~P. Wigner.
\newblock Characteristic {Vectors} of {Bordered} {Matrices} {With} {Infinite}
  {Dimensions}.
\newblock \emph{Annals of Mathematics}, 62\penalty0 (3):\penalty0 548--564,
  1955.
\newblock ISSN 0003-486X.
\newblock \doi{10.2307/1970079}.
\newblock URL \url{https://www.jstor.org/stable/1970079}.
\newblock Publisher: Annals of Mathematics.

\bibitem[Wigner(1958)]{wigner_distribution_1958}
Eugene~P. Wigner.
\newblock On the {Distribution} of the {Roots} of {Certain} {Symmetric}
  {Matrices}.
\newblock \emph{Annals of Mathematics}, 67\penalty0 (2):\penalty0 325--327,
  1958.
\newblock ISSN 0003-486X.
\newblock \doi{10.2307/1970008}.
\newblock URL \url{https://www.jstor.org/stable/1970008}.
\newblock Publisher: Annals of Mathematics.

\bibitem[Williams et~al.(2018)Williams, Kim, Wang, Vyas, Ryu, Shenoy,
  Schnitzer, Kolda, and Ganguli]{williams_unsupervised_2018}
Alex~H. Williams, Tony~Hyun Kim, Forea Wang, Saurabh Vyas, Stephen~I. Ryu,
  Krishna~V. Shenoy, Mark Schnitzer, Tamara~G. Kolda, and Surya Ganguli.
\newblock Unsupervised {Discovery} of {Demixed}, {Low}-{Dimensional} {Neural}
  {Dynamics} across {Multiple} {Timescales} through {Tensor} {Component}
  {Analysis}.
\newblock \emph{Neuron}, 98\penalty0 (6):\penalty0 1099--1115.e8, June 2018.
\newblock ISSN 0896-6273.
\newblock \doi{10.1016/j.neuron.2018.05.015}.
\newblock URL
  \url{https://www.sciencedirect.com/science/article/pii/S0896627318303878}.

\bibitem[Xu(2018)]{xu_convergence_2018}
Yangyang Xu.
\newblock On the convergence of higher-order orthogonal iteration.
\newblock \emph{Linear and Multilinear Algebra}, 66\penalty0 (11):\penalty0
  2247--2265, November 2018.
\newblock ISSN 0308-1087.
\newblock \doi{10.1080/03081087.2017.1391743}.
\newblock URL \url{https://doi.org/10.1080/03081087.2017.1391743}.
\newblock Publisher: Taylor \& Francis \_eprint:
  https://doi.org/10.1080/03081087.2017.1391743.

\bibitem[Zdeborová and Krzakala(2016)]{zdeborova_statistical_2016}
Lenka Zdeborová and Florent Krzakala.
\newblock Statistical physics of inference: {Thresholds} and algorithms.
\newblock \emph{Advances in Physics}, 65\penalty0 (5):\penalty0 453--552,
  September 2016.
\newblock ISSN 0001-8732, 1460-6976.
\newblock \doi{10.1080/00018732.2016.1211393}.
\newblock URL \url{http://arxiv.org/abs/1511.02476}.
\newblock arXiv:1511.02476 [cond-mat, stat].

\bibitem[Zhang and Xia(2018)]{zhang_tensor_2018}
Anru Zhang and Dong Xia.
\newblock Tensor {SVD}: {Statistical} and {Computational} {Limits}.
\newblock \emph{IEEE Transactions on Information Theory}, 64\penalty0
  (11):\penalty0 7311--7338, November 2018.
\newblock ISSN 1557-9654.
\newblock \doi{10.1109/TIT.2018.2841377}.
\newblock URL \url{https://ieeexplore.ieee.org/document/8368145}.
\newblock Conference Name: IEEE Transactions on Information Theory.

\bibitem[Zhang et~al.(2013)Zhang, Zhang, Tao, and Huang]{zhang_tensor_2013}
Liangpei Zhang, Lefei Zhang, Dacheng Tao, and Xin Huang.
\newblock Tensor {Discriminative} {Locality} {Alignment} for {Hyperspectral}
  {Image} {Spectral}–{Spatial} {Feature} {Extraction}.
\newblock \emph{IEEE Transactions on Geoscience and Remote Sensing},
  51\penalty0 (1):\penalty0 242--256, January 2013.
\newblock ISSN 1558-0644.
\newblock \doi{10.1109/TGRS.2012.2197860}.
\newblock URL \url{https://ieeexplore.ieee.org/abstract/document/6213108}.
\newblock Conference Name: IEEE Transactions on Geoscience and Remote Sensing.

\bibitem[Zhou et~al.(2013)Zhou, Li, and Zhu]{zhou_tensor_2013}
Hua Zhou, Lexin Li, and Hongtu Zhu.
\newblock Tensor {Regression} with {Applications} in {Neuroimaging} {Data}
  {Analysis}.
\newblock \emph{Journal of the American Statistical Association}, 108\penalty0
  (502):\penalty0 540--552, June 2013.
\newblock ISSN 0162-1459.
\newblock \doi{10.1080/01621459.2013.776499}.
\newblock URL \url{https://doi.org/10.1080/01621459.2013.776499}.
\newblock Publisher: Taylor \& Francis \_eprint:
  https://doi.org/10.1080/01621459.2013.776499.

\end{thebibliography}

\end{document}